\documentclass[10pt,journal,compsoc]{IEEEtran}

\ifCLASSOPTIONcompsoc
  \usepackage[nocompress]{cite}
\else
  \usepackage{cite}
\fi

\ifCLASSINFOpdf
   \usepackage[pdftex]{graphicx}
   \graphicspath{{..f/}{../jpeg/}}
   \DeclareGraphicsExtensions{.pdf,.jpg,.png}
\else
   \usepackage[dvips]{graphicx}
   \graphicspath{{../eps/}}
   \DeclareGraphicsExtensions{.eps}
\fi

\usepackage{amsmath,amssymb,amsthm}
\usepackage{bm}
\usepackage{multirow}
\newtheorem{theorem}{Theorem}[section]
\newtheorem*{theorem*}{Main Result}
\newtheorem{definition}{Definition}
\newtheorem{lemma}{Lemma}[section]
\newtheorem{corollary}{Corollary}[section]
\newtheorem*{remark}{Remark}
\newtheorem{example}{Example}

\DeclareMathOperator{\argmax}{argmax}

\newcommand{\bvec}[1]{\boldsymbol{#1}}
\newcommand{\GE}{\text{GE}}
\newcommand{\PE}{\text{PE}}
\newcommand{\pre}{\text{ pre}}
\newcommand{\ReLU}{\text{ ReLU}}
\newcommand{\relu}{\text{relu}}
\newcommand{\diag}{\text{ diag}}
\newcommand{\supp}{\text{supp}}
\usepackage{fancyhdr}
\usepackage[ruled,linesnumbered]{algorithm2e}
\usepackage[pagebackref=true,breaklinks=true,letterpaper=true,colorlinks,bookmarks=false]{hyperref}

\usepackage{cleveref}

\usepackage{makecell}

\renewcommand{\mathbf}{\boldsymbol}

\begin{document}

\title{Orthogonal Deep Neural Networks}

\author{Kui~Jia*,
        Shuai~Li*,
        Yuxin~Wen,
        Tongliang~Liu,
        and~Dacheng~Tao
\IEEEcompsocitemizethanks{
\IEEEcompsocthanksitem K. Jia, S. Li, and Y. Wen are with the School of Electronic and Information Engineering, South China University of Technology, Guangzhou, China. \protect\\
E-mails: kuijia@scut.edu.cn, lishuai918@gmail.com, wen.yuxin@mail.scut.edu.cn
\IEEEcompsocthanksitem T. Liu and D. Tao are with the Faculty of Engineering and Information Technologies, The University of Sydney, Darlington, NSW, Australia. \protect\\
E-mails: tliang.liu@gmail.com, dacheng.tao@sydney.edu.au
}% <-this % stops an unwanted space
\thanks{* indicates equal contribution.}
}

%% The paper headers
%\markboth{Journal of \LaTeX\ Class Files,~Vol.~XXX, No.~XXX, March~2017}%
%{Shell \MakeLowercase{\textit{et al.}}: Bare Demo of IEEEtran.cls for Computer Society Journals}

\IEEEtitleabstractindextext{%
\begin{abstract}
In this paper, we introduce the algorithms of Orthogonal Deep Neural Networks (OrthDNNs) to connect with recent interest of spectrally regularized deep learning methods. OrthDNNs are theoretically motivated by generalization analysis of modern DNNs, with the aim to find solution properties of network weights that guarantee better generalization. To this end, we first prove that DNNs are of local isometry on data distributions of practical interest; by using a new covering of the sample space and introducing the local isometry property of DNNs into generalization analysis, we establish a new generalization error bound that is both scale- and range-sensitive to singular value spectrum of each of networks' weight matrices. We prove that the optimal bound w.r.t. the degree of isometry is attained when each weight matrix has a spectrum of equal singular values, among which orthogonal weight matrix or a non-square one with orthonormal rows or columns is the most straightforward choice, suggesting the algorithms of OrthDNNs. We present both algorithms of strict and approximate OrthDNNs, and for the later ones we propose a simple yet effective algorithm called Singular Value Bounding (SVB), which performs as well as strict OrthDNNs, but at a much lower computational cost. We also propose Bounded Batch Normalization (BBN) to make compatible use of batch normalization with OrthDNNs. We conduct extensive comparative studies by using modern architectures on benchmark image classification. Experiments show the efficacy of OrthDNNs.
\end{abstract}

}

\maketitle

%\thispagestyle{empty}
%\pagestyle{empty}

\IEEEdisplaynontitleabstractindextext

% \fi
\IEEEpeerreviewmaketitle

\IEEEraisesectionheading{\section{Introduction}
\label{SecIntro}}

Deep learning or deep neural networks (DNNs) have been achieving great success on many machine learning tasks, with image classification \cite{ILSVRC15} as one of the prominent examples. Key design that supports success of deep learning can date at least back to Neocognitron \cite{Neocognitron} and Convolutional neural networks (CNNs) \cite{LeCun98}, which employ hierarchial, compositional design to facilitate learning target functions that approximately capture statistical properties of natural signals. Modern DNNs are usually over-parameterized and have very high model capacities, yet practically meaningful solutions can be obtained via simple back-propagation training of stochastic gradient descent (SGD) \cite{EfficientBackprop}, where regularization methods such as early stopping, weight decay, and data augmentation are commonly used to alleviate the issue of overfitting.

%Classical results on theoretical analysis of DNNs focus on their representational power , showing they are universal approximators \cite{}. , and also the exponential advantage over shallow networks for certain classes of functions \cite{}
%
%Kurt Hornik, Maxwell B. Stinchcombe, Halbert White: Multilayer feedforward networks are universal approximators. Neural Networks 2(5): 359-366 (1989)
%
%Andrew R. Barron, Universal approximation bounds for superpositions of a sigmoidal function. IEEE Transactions on Information Theory 39(3): 930-945 (1993)

Over the years, new technical innovations have been introduced to improve DNNs in terms of architectural design \cite{ResNet,DenseNet}, optimization \cite{XavierInit,TrustRegion4SaddlePoint,Adam}, and also regularization \cite{Dropout,BatchNorm}, which altogether make efficient and effective training of extremely over-parameterized models possible. While many of these innovations are empirically proposed, some of them are justified by subsequent theoretical studies that explain their practical effectiveness. For example, dropout training \cite{Dropout} is explained as an approximate regularization of adaptive weight decay in \cite{UnderstandDropout,DropoutAdapRegu}. Theoretically characterizing global optimality conditions of DNNs are also presented in \cite{DeepLearningNoPoorLocalMinima,GlobalOptimConditions4DNNs}.

% There exist other result \cite{PoggioDLTheoryII} arguing the existence of a large number of global, zero-error solutions for over-parameterized models, which are very likely to be obtained via SGD training.

% Para3:

The above optimization and regularization methods aim to explain and address the generic difficulties of training DNNs, and to improve efficient use of network parameters; they do not have designs on properties of solutions to which network training should converge. In contrast, there exist other deep learning methods that have favored solution properties of network parameters, and expect such properties to guarantee \emph{good generalization at inference time}. In this work, we specially focus on DNN methods that impose explicit regularization on weight matrices of network layers \cite{EDJM,RobustLargeMarginDNNs}. For example, Sokolic \emph{et al.} \cite{RobustLargeMarginDNNs} propose by theoretical analysis a soft regularizer that penalizes Frobenius norm of the Jacobian. More recently, methods that regularize the whole spectrum of singular values and its range for each of networks' weight matrices are also proposed \cite{ParsevalNet,SVB,EfficientOrthDNNNips2018,BeyondGoodInit,UnitaryRNN,FullCapacityUnitaryRNN}. They achieve clearly improved performance over those without imposing such a regularization. However, many of these methods are empirically motivated, with no theoretical justification on its effect on generalization. We aim to study this theoretical issue in this work.

Motivated by geometric intuitions from isometric mappings \cite{DiscrimRobustTransform}, we introduce a term of \emph{local isometry} into the framework of generalization analysis via algorithmic robustness \cite{RobustnessAndGeneralization}. We use an intuitive and also formal definition of \emph{instance-wise variation space} to characterize data distributions of practical interest, and prove that DNNs are of local isometry on such data distributions. More specifically, we prove that for a DNN trained on such a data distribution, a covering based on a linear partition (induced by the DNN) of the input space can be found such that DNN is \emph{locally} linear in each covering ball, where we give bound on the diameters of covering balls in terms of spectral norms of the DNN's weight matrices. Based on a further proof that for a mapping induced by a linear DNN, degree of isometry is fully controlled by singular value spectrum of each of its weight matrices, we establish our generalization error (GE) bound for (nonlinear) DNNs, and show that it is \emph{both scale- and range-sensitive to singular value spectrum of each of their weight matrices}. An illustration of our proofs is given in \cref{FigDemo}. Derivation of our bound is based on a new covering of the sample space, as illustrated in \cref{FigCoveringSchemeDemo}, which enables explicit characterization of GEs caused by both the \emph{distance expansion} and \emph{distance contraction} of locally isometric mappings.

To attain an optimal GE bound w.r.t. the degree of isometry, we prove that the optimum is achieved when each weight matrix of a DNN has a spectrum of equal singular values, among which orthogonal weight matrix or a non-square one with orthonormal rows or columns is the most straightforward choice, suggesting the algorithms of \emph{Orthogonal Deep Neural Networks (OrthDNNs)}. Training to obtain a \emph{strict OrthDNN} amounts to optimizing the weight matrices over their respective Stiefel manifolds, which, however, is very costly for large-sized DNNs. To achieve efficient learning, we propose a simple yet effective algorithm of \emph{approximate OrthDNNs} called \emph{Singular Value Bounding (SVB)}. SVB periodically bounds, in the SGD based training iterations, all singular values of each weight matrix in a narrow band around the value of $1$, thus achieving near orthogonality (row- or column-wise orthonormality) of weight matrices. In this work, we also discuss alternative schemes of soft regularization \cite{ParsevalNet,BeyondGoodInit,SRIP} to achieve approximate OrthDNNs, and compare with our proposed SVB. Batch Normalization (BN) \cite{BatchNorm} is commonly used in modern DNNs, yet it has a potential risk of ill-conditioned layer transform, making it incompatible with OrthDNNs. We propose \emph{Degenerate Batch Normalization (DBN)} and \emph{Bounded Batch Normalization (BBN)} to remove such a potential risk, and to enable its use with strict and approximate OrthDNNs respectively.

To investigate the efficacy of OrthDNNs, we conduct extensive experiments of benchmark image classification \cite{Cifar,ILSVRC15} on modern architectures \cite{VGGNet,PreActResNet,WideResNet,DenseNet,ResNeXt}. These experiments show that OrthDNNs consistently improve generalization by providing regularization to training of these architectures. Interestingly, approximate OrthDNNs perform as well as strict ones, but at a much lower computational cost. For approximate OrthDNNs, we also compare hard regularization via our proposed SVB and BBN with the alternatives of soft regularization; our results are better than or comparable to those of these alternatives on modern architectures. In some of these studies, we investigate behaviors of our method under learning regimes from small to large sizes of training samples; results confirm the empirical strength of our method, especially for learning problems of smaller sample sizes. We also investigate robustness of our method against corruptions that are commonly encountered in natural images; our results demonstrate better robustness against such corruptions, and the robustness stands gracefully with increase of corruption severity levels.

\subsection{Relations with existing works}

\subsubsection{Generalization analysis of DNNs}

Classical theories of DNNs show that they are universal approximators \cite{Hornik1989,Barron1993}. However, recent results from Zhang \emph{et al.} \cite{ZhangICLRBestPaper} show an apparent puzzle that over-parameterized DNNs are able to shatter randomly labeled training data, suggesting worst-case generalization since test performance can only be at a chance level, while at the same time they perform well on practical learning tasks (e.g., ImageNet classification); the puzzle suggests that traditional analysis of data-independent generalization does not readily apply. They further conjecture \cite{PoggioDLTheoryIIb} that over-parameterized DNNs, when trained via SGD, tend to find local solutions that fall in, with high probability, flat regions in the high-dimensional solution space, which is even obvious when learning tasks are on natural signals; flat-region solutions imply robustness in the parameter space of DNNs, which may further implies robustness in the input data space. Similar argument of flat-region solutions is also presented in \cite{FlatMinima}, although Dinh \emph{et al.} \cite{SharpMinimaCanGeneralize} argue that these flat minima can be equivalently converted as sharp minima without affecting network prediction. Generalization of DNNs is also explained by stochastic optimization. In \cite{StabilitySGD}, the notion of uniform stability \cite{BousquetStability} is extended to characterize the randomness of SGD, and a generalization bound in expectation is established for learning with SGD. The distribution-free stability bound of \cite{StabilitySGD} is improved in \cite{DataDependentStabilitySGD} via the notion of on-average stability, revealing data-dependent behavior of SGD. To understand practical generalization of DNNs, Kawaguchi \emph{et al.} \cite{GeneralizationInDeepLearning} argue that independent of the hypothesis set and algorithms used, the learned model itself, possibly selected via a validation set, is the most important factor that accounts for good generalization; a generalization bound w.r.t. validation error is also presented in \cite{GeneralizationInDeepLearning}.

%\noindent\emph{GE analysis that characterize various norms of DNNs:}
%
%Norm-Based Capacity Control in Neural Networks
%
%Implicit regularization provided by geometry of parameter space is also studied in [Geometry of Optimization and Implicit Regurlization in Deep Learning].

To further characterize generalization of DNNs with their weight matrices, Sokolic \emph{et al.} \cite{RobustLargeMarginDNNs} study DNNs as robust large-margin classifiers via the algorithmic robustness framework \cite{RobustnessAndGeneralization}. They introduce a notion of average Jacobian, and use spectral norm of the Jacobian matrix to locally bound the distance expansion from the input to the output space of a DNN; spectral norm of the Jacobian is further relaxed as the product of spectral norms of the network's weight matrices, which is used to establish the robustness based generalization bound. Bartlett \emph{et al.} \cite{SpectrallyNormalizedMarginBound} use a scale-sensitive measure of complexity to establish a generalization bound. They derive a margin-normalized spectral complexity, i.e., the product of spectral norms of weight matrices divided by the margin, via covering number approximation of Rademacher complexity; they further show empirically that such a bound is task-dependent, suggesting that SGD training learns parameters of a DNN whose complexity scales with the difficulty of the learning task.

While both of our bound and that of \cite{RobustLargeMarginDNNs} are developed under the framework of algorithmic robustness \cite{RobustnessAndGeneralization}, our bound is controlled by the whole spectrum of singular values, rather than spectral norm (i.e., the largest singular value) of each of the network's weight matrices, by introducing a term of local isometry into the framework. This also means that in contrast to \cite{SpectrallyNormalizedMarginBound}, our bound is both scale- and range-sensitive to singular values of weight matrices. The fact that our bound is scale-sensitive in the sense of \cite{SpectrallyNormalizedMarginBound} implies that for difficult learning tasks, e.g., randomly labeled CIFAR10 \cite{ZhangICLRBestPaper}, spectral norms of weight matrices would go extremely large, causing the diameters of covering balls go extremely small and correspondingly the second term of our bound (cf. \cref{thm:ge_contraction}) that characterizes distribution mismatch between training and test samples dominates, and that the bound becomes vacuous. In contrast, for learning tasks of practical interest, e.g., standard CIFAR10 \cite{Cifar}, the spectra of singular values of weight matrices are potentially in a benign range, and the bound is of practical use to inspire design of improved learning algorithms.

% In general, we may say that a bound is X-dependent if the bound contains a quantity that is a function of X, which contrasts with the X-independent bound that generally captures the worst-case properties of learning algorithms. By considering X, the bound could be tighter than the worst-case bound, and may also motivate more effective learning algorithms; furthermore, empirical testing can be conducted by specifying different settings of X (e.g., different data distributions or different learning tasks) to get different function values (thus different bound values), showing an adaptive behavior for the established X-dependent bound.

\subsubsection{Compositional computations and isometries of DNNs}

Mont\'{u}far \emph{et al.} \cite{NumOfLinearRegionDNN} characterize complexity of functions computable by DNNs and establish a lower bound on the maximal number of linear regions into which a DNN (with ReLU activation) can partition the input space, where the bound is derived by compositional replication of layer-wise space partitioning and grows exponentially with depth of the DNN. Our derivation of the analytic form of region-wise linear mapping (cf. \cref{lm:2}) borrows ideas from \cite{NumOfLinearRegionDNN}. Similar compositional derivations for the number of computational paths from the network input to a hidden unit are also presented in \cite{DeepLearningNoPoorLocalMinima,GeneralizationInDeepLearning}.

Geometric intuition of isometric mappings has been introduced to improve robustness of deep feature transformation \cite{DiscrimRobustTransform}, where DNNs are studied as a form of transformation functions. However, their development of robustness bound only uses explicitly the distance expansion constraint of isometric mappings; moreover, their studies are in the context of metric learning and for DNNs, they stay on a general function form, with no indications on how layer-wise weight matrices affect generalization.

\subsubsection{Optimization benefits of isometry/orthogonality}

Previous works \cite{ExactSolution,BeyondGoodInit,SVB} show that orthogonality helps the optimization of DNNs by preventing explosion or vanishing of back-propagated gradients. More specifically, a property of dynamic isometry is studied in \cite{ExactSolution} to understand learning dynamics of deep linear networks. Pennington \emph{et al.} \cite{ResurrectingSigmoid} extend such studies to DNNs by employing powerful tools from free probability theory; they show that with orthogonal weight initialization, sigmoid
activation functions can keep the maximum singular value to be $1$ as layers go deeper, and isometry of DNNs can be preserved for a large amount of time during training. However, the analysis on optimization benefits does not explain the gain in test accuracy, i.e., the generalization.

\subsubsection{Regularization on weight matrices}

Wang \emph{et al.} \cite{EDJM} propose Extended Data Jacobian Matrix (EDJM) as a network analyzing tool, and study how the spectrum of EDJM affects performance of different networks of varying depths, architectures, and training methods. Based on these observations, they propose a spectral soft regularizer that encourages major singular values of EDJM to be closer to the largest one (practically implemented on weight matrix of each layer). As discussed above, a related notion of average Jacobian is used in \cite{RobustLargeMarginDNNs} to motivate a soft regularizer that penalizes spectral norms of weight matrices.

There exist other recent methods \cite{SVB,ParsevalNet,EfficientOrthDNNNips2018,BeyondGoodInit} that improve empirical performance of DNNs by regularizing the whole spectrum of singular values for each of networks' weight matrices. This is implemented in \cite{ParsevalNet,BeyondGoodInit} as soft regularizers that encourage the product between each weight matrix and its transpose to be close to an identity one. Different from \cite{ParsevalNet,BeyondGoodInit}, we propose a hard regularization method termed Singular Value Bounding (SVB), which periodically bounds in the training process all singular values of each weight matrix in a narrow band around the value of $1$, so that orthonormality of rows or columns of weight matrices can be approximately achieved.

%In the present paper, we are motivated by our generalization bound and further investigate the regularization that enforces strict orthonormality, i.e., to optimize each weight matrix over its Stiefel manifold. We also compare the empirical performance of SVB with alterative regularization schemes, such as those used in \cite{ParsevalNet,BeyondGoodInit}. Note that convergence analysis of SGD based network optimization over Stiefel manifolds of weight matrices is studied in \cite{OzayOkataniCNNKernelSubManifold}.

\subsection{Contributions}

There exists a growing recent interest on using spectral regularization to improve training of DNNs. These methods impose explicit regularization on weight matrices of network layers by penalizing either their spectral norms \cite{RobustLargeMarginDNNs,EDJM} or the whole spectrums of their singular values \cite{SVB,ParsevalNet,EfficientOrthDNNNips2018,BeyondGoodInit}. The SVB algorithm proposed in our preliminary work \cite{SVB} is among the later approach. Most of these methods are empirically motivated with no theoretical guarantees. In the present paper, we focus on theoretical analysis of these methods from the perspective of generalization analysis, and prove a novel GE bound for data distributions of practical interest. We also intensively compare empirical performance of these methods, and present their empirical strengths under various learning scenarios. We summarize our technical contributions as follows.

\begin{itemize}
\item We present in this paper a new generalization error bound for DNNs. We first prove that DNNs are of \emph{local isometry} on data distributions of practical interest, where the degree of isometry is fully controlled by singular value spectrum of each of their weight matrices. By using a new covering of the sample space and introducing the local isometry property of DNNs into an algorithmic robustness framework, we establish our GE bound and show that it is \emph{both scale- and range-sensitive to singular value spectrum of each of networks' weight matrices}.
\item We prove that the optimal bound w.r.t. the degree of isometry is attained when each weight matrix of a DNN has a spectrum of equal singular values, among which orthogonal weight matrix or a non-square one with orthonormal rows or columns is the most straightforward choice, suggesting the algorithms of \emph{Orthogonal Deep Neural Networks (OrthDNNs)}. In this paper, we also present the algorithmic details of OrthDNNs.
\item To address the heavy computation of \emph{strict OrthDNNs}, we propose a novel algorithm called \emph{Singular Value Bounding (SVB)}, which achieves \emph{approximate OrthDNNs} via a simple scheme of hard regularization. We discuss alternative schemes of soft regularization, and compare with our proposed SVB. Batch normalization has a potential risk of ill-conditioned layer transform, making it incompatible with OrthDNNs. We propose \emph{Degenerate Batch Normalization (DBN)} and \emph{Bounded Batch Normalization (BBN)} to remove such a potential risk, and to enable its use with strict and approximate OrthDNNs.
\end{itemize}

%\section{Generalization bound of Deep Neural Networks}
%
%%Rewrite the paragraph right after Theorem 4.1
%
%To control the first term, the natural approach is to constrain the variation of the loss function. In [37], they propose $\delta$-isometry as a desirable property in CRL problem to control the variation, where $\delta$-isometry is a geometric property of mapping functions whose definition will be given shortly.

\section{Problem Statement}
\label{sec:problem-formulation}

We start by describing the formalism of classification problems that jointly learn a representation and a classifier, e.g., via Deep Neural Networks (DNNs).

% Our math notations are based on the following conventions: scalars are denoted as normal letters, e.g., $x$; bold, lowercase letters denote vectors, e.g., $\bvec{x}$; bold, uppercase letters denote matrices, e.g., $\bvec{W}$; all the remaining symbols such as letters of calligraphic fonts or other normal letters are defined when needed, and should be self-clear in the context.

\subsection{The classification-representation-learning problem and its generalization error}
\label{sec:repr-learn-its}

Assume a sample space $\mathcal{Z} = \mathcal{X} \times \mathcal{Y}$, where $\mathcal{X}$ is the instance space and $\mathcal{Y}$ is the label space. We restrict ourselves to classification problems in this paper, and have $\bvec{x} \in \mathcal{X}$ as vectors in $\mathbb{R}^{n}$ and $y \in \mathcal{Y}$ as a positive integer less than $|\mathcal{Y}| \in \mathbb{N}$. We use $S_m = \{ s_i = (\bvec{x}_i, y_i) \}_{i=1}^{m}$ to denote the training set of size $m$ whose examples are drawn independent and identically distributed (i.i.d.) according to an unknown distribution $P$. We also denote $S_m^{(x)} = \{\bvec{x}_i\}_{i=1}^{m}$. Given a loss function $\mathcal{L}$, the goal of learning is to identify a function $f_{S_m}: \mathcal{X} \mapsto \mathcal{Y}$ in a hypothesis space (a class $\mathcal{F}$ of functions) that minimizes the expected risk
\[
  R(f) = \mathbb{E}_{z \sim P}\left[\mathcal{L}\left(f(\bvec{x}), y\right)\right] ,
\]
where $z = (\bvec{x}, y) \in \mathcal{Z}$ is sampled i.i.d. according to $P$. Since $P$ is unknown, the observable quantity serving as a proxy
to the expected risk $R(f)$ is the empirical risk
\[
  R_m(f) = \frac{1}{m}\sum\limits_{i=1}^{m}\mathcal{L}\left(f(\bvec{x}_i), y_i\right) .
\]
One of the primary goals in statistical learning theory is to characterize the discrepancy between $R(f_{S_m})$ and $R_m(f_{S_m})$, which is termed as {\it generalization error} --- it is sometimes termed as generalization gap in the literature
\[
  \text{GE}(f_{S_m}) = |R(f_{S_m}) - R_m(f_{S_m})| .
\]

In this paper, we are interested in using DNNs to solve classification problems. It amounts to learning a map $T$, which extracts feature characteristic to a classification task, and minimizing $R_m$ simultaneously. We denote classification with this approach as a {\it Classification-Representation-Learning} (CRL) problem. We single out the map $T$ because most of the theoretical analysis in this paper resolves around it. Rewriting the two risks by incorporating a map $T$ (we write $\bvec{T}$ when it is instantiated by a DNN), we have
\begin{equation}
  \label{true_risk}
  R(f,T) = \mathbb{E}_{z \sim P}\left[\mathcal{L}\left(f({T}\bvec{x}), y\right)\right] ,
\end{equation}
\begin{equation}
  \label{emp_risk}
  R_m(f,T) = \frac{1}{m}\sum\limits_{i=1}^{m}\mathcal{L}\left(f(T\bvec{x}_i), y_i\right) .
\end{equation}

\subsection{Generalization analysis for robust algorithms with isometric mapping}
\label{sec:framework}

The upper bounds of GE are generally established by leveraging on certain measures related to the capacity of hypothesis space $\mathcal{F}$, such as Rademacher complexity and VC-dimension \cite{FoundationMLBook}. These complexity measures capture global properties of $\mathcal{F}$; however, GE bounds based on them ignore the specifically used learning algorithms. To establish a finer bound, one may resort to algorithm-dependent analysis \cite{RobustnessAndGeneralization,liu2017algorithmic}. Our analysis of GE bound in this work is based on the algorithmic robustness framework \cite{RobustnessAndGeneralization} that has the advantage of conveying information of local geometry. We begin with the definition of robustness used in \cite{RobustnessAndGeneralization}.

\begin{definition}[$(K, \epsilon(\cdot))$-robustness]
  An algorithm is $(K, \epsilon(\cdot))$-robust for $K \in \mathbb{N}$ and $\epsilon(\cdot):
  \mathcal{Z}^{m} \mapsto \mathbb{R}$, if $\mathcal{Z}$ can be partitioned into $K$
  disjoint sets, denoted by $\mathcal{C} = \{ C_k \}_{k=1}^{K}$, such that the following
  holds for all $s_i = (\bvec{x}_i, y_i) \in S_m, z=(\bvec{x}, y) \in \mathcal{Z}, C_k \in \mathcal{C}$:
  \begin{align*}
    \forall s_i = (\bvec{x}_i, y_i) \in C_k, \forall z=(\bvec{x}, y) \in C_k \\
    \implies |\mathcal{L}(f(\bvec{x}_i), y_i) - \mathcal{L}(f(\bvec{x}), y)| \leq \epsilon(S_m).
  \end{align*}
\end{definition}
The gist of the definition is to constrain the variation of loss values on test examples w.r.t. those of training ones through local property of the algorithmically learned function. Intuitively, if $s \in S_m$ and $z \in
\mathcal{Z}$ are ``close'' (e.g., in the same partition $C_k$), their loss should also be close, due to the intrinsic constraint imposed by $f$.

For any algorithm that is robust, \cite{RobustnessAndGeneralization} proves
\begin{theorem}[\cite{RobustnessAndGeneralization}]
  \label{thm:1}
  If a learning algorithm is $(K, \epsilon(\cdot))$-robust and $\mathcal{L}$ is bounded,
  a.k.a. $\mathcal{L}(f(\bvec{x}), y) \leq M$ $\forall z \in \mathcal{Z}$, for any $\nu > 0$, with probability
  at least $1 - \nu$ we have
  \begin{equation}
    \label{abs_robust_bound}
    \GE(f_{S_m}) \leq \epsilon(S_m) + M\sqrt{\frac{2K\log(2) + 2 \log(1/\nu)}{m}} .
  \end{equation}
\end{theorem}
To control the first term, a natural approach is to constrain the variation of the loss function.
Covering number \cite{CoveringNumber} provides a way to bound the variation of the loss function, and more importantly, it conceptually realizes the actual number $K$ of disjoint partitions.
\begin{definition}[Covering number]
  Given a metric space $(\mathcal{S}, \rho)$, we say that a subset
  $\hat{\mathcal{S}}$ of $\mathcal{S}$ is a $\gamma$-cover of $\mathcal{S}$, if $\forall s
  \in \mathcal{S}$, $\exists \hat{s} \in \hat{\mathcal{S}}$ such that $\rho(s, \hat{s}) \leq \gamma$. The
  $\gamma$-covering number of $\mathcal{S}$ is
  \begin{displaymath}
    \mathcal{N}_{\gamma}(\mathcal{S}, \rho) = \min\{ |\hat{\mathcal{S}}|:
    \hat{\mathcal{S}} \text{ is a } \gamma\text{-covering of } \mathcal{S} \} .
  \end{displaymath}
\end{definition}

In \cite{DiscrimRobustTransform}, they propose $\delta$-isometry as a desirable property in CRL problem to help control the variation, where $\delta$-isometry is a geometric property of mapping functions.
\begin{definition}[$\delta$-isometry]
  Given a map $T$ that maps a metric space $(\mathcal{P}, \rho_{P})$
to another metric space $(\mathcal{Q}, \rho_{Q})$, it is called
$\delta$-isometry if the following inequality holds
\begin{displaymath}
  \forall \bvec{x}, \bvec{x}' \in \mathcal{P}, |\rho_{Q}(T\bvec{x}, T\bvec{x}') - \rho_{P}(\bvec{x}, \bvec{x}')| \leq \delta .
\end{displaymath}
\end{definition}

When $T$ in \cref{true_risk} and \cref{emp_risk} is of $\delta$-isometry, by using $\rho_{Q}(T\bvec{x}, T\bvec{x}') \leq \rho_{P}(\bvec{x}, \bvec{x}')+ \delta$ a realization of algorithmic robustness (or GE bound in the form of Theorem \ref{thm:1}) similar to \cite{DiscrimRobustTransform} can be established for DNNs as follows.
\begin{theorem}\label{isometry-robustness}
  Given an algorithm in a CRL problem,
  if the Lipschtiz constant of $\mathcal{L} \circ f$ w.r.t. $\bvec{Tx}$ is bounded by $A$, $T$ is of $\delta$-isometry,
  and $\mathcal{X}$ is compact with a covering number
  $\mathcal{N}_{\gamma/2}(\mathcal{X}, \rho)$, then it is
  $(|\mathcal{Y}| \mathcal{N}_{\gamma/2}(\mathcal{X}, \rho), A(\gamma + \delta))$-robust.
\end{theorem}

\begin{remark}
  The result in \cite{DiscrimRobustTransform} is $(|\mathcal{Y}| \mathcal{N}_{\gamma/2}(\mathcal{X}, \rho), 2A(\gamma + \delta))$-robust; the factor of $2$ in the second term is dropped here due to the
  fact that in CRL problems, we are not doing metric learning as in
  \cite{DiscrimRobustTransform}, which involves two pairs of examples, and we
  only compare one pair of examples. Its proof under the context of
  DNN, i.e., the proof of \cref{thm:ge_dnn}, is given in Appendix \ref{proof-thm-ge_dnn}.
\end{remark}

\begin{remark}
Denote $\rho_{Q}(T\bvec{x}, T\bvec{x}') \leq \rho_{P}(\bvec{x}, \bvec{x}')+ \delta$ as the \emph{expansion property} of the $\delta$-isometry, and $\rho_{Q}(T\bvec{x}, T\bvec{x}') \geq \rho_{P}(\bvec{x}, \bvec{x}')- \delta$ as its \emph{contraction property}. We note that the above theorem is established by only exploiting the expansion property. After proving that DNNs achieve locally isometric mappings in \cref{sec:lemmas}, we will show that a better generalization can be derived by considering both the properties.
\end{remark}

\subsection{Notations of deep neural networks}
\label{sec:neural-network}

We study the map $\bvec{T}$ as a neural network. We present the definition of Multi-Layer Perceptron (MLP) here, which captures all ingredients for theoretical analysis and enables us to convey the analysis without unnecessary
complications, though we note that the analysis extends to Convolutional Neural Networks (CNNs) almost equally.

A MLP is a map that takes an input $\bvec{x} \in \mathbb{R}^n$ from the space
$\mathcal{X}$, and builds its output by recursively applying a linear map
$\bvec{W}_{l}$ followed by a pointwise non-linearity $g$
\begin{equation}
  \label{eq:nn}
  \bvec{x}_l = g (\bvec{W}_l \bvec{x}_{l-1}) ,
\end{equation}
where $l \in \{1, \dots, L\}$ indexes the layer, $\bvec{x}_l \in \mathbb{R}^{n_l}$, $\bvec{x}_0 = \bvec{x}$, $\bvec{W}_l \in \mathbb{R}^{n_l\times n_{l-1}}$, and $g$ denotes the activation function, which throughout the paper is the Rectifier Linear Unit (ReLU) \cite{Glorot2011}. Optionally $g$ may include max pooling operator \cite{Becigneul2017} after applying ReLU. We also denote the intermediate feature space $g(\bvec{W}_l\bvec{x}_{l-1})$ as $\mathcal{X}_l$ and $\mathcal{X}_0 = \mathcal{X}$. Each $\mathcal{X}_l$ is a
metric space, and throughout the paper the metric is taken as the $\ell_2$ norm
$||\cdot ||_2$, shortened as $||\cdot||$. We compactly write the map of a MLP as
\begin{displaymath}
  \bvec{T}\bvec{x} = \bvec{W}_Lg(\bvec{W}_{L-1}\ldots g(\bvec{W}_1\bvec{x})) .
\end{displaymath}
We denote the spectrum of singular values of a matrix $\bvec{W}$ by $\sigma(\bvec{W})$, and $\sigma_{\max}$ and $\sigma_{\min}$ are the maximum and minimum (nonzero) singular values of $\bvec{W}$ respectively. We denote the rank of a matrix $\bvec{W}$ by $r(\bvec{W})$, and the null space of $\bvec{W}$ by $\mathcal{N}(\bvec{W})$. We write the complement of $\mathcal{N}(\bvec{W})$ as $\mathcal{X} - \mathcal{N}(\bvec{W})$.

\section{Generalization bounds of deep neural networks}
\label{sec:weight-spectr-gener}

\begin{figure}[t]
  \centering
  \includegraphics[scale=0.3]{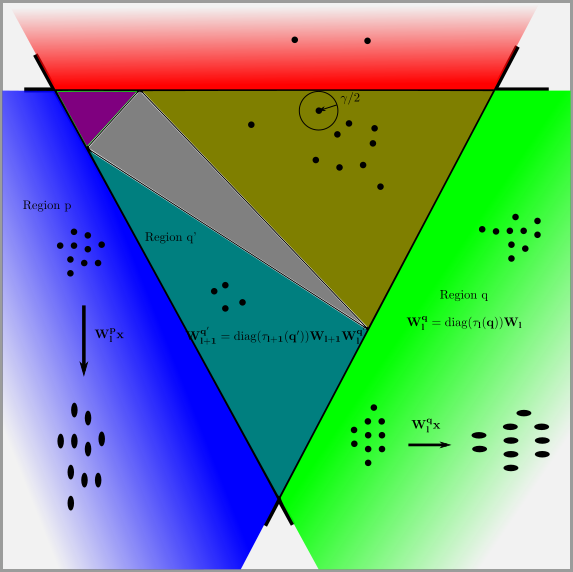}  \\
  \caption{ {\small Illustration for the local $\delta$-isometry of a Deep Neural
Network (DNN) with ReLU activation and max pooling functions. Black dots
represent training instances in the sample space, and regions coded with
different colors represent regions in the instance space that are
hierarchically specified by layers of the DNN. The illustration depicts proofs
of lemmas \ref{lm:1}, \ref{lm:2}, \ref{lm:3}, and \ref{lm:4}. \Cref{lm:1}
proves that the mapping induced by a linear DNN is of $\delta$-isometry, where $\delta$
specified the expansion and contraction properties of the mapping and is
determined by singular value spectrums of weight matrices of all the layers.
\Cref{lm:2} proves that a nonlinear DNN partitions the instance space into
increasingly refined regions. As illustrated in the figure, the space is
firstly partitioned into coarser regions (i.e., the center triangle and other
three regions color coded as red, blue, and green), an additional layer further
partitions some of the coarser regions into sets of smaller regions (e.g.,
those inside the region of center triangle), and the process goes
recursively. Suppose that a region $q$ is created by layer $l$ of the DNN, and
in region $q$, the nonlinear mapping defined by the matrix $\bvec{W}_{l}$ and
activation function reduces to a linear mapping of $\bvec{W}_{l}^{q} =
\diag(\tau_{l}(q))\bvec{W}_l$, where $\tau_{l}(q)$ is a binary vector roughly
indicating active neurons, and $\diag(\cdot)$ diagonalizes $\tau_l(q)$. Suppose $q'$
is created by layer $l+1$ at the bottom part of the center triangle, and in
region $q'$, the nonlinear mappings of layer $l$ and layer $l+1$ are reduced to
a linear mapping of $\bvec{W}_{l+1}^{q'} = \diag (\tau_{l+1}(q'))
\bvec{W}_{l+1}\bvec{W}^{q'}_{l}$, where symbols have similar meanings as
described above. The phenomenon enables to find a covering for the sample
space, such that in each covering ball that contains training instances, e.g.,
$\bvec{x}$, the DNN $\bvec{T}$ defines a transformation that can be
characterized as a linear mapping, e.g., $\bvec{T}_{|\bvec{x}}\bvec{x}$. This
is proved in \cref{lm:3}. Radius of the covering ball is illustrated as $\gamma/2$
in the figure --- a radius is acceptable as long as it is less than the smallest
distance from any of the training instances to their respective region
boundaries. The behaviors of $\delta$-isometry of $\bvec{T}_{|\bvec{x}}$ are also
visualized in regions $p$ and $q$ respectively. The transformation of $\bvec{T}$ applied on
instances $\bvec{x}$ is different in different regions. As a demonstration, in
region $p$, the transformation vertically elongates the distance between
instances, while in region $q$, it horizontally elongates the distance instead. Lastly, in \cref{lm:4}, we prove that by Cauchy interlacing law, a
nonlinear DNN is of local $\delta$-isometry within each covering ball specified
above.  }}
  \label{FigDemo}
\end{figure}

In this section, we develop GE bounds for CRL problems instantiated by DNNs. We identify two quantities that help control a bound, i.e., $\delta$-isometry of $\bvec{T}$ and the diameter $\gamma$ of covering balls of $\mathcal{X}$. We show that both of the two quantities can be controlled by constraining the spectrum of singular values of the weight matrix associated with each network layer, i.e., spectrums of singular values of $\{\bvec{W}_i\}_{i=1,\ldots, L}$.

To proceed, we consider in this paper variations of instances in $\mathcal{X}$ that are of practical interest --- more specifically, those that output nonzero vectors after passing through a DNN. We first prove in \cref{lm:1} that in such a variation subspace, mapping induced by a \emph{linear} neural network is of $\delta$-isometry, where $\delta$ is specified by the \emph{maximum and minimum} singular values of weight matrices of all the network layers. For a \emph{nonlinear} neural network, where we assume ReLU activation and optionally with max pooling, we consider the fact that it divides the input space $\mathcal{X}$ into a set of regions and within each region, it induces a linear mapping. In \cref{lm:2}, we specify the explicit form of region-wise mapping $\bvec{T}_q$, associated with any linear region $q$, with submatrices of $\{\bvec{W}_i^q\}_{i=1,\ldots, L}$, based on which we prove in \cref{lm:3} that a covering set for $\mathcal{X}$ can be found with a diameter $\gamma$ of covering balls that is upper bounded by a quantity inversely proportional to the product of maximum singular values of weight matrices of some network layers. With the $\delta$ and $\gamma$ specified in \cref{lm:1} and \cref{lm:3}, we further prove in \cref{lm:4} that in the instance-wise variation subspaces considered in this paper, a nonlinear neural network $\bvec{T}$ is of local $\delta$-isometry within each covering ball. The proofs are illustrated in \cref{FigDemo}.

To develop a GE bound, we propose a covering scheme that includes instances of different labels into the same balls, thus reducing the size of covering set when compared with that in \cref{thm:1}. The covering scheme is illustrated in \cref{FigCoveringSchemeDemo}. We correspondingly characterize both the errors caused by \emph{distance contraction} between instances of different labels and those by \emph{distance expansion} between instances of the same labels. Based on such characterization, we come with our \emph{main result of \cref{thm:ge_contraction}}.

Given the bound in \cref{thm:ge_contraction}, we prove in \cref{lm:equalSinValue} that the optimal bound w.r.t. $\delta$ is obtained when \emph{all singular values of the weight matrix of each network layer are of equal ones}, which inspires a straightforward choice of enforcing all singular values to have the value of $1$, and thus the algorithms of OrthDNNs.

\begin{figure}[t]
  \centering
  \includegraphics[scale=0.4]{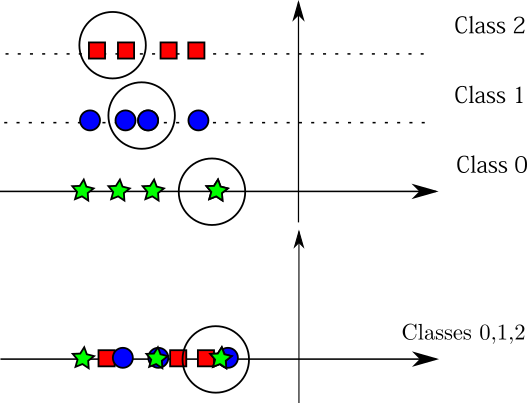}  \\
  \caption{ {\small Illustration of our used covering scheme that includes instances of different labels into same balls. Colored squares, circles, and stars represent instances of different labels, and unfilled (large) circles represent covering balls. \emph{Top:} existing works \cite{RobustnessAndGeneralization,RobustLargeMarginDNNs} usually separate instances of different labels into different covering balls, either by assuming that distances between instances of different labels in the sample space are infinite, or by using covering balls that are small enough not to contain instances of different labels; such a scheme can characterize the generalization errors caused by improper expansion of intra-class variations, but it cannot characterize the errors caused by improper contraction of inter-class differences. \emph{Bottom:} we use a covering scheme that includes instances into covering balls regardless of their labels; it enables characterization of both types of the aforementioned errors, and leads to a possibly tighter generalization error bound in some cases.  }}
  \label{FigCoveringSchemeDemo}
\end{figure}

\subsection{$\delta$-isometry in deep neural networks}
\label{sec:lemmas}

We begin with a few definitions necessary for the subsequent analysis.

\begin{definition}[Variation subspace of an instance]
  Given an instance $\bvec{x} \in \mathcal{X}$, suppose we are interested in the variation
  of a set $\mathcal{X'} \subseteq \mathcal{X}$ w.r.t. $\bvec{x}$. We call the linear vector space
  \[
    \text{span}(\{\bvec{x}' - \bvec{x} | \bvec{x}' \in \mathcal{X'}\})
  \]
  the variation subspace w.r.t. the instance $\bvec{x}$ of $\mathcal{X}$, shorten as variation subspace of instance $\bvec{x}$.
\end{definition}
The definition is to formalize variations of interest of particular
instances, thus enabling us to discuss what variations a DNN is able to
constrain. Correspondingly, we have the following definition of isometry.

\begin{definition}[$\delta$-isometry w.r.t. variation subspace of an instance]
  Given a map $\bvec{T}$ that maps a metric space $(\mathcal{P}, \rho_{P})$
to another metric space $(\mathcal{Q}, \rho_{Q})$, it is called
$\delta$-isometry w.r.t. the variation subspace $\mathcal{P}_x$ of instance $\bvec{x}$, if the following holds
\begin{displaymath}
  \forall \bvec{x}' \in \{\bvec{z} \ |\ \bvec{z} - \bvec{x} \in \mathcal{P}_x\}, |\rho_{Q}(\bvec{T}\bvec{x}, \bvec{T}\bvec{x}') - \rho_{P}(\bvec{x}, \bvec{x}')| \leq \delta .
\end{displaymath}
\end{definition}

We provide an example here to describe $\delta$-isometry w.r.t. the variation subspace of a linear DNN.

\begin{example}
In a linear DNN $\bvec{T}$, given an instance $\bvec{x}$, it is $\delta$-isometry
w.r.t. variation subspace $\mathcal{X} - \mathcal{N}(\bvec{T})$. This is proved
in \cref{lm:1}. Thus, for any variation $\delta\bvec{x} \in \mathcal{X} -
\mathcal{N}(\bvec{T})$, we have $\left| \rho(\bvec{Tx}, \bvec{T}(\bvec{x} +
\delta\bvec{x})) - \rho(\bvec{x}, \bvec{x} + \delta\bvec{x}) \right|\leq \delta$. In this case, the
variation subspace $\mathcal{X} - \mathcal{N}(\bvec{T})$ is the same for any
instance $\bvec{x}$.
\end{example}

The following lemma specifies for a \emph{linear} DNN the $\delta$-isometry w.r.t. data variations of practical interest.
\begin{lemma}
  \label{lm:1}
  Given a linear neural network $\bvec{T}$ and an instance $\bvec{x} \in \mathcal{X}$, if $||\bvec{x}|| \leq b$, i.e., instances are norm bounded, then $\bvec{T}$ is of $2b\max(|\prod_{i=1}^{L}\sigma^{i}_{\max} - 1|, |\prod_{i=1}^{L}\sigma^{i}_{\min} - 1|)$-isometry w.r.t. the variation subspace $\mathcal{X} - \mathcal{N}(\bvec{T})$ of the instance $\bvec{x}$. \footnote{Note that the space $\mathcal{X} - \mathcal{N}(\bvec{T})$ does not depend on $\bvec{x}$, and this is a trivial case where the variation subspaces of all instances are the same. We will see nontrivial cases later when dealing with nonlinear neural networks.}. We also have $\forall \bvec{x}' \in \{\bvec{z} \ |\ \bvec{z} - \bvec{x} \in \mathcal{P}_x\}$, $\rho_{Q}(\bvec{T}\bvec{x}, \bvec{T}\bvec{x}') \leq\rho_{P}(\bvec{x}, \bvec{x}')+2b|\prod_{i=1}^{L}\sigma^{i}_{\max} - 1|$ and $\rho_{Q}(\bvec{T}\bvec{x}, \bvec{T}\bvec{x}') \geq\rho_{P}(\bvec{x}, \bvec{x}')+\rho_{P}(\bvec{x}, \bvec{x}')(\prod_{i=1}^{L}\sigma^{i}_{\min} - 1)$.
\end{lemma}

See the proof in appendix \ref{ProofLinearNNDeltaIsometry}. %  in the supplementary material.

The above lemma shows that as long as $\bvec{x}$ varies within the complement of the
null space of a linear DNN $\bvec{T}$, we can constrain variations induced by the mapping by the specified $\delta$. Outside the space, it is
unlikely of practical interest since $\bvec{T}$ discards all the information about the variations.

To further proceed for nonlinear DNNs, we introduce some terminologies from hyperplane arrangement
\cite{orlik1992arrangements}, and give definitions to describe the
objects of interest exactly.

\begin{definition}[(Finite) Hyperplane arrangement]
  A finite hyperplane arrangement $\mathcal{A}$ is a finite set of affine hyperplanes in some
  vector space $\mathcal{X} \equiv \mathbb{K}^{n}$, where $\mathbb{K}$ is a field and
  is taken as $\mathbb{R}$ in this paper.
\end{definition}

\begin{definition}[Region]
  Denote $H \in \mathcal{A}$ an element of the arrangement, a region of the
  arrangement is a connected component of the complement $\mathbb{R}^{n} - \bigcup\limits_{H \in \mathcal{A}}H$. The set of all regions is
  denoted as $\mathcal{R}(\mathcal{A})$, shortened as $\mathcal{R}$ when no confusion exists.
\end{definition}

We set up a labeling scheme for $r \in \mathcal{R}$. Choosing a linear order in
$\mathcal{A}$, we write $\mathcal{A} = \{H_i\}_{i=1,\ldots, n}$ and $H_i = \ker
\alpha_i$, where $\ker$ denotes the kernel $\{\bvec{x} \in \mathcal{X} | \alpha_i(\bvec{x}) = \langle \bvec{\alpha}_i, \bvec{x} \rangle = 0 \}$ and $\bvec{\alpha}_i$ is the normal of hyperplane $H_i$. Let $\mathbb{J} = \{1, 0\}$, and
$\pi_i: \mathbb{J}^{n} \rightarrow \mathbb{J}$ be the projection onto the $i$-th coordinate. Define
a map $\tau: \mathcal{X} \rightarrow \mathbb{J}^{n}$ by
\begin{displaymath}
  \pi_i\tau(\bvec{x}) =
  \begin{cases}
    1 & \text{if $\alpha_i(\bvec{x}) > 0$}\\
    0 & \text{if $\alpha_i(\bvec{x}) \leq 0$} .
  \end{cases}
\end{displaymath}
With the scheme, for any $\mathcal{A}$, we would have an index set
$\mathcal{T}(\mathcal{A}, \tau)$, shortened as $\mathcal{T}$, such that for any $r \in \mathcal{R}$, it
corresponds to a unique element in $\mathbb{J}^{n}$, denoted as $\tau(r)$. We will
use $\tau(\bvec{x})$ ---- labeling on elements, and $\tau(r)$ --- labeling on regions, interchangely.

\begin{definition}[Neuron]
  A neuron $a_{lk}$ of a neural network $\bvec{T}$ is a functional defined by
  \begin{displaymath}
    a_{lk}(\bvec{x}) = \pi_k g(\bvec{W}_lg(\bvec{W}_{l-1}\ldots g(\bvec{W}_1\bvec{x}))) ,
  \end{displaymath}
  where $l \in \{1,\ldots, L\}$ and $k \in \{1,\ldots, n_l\}$.
  All the neurons at layer $l$ define a map, denoted as
  \begin{displaymath}
    a_{l}(\bvec{x}) =  g(\bvec{W}_lg(\bvec{W}_{l-1}\ldots g(\bvec{W}_1\bvec{x}))) .
  \end{displaymath}
\end{definition}

The lemma that follows is mostly an analysis of the domain of a DNN $\bvec{T}$. We begin with the following definition.

\begin{definition}[Support of Neuron/DNN]
  Given a neuron $a_{lk}$, the support of $a_{lk}$ is the set of instances in $\mathcal{X}$ that satisfy
  \begin{displaymath}
    \text{supp}(a_{lk}) = \{ \bvec{x} \in \mathcal{X} | a_{lk}(\bvec{x}) \not= 0\} .
  \end{displaymath}
  Similarly, the support of a neural network $\bvec{T}$ is the set of instances in $\mathcal{X}$ that satisfy
  \begin{displaymath}
    \text{supp}(\bvec{T}) = \{ \bvec{x} \in \mathcal{X} | \bvec{T}\bvec{x} \not= 0\} .
  \end{displaymath}
\end{definition}

\begin{lemma}
  \label{lm:2}
 A nonlinear neural network $\bvec{T}$ divides $\mathcal{X}$ into a set of regions
$\mathcal{Q}$, and within each region $q \in \mathcal{Q}$, $\bvec{T}$ is linear w.r.t. variations of instances as long as they vary within $q$. We denote the linear mapping at $q$ as $\bvec{T}_{q}$ and have
\begin{align*}
  \bvec{T}_{q} &= \prod\limits_{i=1}^{L}\bvec{W}^{q}_{i}\\
  \bvec{W}^{q}_{l} &= \diag(\tau_l(q))\bvec{W}_{l} ,
\end{align*}
where when layer $l$ does not contain max pooling, and $\tau_l(q)$ is defined as
\begin{displaymath}
  \pi_k\tau_l(q) = \pi_k\tau^{\relu}_{l}(q) =
  \begin{cases}
    1 & \text{if $a_{lk}(\bvec{x}) > 0, \forall \bvec{x} \in q$}\\
    0 & \text{if $a_{lk}(\bvec{x}) \leq 0, \forall \bvec{x} \in q$}.
  \end{cases}
\end{displaymath}
When layer $l$ does contain max pooling, we define $\bvec{W}_{l}^{q}$ as
\begin{displaymath}
  \bvec{W}^{q}_{l} = \bvec{P}_l\diag(\tau_l(q))\bvec{W}_{l}
\end{displaymath}
, and
$\tau_l(q)$ is defined as
\begin{displaymath}
  \tau_{l}(q) = \tau^{\max}_{l}(q)\tau^{\relu}_{l}(q).
\end{displaymath}
In this above definition $\tau^{\relu}$ is defined as before, and
\begin{displaymath}
  \pi_{k}\tau^{\max}_{l}(q) =
  \begin{cases}
    1 & \text{if $k=\argmax\limits_{k \in K} a_{lk}(\bvec{x}), \forall \bvec{x} \in q$}\\
    0 & \text{otherwise} ,
  \end{cases}
\end{displaymath}
where $K$ is the set of indices of neurons being pooled over; $\bvec{P}_l$ is defined as (layer index suppressed)
  \begin{displaymath}
    \bvec{P}_{ik} =
  \begin{cases}
    1 & \text{if $k$ is the index that is pooled over by $i^{\text{th}}$ pooling area}\\
    0 & \text{otherwise}.
  \end{cases}
  \end{displaymath}
  $\bvec{P}_l$ could be understood as a matrix that for each pooling area, it
sums over the dimension/area being pooled, and since only one dimension of the
area is nonzero (due to $\tau_l(q)$), it outputs the maximal value. For clarity and convenience, we would use the definition without max pooling in
discussion, and note that all the results present are proved for both definitions.

\end{lemma}

See the proof in appendix \ref{ProofNonlinearNNLocalLinearity}. %  in the supplementary material.

\begin{remark}
  The function $\tau_l$ is intuitively a selection function that sets some
  rows of $\bvec{W}_l$ to zero, and selects a submatrix from it.
\end{remark}

\begin{remark}
  For $q \in \mathcal{Q}$ and $q \not\in \supp(\bvec{T})$, by the definition of linearity, $\bvec{T}$
is still linear over $q$, i.e., the special case of $\bvec{T}\bvec{x} = 0, \forall \bvec{x} \in q$. In this
case, $\bvec{x} \in \mathcal{N}(\bvec{T}_{q})$.
\end{remark}

In the proof of appendix \ref{ProofNonlinearNNLocalLinearity} for \cref{lm:2}, we prove that within each region $q \in Q$, each neuron of a DNN is a linear
functional. We summarize the result in the following corollary.
\begin{corollary}
  \label{col:2}
  A nonlinear neural network $\bvec{T}$ divides $\mathcal{X}$ into a set of regions
  $\mathcal{Q}$, and within each region $q \in \mathcal{Q}$, the $k^{th}$ neuron $a_{lk}$ of the layer $l$ is linear w.r.t. variations of $\bvec{x}$ within $q$, and we have
\begin{align*}
  a^q_{lk} &= \pi_{k}\prod\limits_{i=1}^{l}\bvec{W}^{q}_{i} .
\end{align*}
\end{corollary}

With the above lemma, we define the behavior of a DNN at a local area around $\bvec{x} \in \mathcal{X}$ or a local area around a set $B \subset \mathcal{X}$ as below.

\begin{definition}[Linear neural network and neuron induced at $\bvec{x} \in \mathcal{X}$ from a
  nonlinear neural network]
  \label{def:induced_nn}
  For any given $\bvec{x} \in \mathcal{X}$ with $\bvec{x} \in q \in \mathcal{Q}$, we call the linear
  neural network $\bvec{T}_{q}$ the linear neural network induced by a nonlinear
  neural network $\bvec{T}$ at $\bvec{x}$ --- denoting it as $\bvec{T}_{|\bvec{x}}$, the linear neuron $a^q_{lk}$
  the linear neuron induced by the nonlinear neuron $a_{lk}$ at $\bvec{x}$ --- denoting
  it as $a_{lk|\bvec{x}}$, and the submatrix of weight matrix of each layer $l$ the submatrix induced
  by nonlinearity --- denoting it as $\bvec{W}_{l|\bvec{x}}$.
\end{definition}

\begin{definition}[Linear neural network and neuron induced at subset $B \subset \mathcal{X}$ from a
  nonlinear neural network]
  \label{def:induced_nn_set}
  For any given $B \subset \mathcal{X}$ with $B \subset q \subset \mathcal{Q}$, we call the linear
  neural network $\bvec{T}_{q}$ the linear neural network induced by a nonlinear
  neural network $\bvec{T}$ at $B$ --- denoting it as $\bvec{T}_{|B}$, the linear neuron $a^q_{lk}$
  the linear neuron induced by the nonlinear neuron $a_{lk}$ at $B$ --- denoting
  it as $a_{lk|B}$, and the submatrix of weight matrix of each layer $l$ the submatrix induced
  by nonlinearity --- denoting it as $\bvec{W}_{l|B}$.
\end{definition}

\begin{lemma}
  \label{lm:3}
  For any nonlinear neural network $\bvec{T}$ of $L$ layers, a covering set for $\mathcal{X}$ can be found with a diameter $\gamma = o(S_m, \bvec{T})/\left(\prod\limits_{i=1}^{l(S_m, \bvec{T})}\sigma^{i}_{\max}\right) >
  0$, such that for any given $\bvec{x} \in S_m^{(x)}$ and $\{\bvec{x}' \in \mathcal{X} |\: \|\bvec{x} - \bvec{x}'\| \leq \gamma\}$, $\bvec{T}\bvec{x} - \bvec{T}\bvec{x}' = \bvec{T}_{|\bvec{x}}(\bvec{x} - \bvec{x}')$, where $o(S_m, \bvec{T})$ is a value depending on the training data and network weights, so is $l(S_m, \bvec{T})$ with $1 \leq l(S_m, \bvec{T}) \leq L$ (the dependence is specified in the proof), and $\sigma^{i}_{\max}$ is the maximum singular value of weight matrix $\bvec{W}_i$ of the $i^{th}$ layer.
\end{lemma}

See the proof in appendix \ref{ProofGammaCover4DNN}. % in the supplementary material.

We come with the local isometry property of DNNs after one more definition.

\begin{definition}[$\gamma$-cover $\delta$-isometry w.r.t. variation subspace of instance]
  Given a map $\bvec{T}$ that maps a metric space $(\mathcal{P}, \rho_{P})$
to another metric space $(\mathcal{Q}, \rho_{Q})$, and an instance $\bvec{x} \in \mathcal{P}$, it is called
$\gamma$-cover $\delta$-isometry w.r.t. variation space $\mathcal{P}_{\bvec{x}}$ of $\bvec{x}$, if a $\gamma$-cover exists such that the following
inequality holds
\begin{displaymath}
  \forall \bvec{x}' \in \{\bvec{z}\ |\ \bvec{z} - \bvec{x} \in \mathcal{P}_{\bvec{x}}, \bvec{z} \in B\}, |\rho_{Q}(\bvec{T}\bvec{x}, \bvec{T}\bvec{x}') - \rho_{P}(\bvec{x}, \bvec{x}')| \leq \delta ,
\end{displaymath}
where $B$ denotes a ball given by the $\gamma$-cover.
\end{definition}

\begin{lemma}
  \label{lm:4}
  Given a nonlinear neural network $\bvec{T}$, if $||\bvec{x}|| \leq b$ $\forall \bvec{x} \in \mathcal{X}$, i.e., instances are norm bounded,
  then  $\bvec{T}$ is of $\gamma$-cover $\delta$-isometry w.r.t. $\mathcal{X} - \mathcal{N}(\bvec{T}_{|\bvec{x}})$ of $\bvec{x} \in S_m^{(x)}$, where $\delta$ and $\gamma$ are respectively specified in \cref{lm:1} and \cref{lm:3}.
\end{lemma}

See the proof in appendix \ref{ProofGammaCoverDeltaIsometry4DNN}. % in the supplementary material.

\begin{example}
  In a nonlinear DNN $\bvec{T}$ as defined in \cref{sec:neural-network}, given
  an instance $\bvec{x}$, it is $\delta$-isometry w.r.t. variation subspace
  $\mathcal{X} - \mathcal{N}(T_{|\bvec{x}})$. Note that for $\bvec{x}$ of different instances, $T_{|\bvec{x}}$ is potentially different. In practice, the singular values of weight matrices of the induced linear DNN $\bvec{T}_{|\bvec{x}}$ only
  relate to the variation in the space $\mathcal{X} - \mathcal{N}(T_{|\bvec{x}})$. The bound
  given in the following will establish the relationship between
  generalization errors and singular values of weight matrices of DNNs by characterizing the constraints that DNNs impose on the variation in this space.
\end{example}

\subsection{Main results of generalization bound}
\label{sec:mainresults}

With the local $\delta$-isometry property of DNNs established, we derive in this section our main results of generalization bound.

\begin{theorem}
  \label{thm:ge_dnn}
  Given a CRL problem, the algorithm to learn is a nonlinear neural network of $L$ layers,
denoted as $\bvec{T}$. Suppose the following assumptions hold: 1) $||\bvec{x}|| \leq b \
\forall \ \bvec{x} \in \mathcal{X}$, i.e., instances are norm bounded; 2) the loss
function $\mathcal{L}$ is bounded, a.k.a. $\forall z \in \mathcal{Z}, \mathcal{L}(f(\bvec{x}), y) \leq M$, and
the Lipschitz constant of $\mathcal{L} \circ f$ w.r.t $\bvec{Tx}$ is bounded by $A$; 3)
$\mathcal{X}$ is a regular $k$-dimensional manifold with a covering number
$(\frac{C_{\mathcal{X}}}{\gamma/2})^{k}$; 4) within each covering ball $B$ of
$\mathcal{X}$ that contains $\bvec{x} \in S_m^{(x)}$, $\bvec{x} - \bvec{x}' \in
\mathcal{X} - \mathcal{N}(T_{|B}) \ \forall \ \bvec{x}, \bvec{x}' \in B$. Then, for any $\nu > 0$, with probability at least $1 - \nu$ we have
\begin{align*}
  \GE(f_{S_m}) & \leq A(\gamma+\delta')+ M\sqrt{\frac{\log(2)2^{k+1}|\mathcal{Y}|C_{\mathcal{X}}^{k}}{\gamma^{k}m} + \frac{2 \log(1/\nu)}{m}}
\end{align*}
with
  \begin{displaymath}
    \delta' = 2b|\prod_{i=1}^{L}\sigma^{i}_{\max} - 1|,
    \gamma = \frac{o(S_m, \bvec{T})}{\prod\limits_{i=1}^{l(S_m, \bvec{T})}\sigma^{i}_{\max}} ,
  \end{displaymath}
where $o(S_m, \bvec{T})$ and $1 \leq l(S_m, \bvec{T}) \leq L$ are values depending on the training set $S_m$ and learned network $\bvec{T}$.
\end{theorem}

A proof sketch is provided below, and the full proof is given in Appendix \ref{proof-thm-ge_dnn}.
\begin{proof}
  By \cref{lm:4}, we have that $\bvec{T}$ is of $\gamma$-cover $\delta$-isometry
w.r.t. variation space of each training instance. The expansion property of $\delta$-isometry gives $\rho_{Q}(\bvec{T}\bvec{x}, \bvec{T}\bvec{x}') \leq\rho_{P}(\bvec{x},
\bvec{x}')+2b|\prod_{i=1}^{L}\sigma^{i}_{\max} - 1|$. By
\cref{isometry-robustness}, we have that DNNs are $(|\mathcal{Y}| 2^k
C_{\mathcal{X}}^k/\gamma^k, A(\gamma + 2b|\prod_{i=1}^{L}\sigma^{i}_{\max} -
1|))$-robust. Note that the proof differs from \cref{isometry-robustness}
subtly, for that the loss difference needs to stay in the variation space of
each covering ball. Since the tricky part is also present in the proof of
\cref{thm:ge_contraction} (the condition $\bvec{x}' - \bvec{x}_j \in
\mathcal{P}_{\bvec{x}_j}$ in \cref{t:5} ), to avoid tautology, we do not write the full proof
here. The proof is finished by applying the robustness conclusion into
\cref{thm:1}.
\end{proof}

Regarding the $4^{th}$ assumption in \cref{thm:ge_dnn}, it intuitively states that the local variation of interest w.r.t. each $\bvec{x} \in S_m^{(x)}$ falls in the space $\mathcal{X} - \mathcal{N}(\bvec{T}_{|B})$. Denote $\bvec{v} = \bvec{x} - \bvec{x}'$, we have $\bvec{T}\bvec{v} = \bvec{T}_{|B}\bvec{v} = 0$ if $\bvec{v} \in \mathcal{N}(\bvec{T}_{|B})$, i.e., the variation vanishes after passing through the network. In practice, we are not interested in such a trivial case of vanishing local variations. Instead, it is the variation in the complement $\mathcal{X} - \mathcal{N}(\bvec{T}_{|B})$ that we want to constrain.

We have assumed that $\mathcal{X}$ is a regular $k$-dimensional manifold, whose covering
number is $(\frac{C_{\mathcal{X}}}{\gamma/2})^{k}$, where
$C_{\mathcal{X}}$ is a constant that captures the ``intrinsic'' properties of
$\mathcal{X}$, and $\gamma$ is the diameter of the covering ball. Such an assumption is
general enough to accommodate at least visual data such as natural images and has been widely used \cite{Verma2013}.

In \cref{thm:ge_dnn}, $|\mathcal{Y}|(\frac{C_{\mathcal{X}}}{\gamma/2})^{k}$ corresponds the covering number of the joint space $\mathcal{X}\times\mathcal{Y}$. We now show that with proper use of the contraction property of isometric mapping of DNNs,
the constant $|\mathcal{Y}|$ in the second term of the upper bound can be
removed while a modified first term has the potential to be small as well,
indicating a better bound. Our idea is to directly exploit the
covering number of the instance space $\mathcal{X}$ and measure the differences
of the loss function $\mathcal{L}(f(\bvec{Tx}),y)$ w.r.t. both arguments.  Such
a measure deals with instances of different labels but are close enough in
$\mathcal{X}$, thus characterizing both the errors that are caused by
erroneously contracting the distance between instances from different classes and
erroneously expanding the distance between instances from the same classes,
instead of those of the same classes alone. However, it will cause the issue of
infinite Lipschitz constant of loss function ${\cal{L}}(f(\bvec{x}), y)$. To see this, consider a binary classification problem whose loss function $\mathcal{L}$ is
\begin{displaymath}
  \mathcal{L}(f(\bvec{x}), y) = - \mathbf{1}_{y=1}\log f(\bvec{x}) - \mathbf{1}_{y=0}\log (1 - f(\bvec{x})) ,
\end{displaymath}
where $(\bvec{x}, y)$ is an example, $f(\bvec{x})$ is a function that maps $\bvec{x}$ to
probability, and $\mathbf{1}$ is an indicator function. We provide an example case to illustrate the influence of metrics on $\mathcal{Z}$.

\noindent\emph{Case}. Let the metric on $\mathcal{Z}$ be $\rho((\bvec{x}, y), (\bvec{x}', y')) = ||(\bvec{x}-\bvec{x}', y-y')|| = ||\bvec{x}-\bvec{x}'|| + |y-y'|$. Suppose that we have a
pair of examples $(\bvec{x}, y = 1)$ and $(\bvec{x}, y' = 0)$ that only differ in labels, we have
\begin{displaymath}
  A \geq  \frac{|\mathcal{L}(f(\bvec{x}), y) - \mathcal{L}(f(\bvec{x}), y')|}{||(\bvec{x}, y) - (\bvec{x}, y')||} = \frac{\log (f(\bvec{x})/(1 - f(\bvec{x})))}{1} .
\end{displaymath}
When there exists an $\bvec{x}$ such that $f(\bvec{x}) \rightarrow 1$, we have $|\mathcal{L}(f(\bvec{x}), y) - \mathcal{L}(f(\bvec{x}), y')| \rightarrow +\infty$. This happens because as $y$ changes values, due to its
discreteness, it could induce a jump discontinuity on $\mathcal{L}(f(\bvec{x}), y)$, even though
$\mathcal{L}(f(\bvec{x}), y)$ is Lipschitz continuous w.r.t. $\bvec{x}$. To avoid this, \cite{RobustnessAndGeneralization} and \cite{DiscrimRobustTransform} employ a large covering number to ensure that examples in the same ball have the same label.

We note that derivation of generalization bounds for robust algorithms concerns with the loss difference $|\mathcal{L}(f(\bvec{x}), y) - \mathcal{L}(f(\bvec{x}'), y')|$ between example pairs. To address the aforementioned issue, we consider two separate cases for the loss difference: the cases of $y = y'$ and $y
\not= y'$. For the case $y = y'$, we exploit the bounded Lipschitz constant of $\mathcal{L} \circ f$ w.r.t. $\bvec{Tx}$. For the case $y \not= y'$, we introduce the following {\it pairwise error function} to characterize the loss difference.

\begin{definition}[Pairwise error function]
Given a CRL problem, of which $\mathcal{L}$ is bounded for any compact set in
$\mathcal{Z}$, a.k.a. for $z$ in any compact subset of $\mathcal{Z}$, $\mathcal{L}(f(\bvec{x}), y) \leq M$,
$\mathcal{X}$ is a regular $k$-dimensional manifold with a $\gamma$-cover, and $\bvec{T}$ is of
$\gamma$-cover $\delta$-isometry, a
pairwise error function (PE) of the tuple $(\mathcal{L}, f, \bvec{T}, \mathcal{Z}, \gamma)$ is defined as
\begin{eqnarray}
  \PE(\delta) = \max\limits_{z = (\bvec{x}, y) \in \mathcal{Z}}\max\limits_{z' \in D}|\mathcal{L}(f(\bvec{T}\bvec{x}), y) - \mathcal{L}(f(\bvec{T}\bvec{x}'), y')| \nonumber
\end{eqnarray}
with $ D  =  \{z' = (\bvec{x}', y') \in \mathcal{Z} \ | \ \gamma - \delta \leq
||\bvec{T}\bvec{x}' - \bvec{T}\bvec{x}|| \leq \gamma + \delta, ||\bvec{x} - \bvec{x}'|| \leq \gamma\}$.

\noindent It characterizes the largest loss difference for examples in $\mathcal{Z}$ that
may arise due to the contraction and expansion properties of $\delta$-isometry mapping. Note that
$\PE(\delta)$ is a monotonously increasing function of $\delta$
--- a larger $\delta$ means more feasible
examples in $\mathcal{X}$ and possibly larger distance contraction/expansion, leading to
a possibly larger value of $\PE(\delta)$.
\end{definition}

\begin{theorem}
  \label{thm:ge_contraction}
  Given a CRL problem, the algorithm to learn is a nonlinear neural network of $L$ layers, denoted as $\bvec{T}$.
  Suppose the following assumptions hold: 1) $||\bvec{x}|| \leq b \
\forall \ \bvec{x} \in \mathcal{X}$, i.e., instances are norm bounded; 2) the loss
function $\mathcal{L}$ is bounded, a.k.a. $\forall z \in \mathcal{Z}, \mathcal{L}(f(\bvec{Tx}), y) \leq M$, and
the Lipschitz constant of $\mathcal{L} \circ f$ w.r.t $\bvec{Tx}$ is bounded by $A$; 3)
$\mathcal{X}$ is a regular $k$-dimensional manifold with a covering number
$(\frac{C_{\mathcal{X}}}{\gamma/2})^{k}$; 4) within each covering ball $B$ of
$\mathcal{X}$ that contains $\bvec{x} \in S_m^{(x)}$, $\bvec{x} - \bvec{x}' \in
\mathcal{X} - \mathcal{N}(\bvec{T}_{|B}) \ \forall \ \bvec{x}, \bvec{x}' \in B$.
Then, for any $\nu > 0$, with probability at least $1 - \nu$ we have
\begin{align}
  \GE(f_{S_m}) & \leq \max\{A(\gamma + \delta), \PE(\delta)\} \label{eq:bound_1}\\
         & + M\sqrt{\frac{\log(2)2^{k+1}C_{\mathcal{X}}^{k}}{\gamma^{k}m} + \frac{2 \log(1/\nu)}{m}} \label{eq:bound_2},
\end{align}
where $\delta = 2b\max(|\prod_{i=1}^{L}\sigma^{i}_{\max} - 1|, |\prod_{i=1}^{L}\sigma^{i}_{\min} - 1|)$, and $\gamma$ is the same as that of
\cref{thm:ge_dnn}.
\end{theorem}

\begin{proof}
  Similar to the proof of \cref{thm:1}, we partition the space $\mathcal{Z}$
via the assumed $\gamma$-cover. Since $\mathcal{X}$ is a $k$-dimensional manifold,
its covering number is upper bounded by $C_{\mathcal{X}}^{k} / (\gamma/2)^{k}$. Let $K$ be
the overall number of covering set, which is upper bounded by $C_{\mathcal{X}}^{k} / (\gamma/2)^{k}$. Denote $C_{i}$ the
$i^{th}$ covering ball
and let $N_{i}$ be the set of indices of training examples that fall into
$C_{i}$. Note that $(|N_{i}|)_{i=1,\ldots, K}$ is an IDD multimonial random variable with
parameters $m$ and $(|\mu(C_{i})|)_{i=1,\ldots, K}$. Then
  \begin{align}
    &|R(f\bvec{T}) - R_m(f\bvec{T})| \nonumber \\
    = & | \sum\limits_{i=1}^{K}\mathbb{E}_{z \sim \mu}[\mathcal{L}(f(\bvec{Tx}), y)|z\in C_i]\mu(C_i) -
        \frac{1}{m}\sum\limits_{i=1}^{m}\mathcal{L}(f(\bvec{Tx}_i), y_i) | \nonumber \\
    \leq &  | \sum\limits_{i=1}^{K}\mathbb{E}_{z \sim \mu}[\mathcal{L}(f(\bvec{Tx}), y)|z\in C_i]\frac{|N_i|}{m} -
        \frac{1}{m}\sum\limits_{i=1}^{m}\mathcal{L}(f(\bvec{Tx}_i), y_i) | \nonumber \\
      & +  | \sum\limits_{i=1}^{K}\mathbb{E}_{z \sim \mu}[\mathcal{L}(f(\bvec{Tx}), y)|z\in C_i]\mu(C_i) \nonumber \\
      & \ \ \ \ \ \ \ \ \ \ \ \ \ \ \ \ - \sum\limits_{i=1}^{K}\mathbb{E}_{z \sim \mu}[\mathcal{L}(f(\bvec{Tx}), y)|z\in C_i]\frac{|N_i|}{m} | \nonumber \\
    \leq & | \frac{1}{m}\sum\limits_{i=1}^{K}\sum\limits_{j\in N_i}\max\limits_{z' \in
        C_i, \bvec{x}' - \bvec{x}_j \in \mathcal{P}_{\bvec{x}_j}}|\mathcal{L}(f(\bvec{Tx}'), y') - \mathcal{L}(f(\bvec{Tx}_j), y_j)| | \label{t:5} \\
    & \ \ \ \ \ \ \ \ \ + | \max\limits_{z \in \mathcal{Z}}|\mathcal{L}(f(\bvec{Tx}),
      y)|\sum\limits_{i=1}^{K}|\frac{|N_i|}{m} - \mu(C_i)| | \label{t:6} .
  \end{align}
  Remember that $z = (\bvec{x}, y)$. We consider the two cases of $y'=y_{j}$ and $y'\not=y_{j}$ to bound \cref{t:5}.

  When $y'=y_{j}$, by the assumption that $\bvec{T}$ is of $\gamma$-cover
  $\delta$-isometry w.r.t. $\mathcal{P}_{\bvec{x}}$ of $\bvec{x} \in S_m^{(x)}$ and
  the Lipschitz constant of $\mathcal{L} \circ f$ w.r.t. $\bvec{Tx}$ is $A$, suppose the maximum is achieved at $\bvec{x}_k$ and
  $\bvec{x}_k \in C_p$, we have
  \begin{align}
    & \max\limits_{z' \in C_p, \bvec{x}' - \bvec{x}_k \in \mathcal{P}_{\bvec{x}_k}}|\mathcal{L}(f(\bvec{Tx}'), y') - \mathcal{L}(f(\bvec{Tx}_k),
      y_k)|\nonumber \\
    \leq & A\max\limits_{z' \in C_p, \bvec{x}' - \bvec{x}_k \in \mathcal{P}_{\bvec{x}_k}}||\bvec{T}_{|\bvec{x}_k}(\bvec{x}' - \bvec{x}_{k})||\label{t:7}\\
    \leq & A\max\limits_{z' \in C_p, \bvec{x}' - \bvec{x}_k \in \mathcal{P}_{\bvec{x}_k}}(||\bvec{x}' - \bvec{x}_{k}|| + 2b|\prod_{i=1}^{L}\sigma^{i}_{\max} - 1|)\label{t:8}\\
    \leq & A(\gamma + 2b|\prod_{i=1}^{L}\sigma^{i}_{\max} - 1|) \nonumber \\
    \leq & A(\gamma + 2b\max(|\prod_{i=1}^{L}\sigma^{i}_{\max} - 1|, |\prod_{i=1}^{L}\sigma^{i}_{\min} - 1|)) , \nonumber
  \end{align}
  where the second inequality holds since the $\gamma$-cover $\delta$-isometry of $\bvec{T}$ also gives $\rho_{Q}(\bvec{T}\bvec{x}, \bvec{T}\bvec{x}') \leq \rho_{P}(\bvec{x}, \bvec{x}') + 2b|\prod_{i=1}^{L}\sigma^{i}_{\max} - 1|$.

  When $y'\not=y_{j}$, given any training $\bvec{x}_k$, we have
  \begin{align}
    & \max\limits_{z' \in C_p, \bvec{x}' - \bvec{x}_k \in \mathcal{P}_{\bvec{x}_k}}|\mathcal{L}(f(\bvec{Tx}'), y') - \mathcal{L}(f(\bvec{Tx}_k),
      y_k)| \nonumber\\
    & =  \max\limits_{z' \in D_{z_k}}|\mathcal{L}(f(\bvec{Tx}'), y') - \mathcal{L}(f(\bvec{Tx}_k),
      y_k)| \label{t:9}\\
    & \leq  \PE(2b\max(|\prod_{i=1}^{L}\sigma^{i}_{\max} - 1|, |\prod_{i=1}^{L}\sigma^{i}_{\min} - 1|)) , \nonumber
  \end{align}
  where the inequality holds since by $\gamma$-cover $\delta$-isometry of $\bvec{T}$, we have $\rho_{Q}(\bvec{T}\bvec{x},
\bvec{T}\bvec{x}') \leq\rho_{P}(\bvec{x},
\bvec{x}') + \delta$ and
$\rho_{Q}(\bvec{T}\bvec{x}, \bvec{T}\bvec{x}') \geq\rho_{P}(\bvec{x},
\bvec{x}') - \delta$, with
\begin{equation}
\delta = 2b\max(|\prod_{i=1}^{L}\sigma^{i}_{\max} - 1|, |\prod_{i=1}^{L}\sigma^{i}_{\min} - 1|) \nonumber.
\end{equation}

  Thus \cref{t:5} is less than or equal to $\max\{A(\gamma + \delta), \PE(\delta)\}$.
  By Breteganolle-Huber-Carol inequality, \cref{t:6} is less than or equal to
  $M\sqrt{\frac{\log(2)2^{k+1}C_{\mathcal{X}}^{k}}{\gamma^{k}m} + \frac{2
      \log(1/\nu)}{m}}$.

  The proof is finished.
\end{proof}

We now specify a case where the obtained bound in \cref{thm:ge_contraction} is tighter than that in \cref{thm:ge_dnn}; for example, in the ball that covers $(\bvec{x}_k,y_k)$, few examples with $y\neq y_k$
are misclassified. In this case, $\PE(2b\max(|\prod_{i=1}^{L}\sigma^{i}_{\max} - 1|,
|\prod_{i=1}^{L}\sigma^{i}_{\min} - 1|)\leq A(\gamma +
2b|\prod_{i=1}^{L}\sigma^{i}_{\max} - 1|)$, and the covering number is shrunken
by a factor of $\sqrt{|\mathcal{Y}|}$. Our result only incrementally
improves the generalization bound. However, it clearly shows that the contraction property of isometric mapping plays an important role in bounding the generalization error.

\subsection{Suggestion of new algorithms}
\label{sec:DiscussMainResultAndSuggestion}

Many quantities exist in the GE bound established in \cref{thm:ge_contraction}. Except $\gamma$ and $\delta$, all others are independent of the neural network. Although both $\gamma$ and $\delta$ are controlled by singular values of weight matrices, we note that $\gamma$, which specifies the size of covering balls for a covering of $\mathcal{X}$, is more of a trade-off parameter that balances between the first and second term of the GE bound, than of a variable used to control GE, as long as its values satisfy the condition of $\gamma \leq o(S_m, \bvec{T})/\left(\prod\limits_{i=1}^{l(S_m, \bvec{T})}\sigma^{i}_{\max}\right)$ established in \cref{lm:3}.

To control the bound via $\delta$, we note that the minimum value of the bound w.r.t. $\delta$ is achieved when $\delta = 0$, which implies $\prod_{i=1}^{L}\sigma^i_{\max} = 1$ and $\prod_{i=1}^{L}\sigma^i_{\min} = 1$. In the following lemma, we show that the condition is achieved only when $\sigma^i_{\max} = \sigma^{i}_{\min}$, $\forall i = 1, \ldots, L$.
\begin{lemma}
  \label{lm:equalSinValue}
  In \cref{thm:ge_contraction}, $\delta = 0$ is achieved only when
  \begin{displaymath}
    \sigma^{i}_{\max} = \sigma^{i}_{\min}, \ \forall i = 1, \ldots, L ,
  \end{displaymath}
  \begin{displaymath}
  \prod_{i=1}^{L}\sigma^i_{\max} = 1, \prod_{i=1}^{L}\sigma^i_{\min} = 1 .
  \end{displaymath}
\end{lemma}

\begin{proof}
  It is straightforward to see that $\delta = 0$ i.f.f. $\prod_{i=1}^{L}\sigma^i_{\max} = 1$
and $\prod_{i=1}^{L}\sigma^i_{\min} = 1$. In the following, we show the two conditions
hold only when $\sigma^i_{\max} = \sigma^{i}_{\min}$, $\forall i = 1, \ldots ,L$.

  For any $i \in \{1, \ldots, L\}$, we reparameterize $\sigma^{i}_{\min}$ as $\sigma^{i}_{\min} = \alpha_i\sigma^{i}_{\max}$. It is
  clear that $\alpha_i \in (0, 1]$.

  Since $\prod_{i=1}^{L}\sigma^i_{\min} = \prod_{i=1}^{L}\sigma^i_{\max} =1$, we have
  \begin{equation}
    1=\prod_{i=1}^{L}\sigma^i_{\min} =\prod_{i=1}^{L}\alpha_i\sigma^i_{\max}=\prod_{i=1}^{L}\alpha_i .
    \label{eq:s_q}
  \end{equation}
  Notice that $\alpha \in (0, 1]$, thus, we have $0 < \prod_{i=1}^{L}\alpha_i \leq 1$. To have
  \cref{eq:s_q}, we need $\alpha_i = 1, \forall i = 1, \ldots, L$, indicating $\sigma^{i}_{\min} = \sigma^{i}_{\max}, \forall i = 1, \ldots, L$.
\end{proof}

We show in \cref{lm:equalSinValue} that the optimal GE bound of \cref{thm:ge_contraction} w.r.t. $\delta$ is achieved only when all singular values of each of weight matrices of a DNN are equal. Among various solutions, the most straightforward one is that all singular values are equal to $1$; in other words, each weight matrix has orthonormal rows or columns. This inspires a new set of algorithms that we generally term as \emph{Orthogonal Deep Neural Networks (OrthDNNs)}.

\section{Algorithms of Orthogonal Deep Neural Networks}
\label{SecAlgms}

In this section, we first present the algorithm of \emph{strict OrthDNNs} by enforcing strict orthogonality of weight matrices during network training. It amounts to optimizing weight matrices on their respective Stiefel manifolds, which however, is computationally prohibitive for large-sized networks. To achieve efficient OrthDNNs, we propose a novel algorithm called Singular Value Bounding (SVB), which achieves \emph{approximate OrthDNNs} via a simple scheme of hard regularization. We discuss alternative schemes of soft regularization for approximate OrthDNNs, and compare with our proposed SVB. Batch Normalization \cite{BatchNorm} is commonly used to accelerate training of modern DNNs, yet it has a potential risk of ill-conditioned layer transform, causing its incompatibility with OrthDNNs. In fact, direct use of BN in OrthDNNs makes it ineffective to enforce strict orthogonality of weight matrices. We propose Degenerate Batch Normalization (DBN) to enable its use with strict OrthDNNs. We also propose Bounded Batch Normalization (BBN) to remove the potential risk of ill-conditioned layer transform. We finally explain how OrthDNNs are used for convolutional kernels.

Denote parameters of a DNN collectively as $\Theta = \{ \mathbf{W}_l, \mathbf{b}_l\}_{l=1}^L$, where $\{b_l\}_{l=1}^L$ are bias terms. We discuss algorithms of strict or approximate OrthDNNs in the following context. Given a training set $\{\mathbf{x}_i, y_i\}_{i=1}^m$, we write the training objective as ${\cal{L}}\left(\{\mathbf{x}_i, y_i\}_{i=1}^m; \Theta\right)$. Training is based on SGD (or its variants \cite{Momentum}), which updates $\Theta$ via a simple rule of $\Theta^{t+1} \leftarrow \Theta^t - \eta \frac{\partial{\cal{L}}}{\partial{\Theta^t}}$, where $\eta$ is the learning rate, and the gradient $\frac{\partial{\cal{L}}}{\partial{\Theta^t}}$ is usually computed from a mini-batch of training examples. Network training proceeds by sampling for each iteration $t$ a mini-batch from $\{\mathbf{x}_i, y_i\}_{i=1}^m$, until a specified number of iterations or the training loss plateaus.

\subsection{The case of strict orthogonality}
\label{SecStiefelOptim}

Enforcing orthogonality of weight matrices during network training amounts to solving the following constrained optimization problem
\begin{eqnarray}\label{EqnStiefelConstrained}
\min_{\Theta = \{ \mathbf{W}_l, \mathbf{b}_l\}_{l=1}^L} {\cal{L}}\left( \{\mathbf{x}_i, y_i\}_{i=1}^m; \Theta\right) \nonumber \\ \mathrm{s.t.} \ \mathbf{W}_l \in {\cal{O}} \ \forall \ l\in \{1, \dots, L\} ,
\end{eqnarray}
where $\cal{O}$ stands for the set of matrices whose row or column vectors are orthonormal. For $\mathbf{W}_l$ of any $l^{th}$ layer, problem (\ref{EqnStiefelConstrained}) in fact constrains its solution set as a Riemannian manifold called Stiefel manifold, which is defined as $ {\cal{M}}_l = \{ \mathbf{W}_l \in \mathbb{R}^{n_l\times n_{l-1}} | \mathbf{W}_l^{\top}\mathbf{W}_l = \mathbf{I} \}$ assuming $n_l \geq n_{l-1}$, and is an embedded submanifold of the space $\mathbb{R}^{n_l\times n_{l-1}}$, where $\mathbf{I}$ is an identity matrix. In literature, optimization of a differentiable cost function on such a matrix manifold and its convergence analysis have been intensively studied \cite{MatrixManifoldOptimBook,SGDRiemannianManifold}. For completeness, we briefly present the solving algorithm of (\ref{EqnStiefelConstrained}) as follows.

Denote $T_{\mathbf{W}_l}{\cal{M}}_l$ as the tangent space to ${\cal{M}}_l$ at the current $\mathbf{W}_l \in {\cal{M}}_l$. First-order methods such as SGD first find a tangent vector $\Omega_{\mathbf{W}_l} \in T_{\mathbf{W}_l}{\cal{M}}_l$ that describes the steepest descent direction for the cost, and update $\mathbf{W}_l$ as $\mathbf{W}_l - \eta \Omega_{\mathbf{W}_l}$ with the step size $\eta$ that satisfies conditions of convergence, and then perform a retraction ${\cal{R}}_{\mathbf{W}_l}(-\eta \Omega_{\mathbf{W}_l})$ that defines a mapping from the tangent space to the Stiefel manifold, which can be achieved by ${\cal{R}}_{\mathbf{W}_l}(-\eta \Omega_{\mathbf{W}_l}) = {\cal{Q}}(\mathbf{W}_l - \eta \Omega_{\mathbf{W}_l})$, where the operator ${\cal{Q}}$ denotes the Q factor of QR matrix decomposition. To obtain the tangent vector $\Omega_{\mathbf{W}_l}$, one may project the gradient $\frac{\partial{\cal{L}}}{\partial{\mathbf{W}_l}}$ in the embedding space $\mathbb{R}^{n_l\times n_{l-1}}$ (or its momentum version \cite{Momentum}) onto the tangent space $T_{\mathbf{W}_l}{\cal{M}}_l$ by ${\cal{P}}_{\mathbf{W}_l}\frac{\partial{\cal{L}}}{\partial{\mathbf{W}_l}}$, where ${\cal{P}}_{\mathbf{W}_l}$ defines a projection operator according to the local geometry of $\mathbf{W}_l \in {\cal{M}}_l$. Convergence analysis for such a scheme to obtain the tangent vector is presented in \cite{OzayOkataniCNNKernelSubManifold}. In Appendix \ref{ApendixSecStiefelOptimAlgm}, we present the algorithmic details for optimization of weight matrices on the Stiefel manifolds.

% Choices of retraction include Riemannian exponential mapping and its approximations. Riemannian exponential mapping is the most natural choice. Unfortunately, its computation is too expensive for practical use. Instead, one usually uses a first-order approximation without sacrificing the convergence properties.
%QR decomposition can be computed using Gram-Schmidt orthonormalization.

\subsection{Achieving near orthogonality via Singular Value Bounding}
\label{SecSVBAlgm}

Constraining solutions of weight matrices of a DNN on their Stiefel manifolds is an interesting direction of research. It also supports analysis of theoretical properties as in \cref{sec:weight-spectr-gener} and the related works \cite{ExactSolution,OzayOkataniCNNKernelSubManifold}.  However, it arguably has the following shortcomings concerned with computation, empirical performance, and also compatibility with existing deep learning methods, which motivate us to address these shortcomings by developing new algorithms of approximate OrthDNNs.

\begin{itemize}
\item Strict constraining of weight matrices on the Stiefel manifolds requires expensive computations --- in particular, the operations of projecting the Euclidean gradient onto the tangent space and retraction onto the Stiefel manifold (Steps 2 and 4 in Appendix \ref{ApendixSecStiefelOptimAlgm}) dominate the costs in each iteration. If we allow the solutions slightly away from the manifolds, the expensive projection and retraction operations are not necessary to be performed in each iteration. Instead, similar pulling-back operations can be performed less frequently, e.g., in every a certain number of iterations, and consequently such a burden of pulling back is amortized.
    % In \cref{}, we show advantage of the second scheme by comparing their computational complexities.

\item Theorem \ref{thm:ge_contraction} gives a bound $GE(f_{S_m})$ of the expected error $R(f_{S_m})$ w.r.t the training error $R_m(f_{S_m})$. To achieve good performance on practical problems, both $R_m(f_{S_m})$ and $GE(f_{S_m})$ should be small. However, optimization of DNNs is characterized by proliferation of local optima/critical points \cite{TrustRegion4SaddlePoint,DeepLearningNoPoorLocalMinima}. When we are motivated to optimize weight matrices on their Stiefel manifolds, obtaining $\mathbf{W}_l \in {\cal{M}}_l$, $l \in \{1, \dots, L\}$, with $\Omega_{\mathbf{W}_l} = 0$, it is very likely that for a $\mathbf{W}_l$, there exists a better local optimum in the embedding Euclidean space that is slightly away from the manifold (e.g., $\Omega_{\mathbf{W}_l} = 0$ while $\frac{\partial{\cal{L}}}{\partial{\mathbf{W}_l}} \neq 0$, or the Euclidean gradient is in the complement null space of the current tangent space), and has a smaller $R_m(f_{S_m})$. If we allow the optimization to step away from, but still pivot around, the manifold, better solutions could be obtained by escaping from local optima on the manifold.
    % The above argument also suggests algorithms of approximate OrthDNNs, whose advantage is shown by the empirical results in \cref{SecExps}.

\item Successful training of modern DNNs depends heavily on BN \cite{BatchNorm}, a technique that can greatly improve training convergence and empirical results. However, as analyzed shortly in \cref{SecBNCompabibility}, BN would change the spectrum of singular values of each layer transform (i.e., the combined linear transform of each layer achieved by weight mapping and BN, as specified in (\ref{EqnBNLinearTransformEquivalent})). Consequently, the efforts spending on enforcing strict orthogonality of weight matrices become ineffectual. Algorithms of approximate OrthDNNs seem more compatible with BN transform.
\end{itemize}

% The above arguments motivate us to develop algorithms of approximate OrthDNNs. We expect such algorithms (partially) enjoy the benefits of orthogonal weight matrices, while having the flexility to pursue better locally optimal solutions.

To develop an algorithm of approximate OrthDNNs, we propose a simple yet effective network training method called Singular Value Bounding (SVB). SVB is a sort of projected SGD method and can be summarized as follows: SVB simply bounds, after every $T_{svb}$ iterations of SGD training, all the singular values of each $\mathbf{W}_l$, for $l = 1, \dots, L$, in a narrow band $[1/(1+\epsilon), (1+\epsilon)]$ around the value of $1$, where $\epsilon \geq 0$ is a specified small constant. Algorithm \ref{AlgmSVB} presents the details.

% Setting $\epsilon = 0$ in SVB makes optimization of weight matrices revolve strictly around their Stiefel manifolds.

\begin{algorithm}[t]
{\footnotesize
%\SetLine
\SetKwInOut{Input}{input}
\SetKwInOut{Output}{output}
\Input{A network of $L$ layers with trainable parameters $\Theta = \{ \mathbf{W}_l, \mathbf{b}_l \}_{l=1}^L$, training loss $\cal{L}$, learning rate $\eta$, the maximal number $T$ of training iterations, a specified number $T_{svb}$ of iteration steps, a small constant $\epsilon$}

Initialize $\Theta$ such that $\mathbf{W}_l^{\top}\mathbf{W}_l = \mathbf{I}$ or $\mathbf{W}_{l}\mathbf{W}_l^{\top} = \mathbf{I}$ for $l = 1, \dots, L$

\For{$t = 0, \dots, T-1$}{
Update $\Theta^{t+1} \leftarrow \Theta^t - \eta \frac{\partial{\cal{L}}}{\partial{\Theta^t}}$ using SGD based methods

\While{training proceeds for every $T_{svb}$ iterations}{

\For{$l=1, \dots, L$}{
Perform $[\mathbf{U}_l, \mathbf{\Sigma}_l, \mathbf{V}_l] = \mathrm{svd}(\mathbf{W}_l)$

Let $\{ \sigma_i^l \}_{i=1}^{n_l}$ be the diagonal entries of $\mathbf{\Sigma}_l$

\For{$i = 1, \dots, n_l$}{
$\sigma_i^l = 1+\epsilon \ \ \text{if} \ \ \sigma_i^l > 1 + \epsilon$

$\sigma_i^l = 1/(1+\epsilon) \ \ \text{if} \ \ \sigma_i^l < 1/(1+\epsilon)$
}
Update $\mathbf{W}_l \leftarrow \mathbf{U}_l \mathbf{\Sigma}_l \mathbf{V}_l^{\top}$ with the bounded diagonal entries $\{ \sigma_i^l \}_{i=1}^{n_l}$ of $\mathbf{\Sigma}_l$
}
% \If{network contains BN layers}{Use BBN of \cref{AlgmBBN} to update BN parameters}
}
}
\Output{Trained network with parameters $\Theta^T$ for inference}
\caption{Singular Value Bounding} \label{AlgmSVB}
}
\end{algorithm}

After each bounding step, optimization of SVB in fact proceeds in the embedding Euclidean space, to search for potentially better solutions, before next bounding step that pulls the solutions back onto ($\epsilon = 0$) or near ($\epsilon > 0$) the Stiefel manifolds. With annealed learning rate schedules, we observe empirical convergence of SVB. Compared with manifold optimization in \cref{SecStiefelOptim}, SVB is more efficient since the dominating computation of SVD is invoked only every a certain number of iterations. Experiments of image classification in \cref{SecExps} show that SVB sometimes outperforms the algorithm of strict OrthDNNs in \cref{SecStiefelOptim}, both of which outperform the commonly used SGD based methods, and in many cases with a large margin.

\subsection{Alternative algorithms for approximate OrthDNNs}

To achieve approximate OrthDNNs, one may alternatively penalize the main objective ${\cal{L}}\left( \{\mathbf{x}_i, y_i\}_{i=1}^m; \Theta\right)$ with an augmented term that encourages orthonormality of columns or rows of weight matrices, resulting in the following unconstrained optimization problem
\begin{eqnarray}\label{EqnSoftRegu}
\min_{\Theta = \{ \mathbf{W}_l, \mathbf{b}_l\}_{l=1}^L} {\cal{L}}\left( \{\mathbf{x}_i, y_i\}_{i=1}^m; \Theta\right) + \lambda \sum_{l=1}^L \| \mathbf{W}_l^{\top}\mathbf{W}_l - \mathbf{I}  \|_F^2 ,
\end{eqnarray}
where $\| \cdot \|_F$ denotes Frobenius norm, $\lambda$ is the penalty parameter,  and we have assumed $n_l \geq n_{l-1}$ for a certain layer $l$. By using increasingly larger values of $\lambda$, the problem (\ref{EqnSoftRegu}) approaches to achieve strict OrthDNNs. One can use SGD based methods to solve (\ref{EqnSoftRegu}), where the additional computation cost incurred by the regularizer is marginal. Soft regularization of the type (\ref{EqnSoftRegu}) is used in the related works \cite{ParsevalNet,BeyondGoodInit}.

To relax the requirement of $n_l \geq n_{l-1}$ assumed in (\ref{EqnSoftRegu}), an algorithm termed Spectral Restricted Isometry Property (SRIP) regularization, which leverages the matrix RIP condition \cite{CandesRIP}, is proposed in \cite{SRIP}, whose objective is written as
\begin{eqnarray}\label{EqnSRIP}
\min_{\Theta = \{ \mathbf{W}_l, \mathbf{b}_l\}_{l=1}^L} {\cal{L}}\left( \{\mathbf{x}_i, y_i\}_{i=1}^m; \Theta\right) + \kappa \sum_{l=1}^L \sigma_{\max}(\mathbf{W}_l^{\top}\mathbf{W}_l - \mathbf{I}) ,
\end{eqnarray}
where $\kappa$ is a penalty parameter, and $\sigma_{\max}(\cdot)$ denotes the spectral norm of a matrix. Although computation of (\ref{EqnSRIP}) involves expensive eigen-decomposition, it can be efficiently approximated via power iteration method. One may refer to \cite{SRIP} for the solving equation. In this work, we compare the alternative (\ref{EqnSoftRegu}) and (\ref{EqnSRIP}) with our proposed SVB.

%However, it is not straightforward to set a proper value of $\lambda$ in order to strike a good balance between the two terms of (\ref{EqnSoftRegu}).

% \subsection{Computational analysis}

\subsection{Compatibility with Batch Normalization}
\label{SecBNCompabibility}

We start this section by showing that the original design of Batch Normalization \cite{BatchNorm} is incompatible with our proposed OrthDNNs. Technically, for a network layer that computes, before the nonlinear activation, $\mathbf{h} = \mathbf{W}\mathbf{x} \in \mathbb{R}^n$, BN inserts a normalization denoted as $\textrm{BN}(\mathbf{h}) = \textrm{BN}(\mathbf{W}\mathbf{x})$, where we have ignored the bias term for simplicity. BN in fact applies the following linear transformation to $\mathbf{h}$
\begin{eqnarray}\label{EqnBNLinearTransform}
\textrm{BN}(\mathbf{h}) = \bvec{\Upsilon}\bvec{\Phi} (\mathbf{h} - \mathbf{\mu}) + \mathbf{\beta} ,
\end{eqnarray}
where each entry of $\mathbf{\mu} \in \mathbb{R}^{n}$ is the output mean at each of the $n$ neurons of the layer, the diagonal matrix $\bvec{\Phi} \in \mathbb{R}^{n\times n}$ contains entries $\{1/\phi_i\}_{i=1}^n$ that is the inverse of the neuron-wise output standard deviation $\phi_i$ (obtained by adding a small constant to the variance for numerical stability), $\bvec{\Upsilon} \in \mathbb{R}^{n\times n}$ is a diagonal matrix containing trainable scalar parameters $\{ \upsilon_i \}_{i=1}^n$, and $\mathbf{\beta} \in \mathbb{R}^n$ is a trainable bias term. Note that during training, $\mu$ and $\phi$ for each neuron are computed using mini-batch examples, and during inference they are fixed representing the statistics of all the training population, which are usually obtained by running average. Thus the computation (\ref{EqnBNLinearTransform}) for each example is deterministic after network training.

Inserting $\mathbf{h} = \mathbf{W}\mathbf{x}$ into (\ref{EqnBNLinearTransform}) we get
\begin{eqnarray}\label{EqnBNLinearTransformEquivalent}
\textrm{BN}(\mathbf{x}) = \widetilde{\mathbf{W}}\mathbf{x} + \tilde{\mathbf{b}} \ \ \textrm{s.t.} \ \ \widetilde{\mathbf{W}} = \bvec{\Upsilon}\bvec{\Phi}\mathbf{W} \ \ \tilde{\mathbf{b}} = \mathbf{\beta} - \bvec{\Upsilon}\bvec{\Phi}\mathbf{\mu} ,
\end{eqnarray}
which is simply a standard layer with change of variables. The following lemma suggests that BN is incompatible with OrthDNNs: even though $\mathbf{W}$ is enforced to have orthonormal rows or columns, BN would change the conditioning of layer transform, i.e., spectrum of singular values of $\widetilde{\mathbf{W}}$, by learning $\bvec{\Upsilon}$ and $\bvec{\Phi}$ whose product does not necessarily contain diagonal entries $\{ \upsilon_i/\phi_i \}_{i=1}^n$ of equal value.

\begin{lemma}
  \label{lm:BBN}
  For a matrix $\mathbf{W} \in \mathbb{R}^{M\times N}$ with singular values of all $1$, and a diagonal matrix $\mathbf{G} \in \mathbb{R}^{M\times M}$ with nonzero entries $\{ g_i \}_{i=1}^M$, let $g_{\max} = \max(|g_1|, \dots, |g_M|)$ and $g_{\min} = \min(|g_1|, \dots, |g_M|)$, the singular values of $\widetilde{\mathbf{W}} = \mathbf{G}\mathbf{W}$ is bounded in $[g_{min}, g_{\max}]$. When $\mathbf{W}$ is fat, i.e., $M \le N$, and $\textrm{rank}(\mathbf{W}) = M$, singular values of $\widetilde{\mathbf{W}}$ are exactly $\{ |g_i| \}_{i=1}^M$.
\end{lemma}

See the proof in Appendix \ref{ApendixSecBBNProof}.

To make BN compatible with \emph{strict OrthDNNs}, we propose \emph{Degenerate Batch Normalization (DBN)} that learns layer-wise $\bar{\upsilon}$ and $\bar{\phi}$ instead of neuron-wise $\{ \upsilon_i \}_{i=1}^n$ and $\{ \phi_i \}_{i=1}^n$, so that the learned $\bvec{\Upsilon}$ and $\bvec{\Phi}$ respectively contain diagonal entries of equal value. Such a $\bvec{\Upsilon}$ is learned simply by using a single trainable parameter $\bar{\upsilon}$ shared by all $n$ neurons of the layer. To learn such a $\bvec{\Phi}$, DBN still computes neuron-wise $\{ \phi_i \}_{i=1}^n$ in each iteration of training, but uses running average to learn $\bar{\phi} = \frac{1}{n}\sum_{i=1}^n \phi_i$ that would be shared by the $n$ neurons. The proposed DBN enjoys the benefit of neuron-wise normalization in BN, and when training converges, it learns a layer transform $\widetilde{\mathbf{W}} = \bvec{\Upsilon}\bvec{\Phi}\mathbf{W} = \bar{\upsilon}/\bar{\phi} \mathbf{W}$ whose conditioning is the same as that of $\mathbf{W}$, thus achieving compatibility with strict OrthDNNs.

Section \ref{SecSVBAlgm} discusses the potential advantages of approximate OrthDNNs over the strict ones. To make BN compatible with \emph{approximate OrthDNNs}, especially with our proposed SVB method, we propose \emph{Bounded Batch Normalization (BBN)} that controls the variations among $\{ \upsilon_i/\phi_i \}_{i=1}^n$, so that the conditioning of layer transform is not severely affected by $\bvec{\Upsilon}\bvec{\Phi}$. More specifically, BBN computes $\alpha = \frac{1}{n}\sum_{i=1}^n \upsilon_i/\phi_i$ in each iteration of training, and bounds each of $\{ \frac{1}{\alpha} \upsilon_i/\phi_i \}_{i=1}^n$ in a narrow band $[1/(1+\tilde{\epsilon}), (1+\tilde{\epsilon})]$ around the value of $1$, where $\tilde{\epsilon} \geq 0$ is a scalar parameter. Algorithm \ref{AlgmBBN} presents details of the proposed BBN.

\begin{algorithm}[t]
{\footnotesize
%\SetLine
\SetKwInOut{Input}{input}
\SetKwInOut{Output}{output}
\Input{A network with $L$ BN layers, trainable parameters $\{ \bvec{\Upsilon}^t_l \}_{l=1}^L$, $\{ \mathbf{\beta}^t_l \}_{l=1}^L$, and statistics $\{ \mathbf{\mu}^t_l \}_{l=1}^L$, $\{ \bvec{\Phi}^t_l \}_{l=1}^L$ of BN layers at iteration $t$, a small constant $\tilde{\epsilon}$ }

Update to get $\{ \bvec{\Upsilon}^{t+1}_l \}_{l=1}^L$ from $\{ \bvec{\Upsilon}^t_l \}_{l=1}^L$ (and $\{ \mathbf{\beta}^{t+1}_l \}_{l=1}^L$ from $\{ \mathbf{\beta}^t_l \}_{l=1}^L$), using SGD based methods

Update to get $\{ \bvec{\Phi}^{t+1}_l \}_{l=1}^L$ from $\{ \bvec{\Phi}^t_l \}_{l=1}^L$ (and $\{ \mathbf{\mu}^{t+1}_l \}_{l=1}^L$ from $\{ \mathbf{\mu}^t_l \}_{l=1}^L$), using running average over statistics of mini-batch examples

\For{$l = 1, \dots, L$}{
Let $\{ \upsilon_i \}_{i=1}^{n_l}$ and $\{1/\phi_i\}_{i=1}^{n_l}$ be respectively the diagonal entries of $\bvec{\Upsilon}^{t+1}_l$ and $\bvec{\Phi}^{t+1}_l$

Let $\alpha = \frac{1}{n_l}\sum_{i=1}^{n_l} \upsilon_i/\phi_i$

\For{$i = 1, \dots, n_l$}{

$\upsilon_i = \alpha\phi_i(1+\tilde{\epsilon}) \ \ \text{if} \ \ \frac{1}{\alpha}\upsilon_i/\phi_i > 1 + \tilde{\epsilon}$

$\upsilon_i = \alpha\phi_i/(1+\tilde{\epsilon)} \ \ \text{if} \ \ \frac{1}{\alpha}\upsilon_i/\phi_i < 1/(1+\tilde{\epsilon)}$
}
}
\Output{Updated BN parameters and statistics at iteration $t+1$}
\caption{Bounded Batch Normalization} \label{AlgmBBN}
}
\end{algorithm}

The introduction of DBN makes it possible to empirically compare strict and approximate OrthDNNs in the context of modern architectures, which is presented in \cref{SecExps}. Experiments in \cref{SecExps} also show that performance is improved when using BBN instead of BN, confirming the benefit by resolving BN's compatibility with OrthDNNs.

\subsection{Orthogonal Convolutional Neural Networks}
\label{SecOrthCNNs}

In previous sections, we present theories and algorithms of OrthDNNs by writing their layer-wise weights in matrix forms. When applying DNNs to image data, one is actually using networks with convolutional layers. For an $l^{th}$ convolutional layer with weight tensor of the size $n_l\times n_{l-1}\times n_h\times n_w$, where $n_h$ and $n_w$ denote the height and width of the convolutional kernel, we choose to convert the tensor as a matrix of the size $n_l\times n_{l-1} n_h n_w$ based on the following rational. Natural images are usually modeled by first learning filters from (densely overlapped) local patches, and then applying the thus learned filters to images to aggregate the corresponding local statistics. The convolutional layer in fact linearly transforms the $n_{l-1}$ input feature maps in the same way, by applying each of $n_l$ filters of the size $n_{l-1}\times n_h\times n_w$ to $n_{l-1}$ patches of the size $n_h\times n_w$ in a sliding window fashion, resulting in local responses that are arrayed in the form of $n_l$ feature maps, which have the same size as that of input feature maps when padding the boundaries. In other words, the convolutional layer applies linear transformation, using $n_l$ filters, to $n_{l-1} n_h n_w$-dimensional instances that are collected from local patches of input feature maps. Correspondingly, we choose to convert the weight tensor containing the $n_l$ filters of dimension $n_{l-1} n_h n_w$ to its matrix form, and apply to it specific algorithms of OrthDNNs.

However, we note that our way of forming the weight matrices does not exactly specify linear transformations of convolutional layers. For the $l^{th}$ layer, its exact weight matrix is in fact a function of the kernel tensor and contains doubly block circulant submatrices \cite{sedghi2018singular}. Our preliminary experiments show that OrthDNNs based on such forms are effective to regularize network training as well. We would conduct further investigations with additional experiments in future research.

\section{Experiments}
\label{SecExps}

We present in this section extensive experiments of image classification to verify the efficacy of OrthDNNs. We are particularly interested in how algorithms of strict or approximate OrthDNNs provide regularization to various architectures of modern DNNs, such as ConvNets \cite{LeCun98, VGGNet}, ResNets \cite{ResNet, PreActResNet}, DenseNet \cite{DenseNet}, and ResNeXt \cite{ResNeXt}. We use the benchmark datasets of CIFAR10, CIFAR100 \cite{Cifar}, and ImageNet \cite{ILSVRC15} for these experiments. We compare empirical performance and efficiency among different algorithms of strict and approximate OrthDNNs. For some of these comparisons, we also investigate behaviours of OrthDNNs under regimes of both small and large sizes of training samples, and robustness of OrthDNNs against corruptions that are commonly encountered in natural images, in order to better understand the empirical strength of OrthDNNs. For network training, we use SGD with momentum and initialize networks using orthogonal weight matrices, where the momentum is set as $0.9$ with a weight decay of $0.0001$. When our proposed SVB is turned on, we apply it to weight matrices of all layers after every epoch of training.

\subsection{Comparative studies on algorithms of strict and approximate OrthDNNs}
\label{SecExpsStrictStudies}

\begin{figure*}[t]
  \centering
  \includegraphics[scale=0.4]{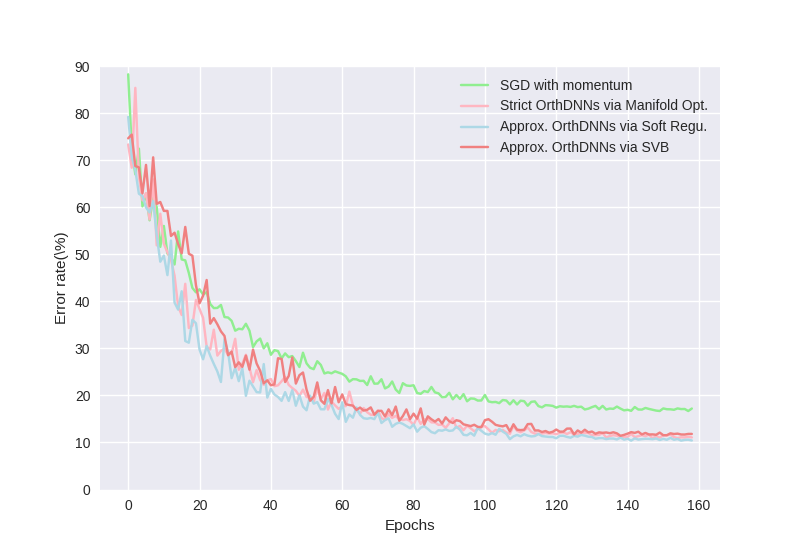}
  \includegraphics[scale=0.4]{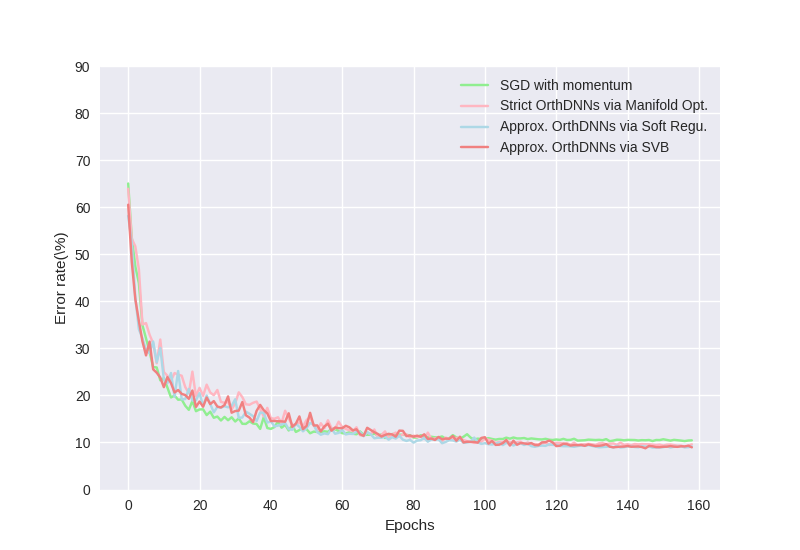}  \\
  \caption{ {\small Validation curves of strict and approximate OrthDNNs on the CIFAR10 dataset \cite{Cifar} using architectures of a ConvNet (left) and a ResNet (right) respectively of $20$ weight layers. }}
  \label{FigExpControlStudies}
\end{figure*}

\begin{table}[t]
\caption{{\small Comparison of strict and approximate OrthDNNs on the CIFAR10 dataset \cite{Cifar}, using ConvNet and ResNet architectures respectively of $20$ weight layers (referring to the main text for their specifics). Degenerate batch normalization is used due to its compatibility with strict OrthDNNs.  } }
\label{TableExpControlStudies}
\begin{center}
\begin{tabular}{c|c|c|c}
\hline
  \thead{\scriptsize Network} &
  \thead{\scriptsize Training method} &
  \thead{\scriptsize Error rate ($\%$)} &
  \makecell{\scriptsize Averaged time \\  per iter. (sec.)}\\
\hline
\multirow{3}{*}{\makecell{\scriptsize \\ \\ ConvNet}}	&
{\scriptsize SGD with momentum}   &  $16.68$	& $0.0702$\\
\cline{2-4}
~&\makecell{\scriptsize Strict OrthDNNs \\via Manifold Opt.} & $10.85$	& $0.2034$\\
\cline{2-4}
~&\makecell{\scriptsize Approx. OrthDNNs \\via Soft Regu.}& $10.39$	& $0.0930$\\
\cline{2-4}
~&\makecell{\scriptsize Approx. OrthDNNs \\via SRIP}& $10.67$	& $0.0760$ \\
\cline{2-4}
~&\makecell{\scriptsize Approx. OrthDNNs \\via SVB}& $11.41$	& $0.0718$\\
\hline

\multirow{3}{*}{\makecell{\scriptsize \\ \\ ResNet \\   }}	&
{\scriptsize SGD with momentum}   &  $10.27$	& $0.0754$\\
\cline{2-4}
~&\makecell{\scriptsize Strict OrthDNNs \\via Manifold Opt.}& $9.03$	& $0.2042$\\
\cline{2-4}
~&\makecell{\scriptsize Approx. OrthDNNs \\via Soft Regu.}& $8.91$	& $0.0844$\\
\cline{2-4}
~&\makecell{\scriptsize Approx. OrthDNNs \\via SRIP}& $8.72$	& $0.0832$\\
\cline{2-4}
~&\makecell{\scriptsize Approx. OrthDNNs \\via SVB}& $8.76$	& $0.0768$\\
\hline
\end{tabular}
\end{center}
\end{table}

In this section, we use architectures of ConvNet and ResNet on CIFAR10 to study the behaviours of strict and approximate OrthDNNs.  The CIFAR10 dataset consists of $60,000$ $32\times 32$ color images of $10$ object categories ($50,000$ training and $10,000$ testing ones). We use raw images without pre-processing. Data augmentation follows the standard manner in \cite{DeeplySupervisedNet}: during training, we zero-pad $4$ pixels along each image side, and sample a $32\times 32$ region crop from the padded image or its horizontal flip; during testing, we use the original non-padded image. Our ConvNet architectures follow \cite{VGGNet,ResNet}. Each network starts with a conv layer of $16$ $3\times 3$ filters, and then sequentially stacks three types of $2X$ conv layers of $3\times 3$ filters, each of which has the feature map sizes of $32$, $16$, and $8$, and filter numbers of $16$, $32$, and $64$, respectively; spatial sub-sampling of feature maps is achieved by conv layers of stride $2$; the network ends with a global average pooling and a fully-connected layer. The ResNet construction is based on the ConvNets presented above, where we use an ``identity shortcut'' to connect every two conv layers of $3\times 3$ filters and use a ``projection shortcut'' when sub-sampling of feature maps is needed; we adopt the pre-activation version \cite{PreActResNet}. Thus, for both types of networks, we have $6X+2$ weight layers in total. We set $X = 3$ for experiments in this section, giving networks of $20$ weight layers.

The DBN proposed in \cref{SecBNCompabibility} is designed to be compatible with strict OrthDNNs. We use DBN to enable training and comparison of strict and approximate OrthDNNs on the networks constructed above. We implement strict OrthDNNs as the algorithm presented in \cref{SecStiefelOptim}. We use our proposed SVB and those in \cite{ParsevalNet,BeyondGoodInit,SRIP} (i.e., the problems (\ref{EqnSoftRegu}) and (\ref{EqnSRIP})) for approximate OrthDNNs. The learning rates start at $0.1$ and end at $0.001$, and decay every two epochs until the end of $160$ epochs of training, where we set the mini-batch size as $128$. We fix the parameter $\epsilon$ of SVB as $0.05$, while both $\lambda$ of soft regularization in (\ref{EqnSoftRegu}) and $\kappa$ of SRIP in (\ref{EqnSRIP}) are optimally tuned as $0.1$.

% --- the resulting pre-activation ResNet has three shortcut connections, and two of them are of projection shortcut that have associated layer weights \cite{ResNet}.

Table \ref{TableExpControlStudies} gives the results with the curves of training convergence plotted in \cref{FigExpControlStudies}. Table \ref{TableExpControlStudies} shows that on both of the two networks, algorithms of strict and approximate OrthDNNs outperform standard SGD based method, confirming the improved generalization by their regularization of network training. Moreover, approximate OrthDNNs via either SVB, soft regularization, or SRIP perform as well as strict ones, but at a much lower computational cost \footnote{In \cref{TableExpControlStudies}, the respective dominating computations of QR decomposition for manifold optimization and singular value decomposition for SVB are based on CUDA implementation.}, suggesting their advantage in practical use. Due to the prohibitive computation of strict OrthDNNs on modern architectures of larger sizes, we choose to use approximate OrthDNNs in subsequent experiments, and correspondingly use BN or our proposed BBN to replace DBN.

\subsection{Comparison of hard and soft regularization for approximate OrthDNNs}
\label{SecExpsApproStudies}

\begin{table}[t]
\caption{{\small Comparison of approximate OrthDNNs via hard and soft regularization on the CIFAR10 dataset \cite{Cifar}, using a ResNet of $68$ weight layers. Each setting is run for five times, and results are reported in the format of best (mean $\pm$ standard deviation).}}
\label{TableResNetAblation}
\begin{center}
\begin{tabular}{c|c}
\hline
  \thead{\scriptsize Training method} &
  \thead{\scriptsize Error rate ($\%$)} \\
\hline
{\scriptsize SGD with momentum + BN}   &  $6.25$ ($6.43 \pm 0.15$)  \\
\hline
{\scriptsize Soft Regularization + BN} & $6.12$ ($6.28 \pm 0.12$)\\
\hline
{\scriptsize SRIP + BN} & $5.86$ ($5.95 \pm 0.08$) \\
\hline
{\scriptsize SVB + BN}     &  $5.84$ ($5.96 \pm 0.17$)  \\
\hline
{\scriptsize Soft Regularization + BBN} & $6.22$ ($6.30 \pm 0.07$) \\
\hline
{\scriptsize SRIP + BBN} & $5.99$ ($6.10 \pm 0.11$)  \\
\hline
{\scriptsize SVB + BBN}    &  $5.79$ ($5.88 \pm 0.07$) \\
\hline
\end{tabular}
\end{center}
\end{table}

In this section, we study algorithms of approximate OrthDNNs by comparing our proposed SVB with soft regularization \cite{ParsevalNet,BeyondGoodInit} and SRIP \cite{SRIP}. The experiments are conducted on CIFAR10 using a pre-activation version of ResNet constructed in the same way as in \cref{SecExpsStrictStudies}. Setting $X=11$ gives a total of $68$ weight layers. To train the network, we use learning rates that start at $0.5$ and end at $0.001$, and decay every two epochs until the end of $160$ epochs of training, where we set the mini-batch size as $128$. Comparison is made with the baseline of standard SGD with momentum. We also switch BBN on or off to verify its effectiveness. We fix $\epsilon$ of SVB as $0.5$, while the penalty $\lambda$ of soft regularization and $\kappa$ of SRIP are optimally tuned as $0.005$ and $0.01$ respectively. We fix $\tilde{\epsilon}$ of BBN as $0.2$. We run each setting of experiments for five times, and report results in the format of best (mean $\pm$ standard deviation).

Table \ref{TableResNetAblation} shows that approximate OrthDNNs via SVB, soft regularization, and SRIP provide effective regularization to network training, and SVB and SRIP outperform soft regularization with a noticeable margin. Compared with BN, our proposed BBN can better regularize training and give slightly improved performance. Note that algorithmic design of BBN may not be compatible with soft regularization and RRIP, which explains the degraded performance when using them together.

\subsection{Experiments with Modern Architectures}
\label{SecExpModernArch}

In this section, we investigate how our proposed SVB and BBN methods provide regularization to modern architectures of ResNet\cite{PreActResNet}, Wide ResNet\cite{WideResNet}, DenseNet\cite{DenseNet}, and ResNeXt\cite{ResNeXt}. We use CIFAR10, CIFAR100 \cite{Cifar}, and ImageNet \cite{ILSVRC15} for these experiments. The CIFAR100 dataset has the same number of $32\times 32$ color images as CIFAR10 does, but it has $100$ object categories where each category contains one-tenth images of those of CIFAR10. We use data augmentation in the same way as for CIFAR10. The ImageNet dataset contains $1.28$ million images of $1,000$ categories for training, and $50,000$ images for validation. We use data augmentation as in \cite{ResNeXt}.

For experiments on CIFAR10 and CIFAR100, we use the following specific architectures. ResNet is constructed in the same way as in \cref{SecExpsStrictStudies}; we set $X=9$ here giving a total of $56$ weight layers. Wide ResNet is the same as ``WRN-28-10'' in \cite{WideResNet}. ResNeXt is the same as ``ResNeXt-$29$ ($16\times 64$d)'' in \cite{ResNeXt}, i.e., the depth $L=29$, cardinality $C=16$, and the feature width in each cardinal branch $d=64$. We use consistent hyper-parameters to train these architectures. The learning rates start at $0.5$ and end at $0.001$, and decay every two epochs until the end of $300$ epochs of training, where we set the mini-batch size as $64$ --- note that this schedule with more training epochs and smaller mini-batch size usually gives better empirical performance than the training schedule used in \cref{SecExpsApproStudies} does. We fix $\epsilon$ and $\tilde{\epsilon}$ of SVB and BBN as $0.5$ and $0.2$ respectively. Table \ref{TableCIFAR10Comp} confirms that SVB and BBN improve generalization of various architectures. We also observe that improvements on CIFAR100 are generally greater than those on CIFAR10, which may be due to the problem nature of smaller sample size for CIFAR100. We will investigate how our methods perform with varying sample sizes more thoroughly in the subsequent section.

% DenseNet is the same as that of ``DenseNet-BC'' \cite{DenseNet}, and we set the depth $L=190$ and each block output $k=40$ features; in this work, we adopt the memory-efficient version of DenseNet \cite{Eff_DenseNet}.

For experiments on ImageNet, we use top-performing models of the following architectures: ``ResNet-152'' of \cite{PreActResNet}, ``DenseNet-264'' of \cite{DenseNet}, and ``ResNeXt-$101$ (64$\times 4$d)'' of \cite{ResNeXt}. To train these models, we use the same hyper-parameters as respectively reported in these methods. The parameters $\epsilon$ and $\tilde{\epsilon}$ of SVB and BBN are fixed as $0.5$ and $0.5$ respectively. Results in \cref{TableImageNetComp} confirm that approximate OrthDNNs via our proposed methods improve generalization by providing effective regularization to large-scale learning.

%  \cite{PreActResNet,DenseNet,ResNeXt,SENet}, with the mini-batch size as the only difference that we consistently set as $256$ --- this is because some of these models (e.g., ``SENet-154'') are trained using as large sizes of mini-batches as $2048$, requiring a GPU cluster that is unavailable to us.

%\footnote{Note that when applying SVB to a conv layer in ResNeXt, instead of converting its whole kernel tensor into the size of $N_{out}\times N_{in} N_h  N_w$, we are converting each of its cardinal branches into $N_{cardinalityOut}\times N_{cardinalityIn} N_h  N_w$, where $N_{cardinalityOut}$ and $N_{cardinalityIn}$ denotes the numbers of output and input feature channels for each cardinal branch respectively.}

\begin{table}[t]
  \caption{{\small Error rates ($\%$) on the CIFAR10 and CIFAR100 \cite{Cifar} datasets when applying our proposed SVB and BBN to various modern architectures (referring to the main text for their specifics).  }}
  \label{TableCIFAR10Comp}
  \begin{center}
    \begin{tabular}{c|c|c}
    \hline
     \thead{\scriptsize Method} &
     \thead{\scriptsize CIFAR$10$}  &
     \thead{\scriptsize CIFAR$100$}  \\
    \hline
    {\scriptsize ResNet W/O SVB+BBN}          &  $5.68$ & $27.71$ \\
    {\scriptsize ResNet WITH SVB+BBN}         &  $5.28$ & $26.47$ \\
    \hline
    {\scriptsize Wide ResNet W/O SVB+BBN}     &  $3.78$ & $20.02$  \\
    {\scriptsize Wide ResNet WITH SVB+BBN}    &  $3.24$ & $18.75$  \\
    \hline
%    {\scriptsize DenseNet W/O SVB+BBN}        &  $4.49$ & $21.78$  \\
%    {\scriptsize DenseNet WITH SVB+BBN}       &  $4.12$ & $19.58$  \\
%    \hline
    {\scriptsize ResNeXt W/O SVB+BBN}         &  $4.12$ & $20.65$ \\
    {\scriptsize ResNeXt WITH SVB+BBN}        &  $3.33$ & $16.94$ \\
   	\hline
    \end{tabular}
    \end{center}
\end{table}

\begin{table}[t]
  \caption{{\small Error rates ($\%$) on the validation set of ImageNet \cite{ILSVRC15} when applying our proposed SVB and BBN to various modern architectures (referring to the main text for their specifics). Results are based on single-crop testing of the size $320\times 320$. }}
  \label{TableImageNetComp}
  \begin{center}
    \begin{tabular}{c|c|c}
    \hline
     \thead{\scriptsize Method} &
     \thead{\scriptsize Top-1 error}  &
     \thead{\scriptsize Top-5 error}  \\
    \hline
    {\scriptsize ResNet W/O SVB+BBN}          &  $21.00$ & $5.72$ \\
    {\scriptsize ResNet WITH SVB+BBN}         &  $20.74$ & $5.35$  \\
    \hline
    {\scriptsize DenseNet W/O SVB+BBN}        &  $22.32$ & $6.33$  \\
    {\scriptsize DenseNet WITH SVB+BBN}       &  $21.80$ & $5.83$  \\
    \hline
    {\scriptsize ResNeXt W/O SVB+BBN}         &  $19.39$ & $4.43$  \\
    {\scriptsize ResNeXt WITH SVB+BBN}        &  $18.89$ & $4.27$  \\
   	\hline
    \end{tabular}
    \end{center}
\end{table}

\subsection{Effects of Varying Sample Sizes}
\label{SecExpVaryingSampleSize}

We are also interested in the efficacy of SVB and BBN for problems with varying sizes of training samples. To this end, we respectively sample $1/10$, $1/5$, $1/2$, or all of training images per category from ImageNet \cite{ILSVRC15}, which constitute our ImageNet training subsets of varying sizes. We train the ``ResNeXt-101 (64$\times$4d)'' model of \cite{ResNeXt} in the same way as in \cref{SecExpModernArch} for this investigation. Fig. \ref{FigSmallImageNetComp} shows that SVB and BBN consistently improve classification across the regimes from small to large sizes of training samples, and the improvements are more obvious for the smaller ones. %, suggesting the great potential of approximate OrthDNNs for machine learning problems of small sample sizes.

\begin{figure}[t]
  \centering
  \includegraphics[scale=0.17]{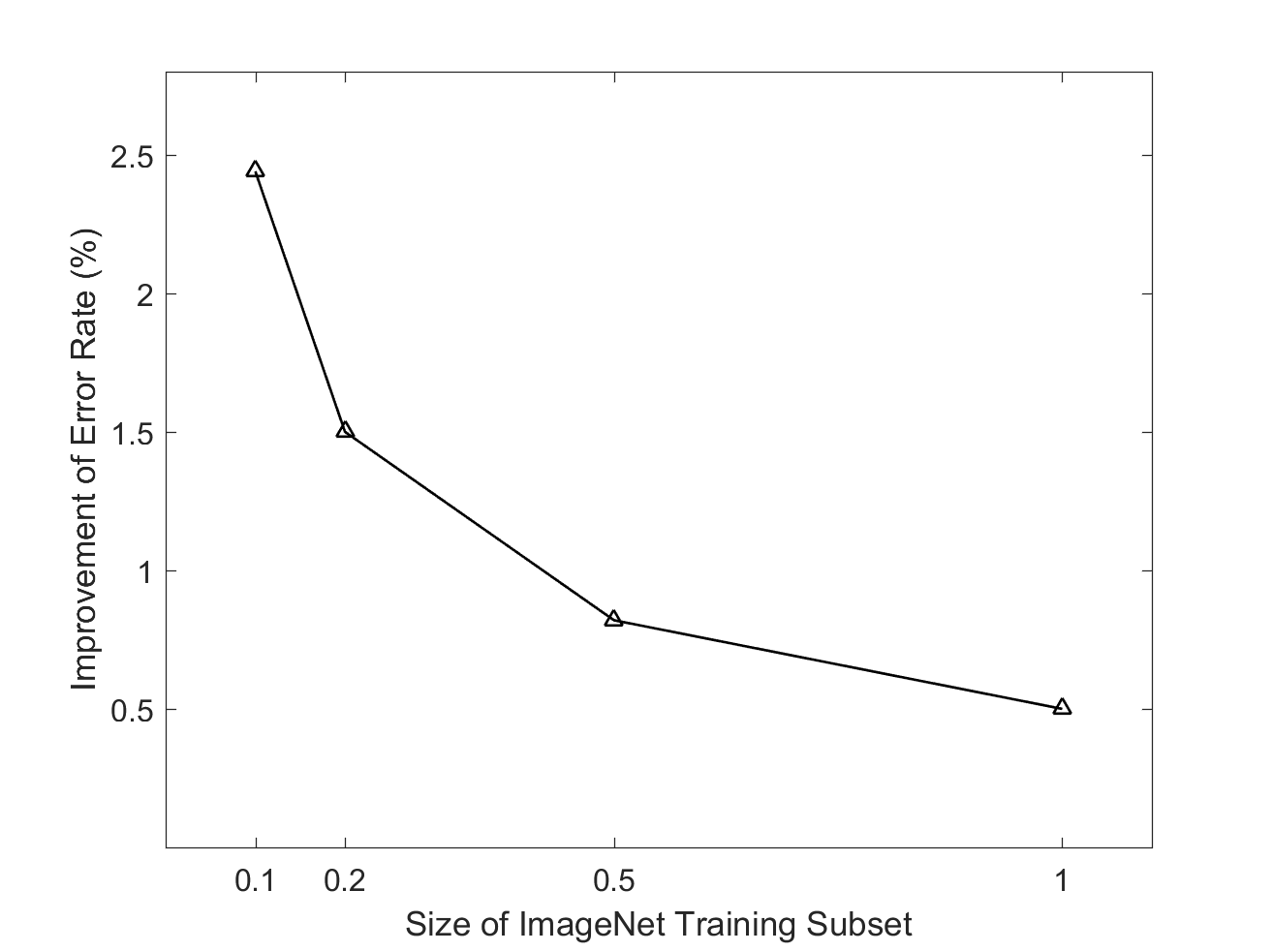}  \\
  \caption{ {\small Regularization effects of our proposed SVB and BBN for varying sizes of training samples. Results in terms of top-1 error rate improvement ($\%$) are obtained by training ResNeXt-101 \cite{ResNeXt} on ImageNet subsets that are constructed by respectively sampling $1/10$, $1/5$, $1/2$, or all of training images per category from ImageNet. Our methods regularize network training and achieve improved results over the respective baselines of $44.10\%$, $34.21\%$, $25.32\%$, and $19.39\%$. Results are based on single-crop testing of the size $320\times 320$. }}
  \label{FigSmallImageNetComp}
\end{figure}

\subsection{Robustness against Common Corruptions}
\label{SecExpCommonCorruptionRubustness}

We have shown in previous experiments that OrthDNNs, particularly our proposed SVB and BBN, have better generalization to testing samples that are drawn from the same distributions of training ones. In this section, we investigate the robustness of OrthDNNs when testing samples are corrupted such that they are getting away from the distributions of training ones. We focus on corruptions that are frequently encountered in natural images, e.g., the ``common'' corruptions of noise, blur, weather, or digitization \cite{ImageNetC}. Existing research suggests fragility of deep learning models to corruptions of such kinds \cite{DodgeAndKaram}, and that fine-tuning on specific corruption types would help, but cannot well generalize to other types of corruptions \cite{Geirhos2017a,Vasiljevic17}. To this end, we use the ImageNet-C dataset \cite{ImageNetC} that is produced by applying 15 corruption types of 5 severity levels to validation images of ImageNet \cite{ILSVRC15} \footnote{The 15 types of corruptions include Gaussian Noise, Shot Noise, Impulse Noise, Defocus Blur, Frosted Glass Blur, Motion Blur, Zoom Blur, Snow, Frost, Fog, Brightness, Contrast, Elastic, Pixelate, and JPEG. Refer to \cite{ImageNetC} for examples of corrupted ImageNet validation images.}. As indicated in \cite{ImageNetC}, networks should not be trained or fine-tuned on this dataset for testing of their robustness. We again use the trained models of ResNeXt-101 as described in \cref{SecExpModernArch}, with or without the regularization of SVB and BBN. Performance improvements of top-1 error rates are plotted in \cref{FigExpCommonCorruption}, where result for each severity level is an average over the 15 types of corruptions. Compared with the result on clean images of ImageNet validation set, \cref{FigExpCommonCorruption} demonstrates better robustness of our proposed methods against common corruptions, and the robustness stands gracefully with the increase of severity levels.

\begin{figure}[t]
  \centering
  \includegraphics[scale=0.17]{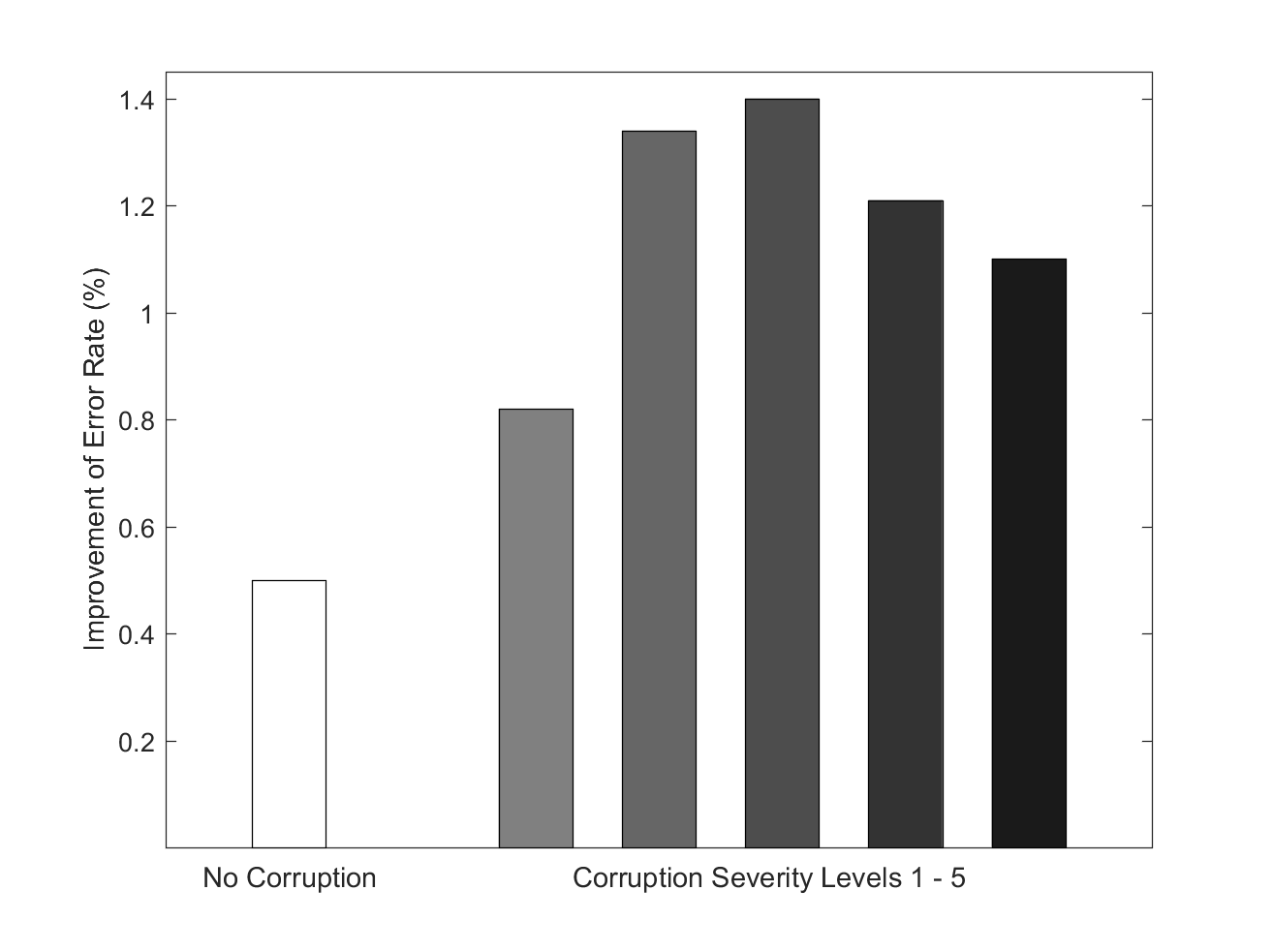} \\
  \caption{ {\small Robustness test of 5 severity levels on the ImageNet-C dataset \cite{ImageNetC}. Results in terms of top-1 error rate improvement ($\%$) are obtained by applying ResNeXt-101 \cite{ResNeXt} models, which are trained without or with regularization of our proposed SVB and BBN, to either clean or corrupted validation images of ImageNet. Our methods give improved results over the respective baselines of $19.39\%$, $30.56\%$, $39.26\%$, $47.09\%$, $58.50\%$, and $70.65\%$.  }  }
  \label{FigExpCommonCorruption}
\end{figure}

%Evaluation metrics are based on mean Corruption Error (mCE) and Relative mCE of \cite{ImageNetC} \footnote{Denote as $\mathrm{Err}_{i,j}^{\mathrm{base}}$ the top-1 error rate of a baseline network on the ImageNet-C dataset \cite{ImageNetC}, and as $\mathrm{Err}_{i,j}$ for that of a network of interest, where $i\in\{1, \dots, 5\}$ indexes the severity levels and $j\in\{1, \dots, 15\}$ indexes the corruption types. We also denote as $\mathrm{Err}_{\mathrm{clean}}^{\mathrm{base}}$ and $\mathrm{Err}_{\mathrm{clean}}$ the corresponding error rates on the original validation images of ImageNet. mCE of the network of interest is computed as $\mathrm{mCE} = 1/15\sum_{j=1}^{15}\mathrm{CE}_j$, with $\mathrm{CE}_j = \sum_{i=1}^5\mathrm{Err}_{i,j} / \sum_{i=1}^5\mathrm{Err}_{i,j}^{\mathrm{base}}$. Relative mCE is computed as ... To compute the values of mCE and Relative mCE for any specific type of corruption, we do not sum over the 5 severity levels in the above formulas. }.
%
%where the use of error rates from a baseline method is to make corruptions of different types comparable with each other.  mCE evaluates the capabilities of deep learning methods, while Relative mCE evaluates ...
%
%We use models of ResNet families for experiments in this section, ResNet-34 ... These models are trained on ImageNet by ... , with or without our proposed SVB and BBN. These models respectively report top-1 error rates on the clean validation images of ImageNet ...

\section{Conclusion}

In this paper, we present theoretical analysis to connect with the recent interest of spectrally regularized deep learning methods. Technically, we prove a new generalization error bound for DNNs, which is both scale- and range-sensitive to singular value spectrum of each of networks' weight matrices. The bound is established by first proving that DNNs are of local isometry on data distributions of practical interest, and then introducing the local isometry property of DNNs into a PAC based generalization analysis. We further prove that the optimal bound w.r.t. the degree of isometry is attained when each weight matrix has a spectrum of equal singular values --- OrthDNNs with weight matrices of orthonormal rows or columns are thus the most straightforward choice. Based on such analysis, we present algorithms of strict and approximate OrthDNNs, and propose a simple yet effective algorithm called Singular Value Bounding. We also propose Bounded Batch Normalization to make compatible use of batch normalization with OrthDNNs. Experiments on benchmark image classification show the efficacy and robustness of OrthDNNs and our proposed SVB and BNN methods. % especially for problems where the sizes of training samples are relatively small. We expect OrthDNNs provide effective regularization to other deep learning problems of smaller sample sizes.

%% use section* for acknowledgment
%\ifCLASSOPTIONcompsoc
%  % The Computer Society usually uses the plural form
%\else
%  % regular IEEE prefers the singular form
%\fi

\ifCLASSOPTIONcaptionsoff
  \newpage
\fi

\bibliographystyle{plain}
\bibliography{DeepLearning}

\begin{thebibliography}{10}

\bibitem{MatrixManifoldOptimBook}
P.~A. Absil, R.~Mahony, and R.~Sepulchre.
\newblock {\em Optimization Algorithms on Matrix Manifolds}.
\newblock Princeton University Press, Princeton, NJ, USA, 2007.

\bibitem{UnitaryRNN}
Martin Arjovsky, Amar Shah, and Yoshua Bengio.
\newblock Unitary evolution recurrent neural networks.
\newblock {\em CoRR}, arXiv:1511.06464, 2016.

\bibitem{UnderstandDropout}
Pierre Baldi and Peter~J Sadowski.
\newblock Understanding dropout.
\newblock In {\em Advances in Neural Information Processing Systems 26}, pages
  2814--2822. 2013.

\bibitem{SRIP}
Nitin Bansal, Xiaohan Chen, and Zhangyang Wang.
\newblock Can we gain more from orthogonality regularizations in training deep
  cnns?
\newblock In {\em Proceedings of the 32Nd International Conference on Neural
  Information Processing Systems}, NIPS'18, pages 4266--4276, 2018.

\bibitem{EfficientOrthDNNNips2018}
Nitin Bansal, Xiaohan Chen, and Zhangyang Wang.
\newblock Can we gain more from orthogonality regularizations in training deep
  networks?
\newblock In {\em Advances in Neural Information Processing Systems 31}, pages
  4261--4271. 2018.

\bibitem{Barron1993}
Andrew~R. Barron.
\newblock Universal approximation bounds for superpositions of a sigmoidal
  function.
\newblock {\em IEEE Transactions on Information Theory}, 39(3):930--945, 1993.

\bibitem{SpectrallyNormalizedMarginBound}
Peter~L. Bartlett, Dylan~J. Foster, and Matus~J. Telgarsky.
\newblock Spectrally-normalized margin bounds for neural networks.
\newblock In {\em Advances in Neural Information Processing Systems}, pages
  6241--6250, 2017.

\bibitem{Becigneul2017}
Gary B{\'{e}}cigneul.
\newblock {On the effect of pooling on the geometry of representations}.
\newblock Technical report, 2017.

\bibitem{SGDRiemannianManifold}
S.~Bonnabel.
\newblock Stochastic gradient descent on riemannian manifolds.
\newblock {\em IEEE Transactions on Autom. Control}, 58(9):2217--2229, 2013.

\bibitem{BousquetStability}
Olivier Bousquet and Andr{\'e} Elisseeff.
\newblock Stability and generalization.
\newblock {\em Journal of Machine Learning Research}, 2:499--526, March 2002.

\bibitem{CandesRIP}
E.~J. Candes and T.~Tao.
\newblock Decoding by linear programming.
\newblock {\em IEEE Trans. Inf. Theor.}, 51(12):4203--4215, December 2005.

\bibitem{Chafa2011SingularVO}
Djalil Chafa{\"i}, Djalil Chaf{\"a}ı, Olivier Gu{\'e}don, Guillaume Lecue, and
  Alain Pajor.
\newblock Singular values of random matrices.
\newblock
  \url{https://pdfs.semanticscholar.org/37f9/fc9b8cb7a04c7863a0d53c4c3a84a8a7da64.pdf}.

\bibitem{ParsevalNet}
Moustapha Cisse, Piotr Bojanowski, Edouard Grave, Yann Dauphin, and Nicolas
  Usunier.
\newblock Parseval networks: Improving robustness to adversarial examples.
\newblock In {\em Proceedings of the 34th International Conference on Machine
  Learning}, pages 854--863, 2017.

\bibitem{TrustRegion4SaddlePoint}
Yann~N. Dauphin, Razvan Pascanu, {\c{C}}aglar G{\"{u}}l{\c{c}}ehre, KyungHyun
  Cho, Surya Ganguli, and Yoshua Bengio.
\newblock Identifying and attacking the saddle point problem in
  high-dimensional non-convex optimization.
\newblock In {\em Advances in Neural Information Processing Systems 27}, pages
  2933--2941, 2014.

\bibitem{SharpMinimaCanGeneralize}
Laurent Dinh, Razvan Pascanu, Samy Bengio, and Yoshua Bengio.
\newblock Sharp minima can generalize for deep nets.
\newblock In {\em International Conference on Machine Learning}, pages
  1019--1028, 2017.

\bibitem{DodgeAndKaram}
Samuel Dodge and Lina Karam.
\newblock A study and comparison of human and deep learning recognition
  performance under visual distortions.
\newblock {\em arXiv preprint arXiv:1705.02498}, 2017.

\bibitem{Neocognitron}
K.~Fukushima.
\newblock Neocognitron: A self-organizing neural network for a mechanism of
  pattern recognition unaffected by shift in position.
\newblock {\em Biological Cybernetics}, 36(4):193--202, 1980.

\bibitem{Geirhos2017a}
R.~Geirhos, D.~H.~J. Janssen, H.~H. Schütt, J.~Rauber, M.~Bethge, and F.~A.
  Wichmann.
\newblock Comparing deep neural networks against humans: object recognition
  when the signal gets weaker.
\newblock {\em arXiv preprint arXiv:1706.06969}, 2017.

\bibitem{XavierInit}
Xavier Glorot and Yoshua Bengio.
\newblock Understanding the difficulty of training deep feedforward neural
  networks.
\newblock In {\em Proceedings of the International Conference on Artificial
  Intelligence and Statistics}, 2010.

\bibitem{Glorot2011}
Xavier Glorot, Antoine Bordes, and Yoshua Bengio.
\newblock Deep sparse rectifier neural networks.
\newblock In {\em Proceedings of the International Conference on Artificial
  Intelligence and Statistics}, 2011.

\bibitem{StabilitySGD}
Moritz Hardt, Ben Recht, and Yoram Singer.
\newblock Train faster, generalize better: Stability of stochastic gradient
  descent.
\newblock In {\em Proceedings of The 33rd International Conference on Machine
  Learning}, volume~48, pages 1225--1234, 2016.

\bibitem{ResNet}
Kaiming He, Xiangyu Zhang, Shaoqing Ren, and Jian Sun.
\newblock Deep residual learning for image recognition.
\newblock In {\em arXiv prepring arXiv:1506.01497}, 2015.

\bibitem{PreActResNet}
Kaiming He, Xiangyu Zhang, Shaoqing Ren, and Jian Sun.
\newblock Identity mappings in deep residual networks.
\newblock In {\em European Conference on Computer Vision}, 2016.

\bibitem{ImageNetC}
Dan Hendrycks and Thomas Dietterich.
\newblock Benchmarking neural network robustness to common corruptions and
  perturbations.
\newblock In {\em International Conference on Learning Representations}, 2019.

\bibitem{Dropout}
Geoffrey~E. Hinton, Nitish Srivastava, Alex Krizhevsky, Ilya Sutskever, and
  Ruslan Salakhutdinov.
\newblock Improving neural networks by preventing co-adaptation of feature
  detectors.
\newblock {\em CoRR}, abs/1207.0580, 2012.

\bibitem{Hornik1989}
K.~Hornik, M.~Stinchcombe, and H.~White.
\newblock Multilayer feedforward networks are universal approximators.
\newblock {\em Neural Networks}, 2(5):359--366, 1989.

\bibitem{DenseNet}
Gao Huang, Zhuang Liu, and Kilian~Q. Weinberger.
\newblock Densely connected convolutional networks.
\newblock {\em CoRR}, abs/1608.06993, 2016.

\bibitem{DiscrimRobustTransform}
Jiaji Huang, Qiang Qiu, Guillermo Sapiro, and Robert Calderbank.
\newblock Discriminative robust transformation learning.
\newblock In {\em Advances in Neural Information Processing Systems 28}, pages
  1333--1341. 2015.

\bibitem{BatchNorm}
Sergey Ioffe and Christian Szegedy.
\newblock Batch normalization: Accelerating deep network training by reducing
  internal covariate shift.
\newblock In {\em Proceedings of the 32nd International Conference on Machine
  Learning}, pages 448--456, 2015.

\bibitem{SVB}
Kui Jia, Dacheng Tao, Shenghua Gao, and Xiangmin Xu.
\newblock Improving training of deep neural networks via singular value
  bounding.
\newblock In {\em {IEEE} Conference on Computer Vision and Pattern
  Recognition}, pages 3994--4002, 2017.

\bibitem{DeepLearningNoPoorLocalMinima}
Kenji Kawaguchi.
\newblock Deep learning without poor local minima.
\newblock In {\em Advances in Neural Information Processing Systems}, 2016.

\bibitem{GeneralizationInDeepLearning}
Kenji Kawaguchi, Leslie~Pack Kaelbling, and Yoshua Bengio.
\newblock Generalization in deep learning.
\newblock {\em CoRR}, abs/1710.05468, 2017.

\bibitem{FlatMinima}
Nitish~Shirish Keskar, Dheevatsa Mudigere, Jorge Nocedal, Mikhail Smelyanskiy,
  and Ping Tak~Peter Tang.
\newblock On large-batch training for deep learning: Generalization gap and
  sharp minima.
\newblock In {\em International Conference on Learning Representations}, 2017.

\bibitem{Adam}
Diederik Kingma and Jimmy Ba.
\newblock Adam: A method for stochastic optimization.
\newblock In {\em International Conference on Learning Representations}, 2015.

\bibitem{CoveringNumber}
A.~N. Kolmogorov and V.~Tihomirov.
\newblock $\epsilon$-entropy and $\epsilon$-capacity of sets in functional
  spaces.
\newblock {\em American Mathematical Society Translations (2)}, 17:227--364,
  2002.

\bibitem{Cifar}
Alex Krizhevsky.
\newblock Learning multiple layers of features from tiny images.
\newblock {\em Tech. Report}, 2009.

\bibitem{DataDependentStabilitySGD}
Ilja Kuzborskij and Christoph~H. Lampert.
\newblock Data-dependent stability of stochastic gradient descent.
\newblock {\em CoRR}, 1703.01678, 2017.

\bibitem{LeCun98}
Y.~LeCun, L.~Bottou, Y.~Bengio, and P.~Haffner.
\newblock Gradient-based learning applied to document recognition.
\newblock {\em Proceedings of the IEEE}, 86(11):2278--2324, 1998.

\bibitem{EfficientBackprop}
Yann LeCun, L{\'e}on Bottou, Genevieve~B. Orr, and Klaus-Robert M\"{u}ller.
\newblock Efficient backprop.
\newblock In {\em Neural Networks: Tricks of the Trade}, pages 9--50, 1998.

\bibitem{DeeplySupervisedNet}
Chen{-}Yu Lee, Saining Xie, Patrick~W. Gallagher, Zhengyou Zhang, and Zhuowen
  Tu.
\newblock Deeply-supervised nets.
\newblock In {\em Proceedings of the Eighteenth International Conference on
  Artificial Intelligence and Statistics}, 2015.

\bibitem{liu2017algorithmic}
Tongliang Liu, G{\'a}bor Lugosi, Gergely Neu, and Dacheng Tao.
\newblock Algorithmic stability and hypothesis complexity.
\newblock In {\em International Conference on Machine Learning}, pages
  2159--2167, 2017.

\bibitem{FoundationMLBook}
Mehryar Mohri, Afshin Rostamizadeh, and Ameet Talwalkar.
\newblock {\em Foundations of Machine Learning}.
\newblock The MIT Press, 2012.

\bibitem{NumOfLinearRegionDNN}
Guido Mont\'{u}far, Razvan Pascanu, Kyunghyun Cho, and Yoshua Bengio.
\newblock On the number of linear regions of deep neural networks.
\newblock In {\em Advances in Neural Information Processing Systems}, pages
  2924--2932, 2014.

\bibitem{orlik1992arrangements}
P.~Orlik and H.~Terao.
\newblock {\em Arrangements of Hyperplanes}.
\newblock Grundlehren der mathematischen Wissenschaften. Springer Berlin
  Heidelberg, 1992.

\bibitem{OzayOkataniCNNKernelSubManifold}
Mete Ozay and Takayuki Okatani.
\newblock Optimization on submanifolds of convolution kernels in cnns.
\newblock {\em CoRR}, abs/1610.07008, 2016.

\bibitem{ResurrectingSigmoid}
Jeffrey Pennington, Samuel~S. Schoenholz, and Surya Ganguli.
\newblock Resurrecting the sigmoid in deep learning through dynamical isometry:
  theory and practice.
\newblock In {\em Advances in Neural Information Processing Systems}, pages
  4788--4798, 2017.

\bibitem{ILSVRC15}
Olga Russakovsky, Jia Deng, Hao Su, Jonathan Krause, Sanjeev Satheesh, Sean Ma,
  Zhiheng Huang, Andrej Karpathy, Aditya Khosla, Michael Bernstein,
  Alexander~C. Berg, and Li~Fei-Fei.
\newblock {ImageNet Large Scale Visual Recognition Challenge}.
\newblock {\em International Journal of Computer Vision}, 115(3):211--252,
  2015.

\bibitem{ExactSolution}
Andrew~M. Saxe, James~L. McClelland, and Surya Ganguli.
\newblock Exact solutions to the nonlinear dynamics of learning in deep linear
  neural networks.
\newblock In {\em International Conference on Learning Representations}, 2014.

\bibitem{sedghi2018singular}
Hanie Sedghi, Vineet Gupta, and Philip~M. Long.
\newblock The singular values of convolutional layers.
\newblock In {\em Proceedings of the International Conference on Learning and
  Representation (ICLR)}, 2019.

\bibitem{VGGNet}
K.~Simonyan and A.~Zisserman.
\newblock Very deep convolutional networks for large-scale image recognition.
\newblock {\em CoRR}, abs/1409.1556, 2014.

\bibitem{RobustLargeMarginDNNs}
Jure Sokolic, Raja Giryes, Guillermo Sapiro, and Miguel R.~D. Rodrigues.
\newblock Robust large margin deep neural networks.
\newblock {\em {IEEE} Trans. Signal Processing}, 65(16):4265--4280, 2017.

\bibitem{Momentum}
Ilya Sutskever, James Martens, George~E. Dahl, and Geoffrey~E. Hinton.
\newblock On the importance of initialization and momentum in deep learning.
\newblock In {\em Proceedings of the 30th International Conference on Machine
  Learning}, volume~28, pages 1139--1147, May 2013.

\bibitem{Vasiljevic17}
Igor Vasiljevic, Ayan Chakrabarti, and Gregory Shakhnarovich.
\newblock Examining the impact of blur on recognition by convolutional
  networks.
\newblock {\em arXiv preprint arXiv:1611.05760}, 2016.

\bibitem{Verma2013}
Nakul Verma.
\newblock {Distance Preserving Embeddings for General n-Dimensional Manifolds}.
\newblock {\em Journal of Machine Learning Research}, 14:2415--2448, 2013.

\bibitem{DropoutAdapRegu}
Stefan Wager, Sida Wang, and Percy~S Liang.
\newblock Dropout training as adaptive regularization.
\newblock In {\em Advances in Neural Information Processing Systems 26}, pages
  351--359. 2013.

\bibitem{EDJM}
Shengjie Wang, Abdel{-}rahman Mohamed, Rich Caruana, Jeff~A. Bilmes, Matthai
  Philipose, Matthew Richardson, Krzysztof Geras, Gregor Urban, and {\"{O}}zlem
  Aslan.
\newblock Analysis of deep neural networks with extended data jacobian matrix.
\newblock In {\em Proceedings of the 33nd International Conference on Machine
  Learning}, pages 718--726, 2016.

\bibitem{FullCapacityUnitaryRNN}
Scott Wisdom, Thomas Powers, John~R. Hershey, Jonathan~Le Roux, and Les Atlas.
\newblock Full-capacity unitary recurrent neural networks.
\newblock {\em CoRR}, arXiv:1611.00035, 2016.

\bibitem{BeyondGoodInit}
Di~Xie and Jiang Xiongand~Shiliang Pu.
\newblock All you need is beyond a good init: Exploring better solution for
  training extremely deep convolutional neural networks with orthonormality and
  modulation.
\newblock In {\em Computer Vision and Pattern Recognition}, 2017.

\bibitem{ResNeXt}
Saining Xie, Ross~B. Girshick, Piotr Doll{\'{a}}r, Zhuowen Tu, and Kaiming He.
\newblock Aggregated residual transformations for deep neural networks.
\newblock In {\em {IEEE} Conference on Computer Vision and Pattern
  Recognition}, pages 5987--5995, 2017.

\bibitem{RobustnessAndGeneralization}
Huan Xu and Shie Mannor.
\newblock Robustness and generalization.
\newblock {\em Machine Learning}, 86(3):391--423, 2012.

\bibitem{GlobalOptimConditions4DNNs}
Chulhee Yun, Suvrit Sra, and Ali Jadbabaie.
\newblock Global optimality conditions for deep neural networks.
\newblock In {\em International Conference on Learning Representations}, 2018.

\bibitem{WideResNet}
Sergey Zagoruyko and Nikos Komodakis.
\newblock Wide residual networks.
\newblock {\em CoRR}, abs/1605.07146, 2016.

\bibitem{ZhangICLRBestPaper}
Chiyuan Zhang, Samy Bengio, Moritz Hardt, Benjamin Recht, and Oriol Vinyals.
\newblock Understanding deep learning requires rethinking generalization.
\newblock In {\em International Conference on Learning Representations}, 2017.

\bibitem{PoggioDLTheoryIIb}
Chiyuan Zhang, Qianli Liao, Alexander Rakhlin, Brando Miranda, Noah Golowich,
  and Tomaso~A. Poggio.
\newblock Theory of deep learning iib: Optimization properties of {SGD}.
\newblock {\em CoRR}, 1801.02254, 2018.

\end{thebibliography}

\appendices

\section{Proof of Lemma 3.1.}
\label{ProofLinearNNDeltaIsometry}

To prove \cref{lm:1}, we begin with the following lemma regarding matrix pseudo-inverse.

\begin{lemma}
  \label{lm:aux}
  Given a matrix $\bvec{W} \in \mathbb{R}^{M \times N}$ and $\bvec{x} \in \mathcal{X} -
  \mathcal{N}(\bvec{W})$, where $\mathcal{X}$ is $\mathbb{R}^{N}$, we have
  \begin{displaymath}
    \bvec{W}^{\dag}\bvec{W}\bvec{x} = \bvec{x}
  \end{displaymath}
  where $\bvec{W}^{\dag}$ is the pseudo-inverse of $\bvec{W}$, given as $\bvec{W}^{\dag} = \bvec{V}\bvec{\Sigma}^{\dag}\bvec{U}^{T}$
  when $\bvec{W}$ has the singular value decomposition $\bvec{W} = \bvec{U}\bvec{\Sigma} \bvec{V}^{T}$, and $\bvec{\Sigma}^{\dag}$ is
  the matrix obtained by first taking the transpose of $\bvec{\Sigma}$, and then the inverse of its non-zero elements.
\end{lemma}

\begin{proof}
  Let $\mathcal{R}$ be the index set such that $\bvec{\Sigma}_{rr} \not= 0$, $\forall r \in \mathcal{R}$.
  Given any $\bvec{x} \in \mathcal{X} - \mathcal{N}(\bvec{W})$, $\bvec{x}$ can be represented as
  \begin{displaymath}
    \bvec{x} = \bvec{V} \bvec{\alpha} ,
  \end{displaymath}
  where entries of $\bvec{\alpha}$ have $\alpha_r = 0$ when $r \not\in \mathcal{R}$. Then
  \begin{align*}
    \bvec{W}^{\dag}\bvec{ï¼·}\bvec{W}\bvec{x} &= \bvec{V}\bvec{\Sigma}^{\dag}\bvec{U}^{T}\bvec{U}\bvec{\Sigma} \bvec{V}^{T}\bvec{V}\bvec{\alpha}\\
               &= \bvec{V}\bvec{\Sigma}^{\dag}\bvec{\Sigma} \bvec{\alpha} .
  \end{align*}

  Since $\forall \alpha_r \not= 0$, $\bvec{\Sigma}_{rr} \not= 0, \bvec{\Sigma}^{\dag}_{rr} \not= 0$, we have $\bvec{\Sigma}^{\dag}\bvec{\Sigma}
  \bvec{\alpha} = \bvec{\alpha}$, which is to say $\bvec{W}^{\dag}\bvec{W}\bvec{x} = \bvec{V}\bvec{\alpha} = \bvec{x}$.
\end{proof}

Now, we are going to prove \cref{lm:1}.
\begin{proof}
  For any $\bvec{W}_i$ with $i \in \{1,\ldots,L\}$, performing singular value decomposition (SVD) upon it, we have
  \begin{displaymath}
    \bvec{W}_i = \bvec{U}_i\bvec{\Sigma}_i \bvec{V}^{T}_i ,
  \end{displaymath}
  where $\bvec{U}_i$ and $\bvec{V}_i$ are both orthogonal matrices.

  Given any $\bvec{\Delta} = \bvec{x} - \bvec{x}' \in \mathcal{X} - \mathcal{N}(T)$, let $\bvec{\Delta}_{i-1} = \prod_{j=1}^{i-1}\bvec{W}_{j}\bvec{\Delta}$, we have
  \begin{displaymath}
    \bvec{W}_{i}\bvec{\Delta}_{i-1} = \bvec{U}_i\bvec{\Sigma}_{i}\bvec{V}_{i}^T\bvec{\Delta}_{i-1} .
  \end{displaymath}
  Let $\bvec{\Delta}'_{i-1} = \bvec{V}_{i}^{T}\bvec{\Delta}_{i-1}$. We show that if $\bvec{\Delta}'_{i-1,k} \not= 0$, then $\bvec{\Sigma}_{i,kk} \not= 0$, which
  implies that $\bvec{\Delta}_{i-1}$ lies in the subspace spanned by right singular
  vectors of $\bvec{W}_i$ that have nonzero singular values, where $\bvec{\Delta}'_{i-1,k}$ and $\bvec{\Sigma}_{i,kk}$ are respectively the $k^{th}$ element of $\bvec{\Delta}'_{i-1}$ and $k^{th}$ diagonal element of $\bvec{\Sigma}_i$.

  Suppose otherwise, for a set $\mathcal{K}$, $\bvec{\Delta}'_{i-1,k} \not= 0$ and $\bvec{\Sigma}_{i,kk} =
  0$, $\forall k \in \mathcal{K}$. Let $\bvec{v}_k$ denote the $k^{th}$ column of $\bvec{V}_{i}$, we reparameterize $\bvec{\Delta}_{i-1}$ as
  \begin{displaymath}
    \bvec{\Delta}_{i-1} = \sum\limits_{k\not\in \mathcal{K}} \bvec{\Delta}'_{i-1,k} \bvec{v}_{k} + \sum\limits_{k \in \mathcal{K}}\bvec{\Delta}'_{i-1,k} \bvec{v}_{k} ,
  \end{displaymath}
  where the terms having $\bvec{\Delta}'_{i-1,k} = 0$ are omitted.

  Denote $\bvec{W} = \prod_{j=1}^{i-1}\bvec{W}_j$. Since $\bvec{\Delta} \in \mathcal{X} -
  \mathcal{N}(\bvec{T})$, we have $\bvec{\Delta} \in \mathcal{X} - \mathcal{N}(\bvec{W})$. By \cref{lm:aux}, we
  have $\bvec{W}^{\dag}\bvec{\Delta_{i-1}} = \bvec{W}^{\dag}\bvec{W}\bvec{\Delta} = \bvec{\Delta} \in \mathcal{X} -
  \mathcal{N}(\bvec{T})$. Since
  \begin{align*}
    \bvec{W}^{\dag}\bvec{\Delta_{i-1}} &= \bvec{W}^{\dag}(\sum\limits_{k\not\in \mathcal{K}} \bvec{\Delta}'_{i-1,k}
                                 \bvec{v}_{k} + \sum\limits_{k\in \mathcal{K}}\bvec{\Delta}'_{i-1,k} \bvec{v}_{k})\\
                               &= \sum\limits_{k\not\in \mathcal{K}} \bvec{\Delta}'_{i-1,k}
                                 \bvec{W}^{\dag}\bvec{v}_{k} + \sum\limits_{k\in \mathcal{K}}\bvec{\Delta}'_{i-1,k} \bvec{W}^{\dag}\bvec{v}_{k} ,
  \end{align*}
  which is to say
  \begin{displaymath}
    \bvec{\Delta} = \sum\limits_{k\not \in \mathcal{K}} \bvec{\Delta}'_{i-1,k} \bvec{W}^{\dag}\bvec{v}_{k} + \sum\limits_{k\in \mathcal{K}}\bvec{\Delta}'_{i-1,k} \bvec{W}^{\dag}\bvec{v}_{k} .
  \end{displaymath}
  By assumption we have $\bvec{\Delta}'_{i-1,k} \not=0$; we also have $\bvec{W}^{\dag}\bvec{v}_k \not= \bvec{0}$ --- otherwise $\bvec{\Delta}'_{i-1,k}$ would be zero, and $\bvec{W}^{\dag}\bvec{v}_k \perp \bvec{W}^{\dag}\bvec{v}_{k'}$, for $k \not= k'$; Considering that $\bvec{\Delta} \in \mathcal{X} - \mathcal{N}(\bvec{T})$, we have $\bvec{W}^{\dag}\bvec{v}_k \in
\mathcal{X} - \mathcal{N}(\bvec{T})$. However, we assume $\bvec{\Sigma}_{i,kk} = 0$ for $k \in \mathcal{K}$, and have
  \begin{displaymath}
    \bvec{T}(\sum\limits_{k\in \mathcal{K}}\bvec{\Delta}'_{i-1,k} \bvec{W}^{\dag}\bvec{v}_{k}) = \bvec{0} ,
  \end{displaymath}
  which implies $\bvec{W}^{\dag}\bvec{v}_k \in \mathcal{N}(\bvec{T})$ and leads to a contradiction. We thus prove that if $\bvec{\Delta}'_{i-1,k} \not= 0$, then $\bvec{\Sigma}_{i,kk} \not= 0$.

  With the above result, we can constrain the sample variation more precisely through singular values of weight matrices. To be specific, for any pair of
  $\bvec{x}_{i-1} \in \mathcal{X}_{i-1}$ and $\bvec{x}'_{i-1} \in \mathcal{X}_{i-1}$, we have
  \begin{align*}
    ||\bvec{W}_i\bvec{x}_{i-1} - \bvec{W}_i\bvec{x}_{i-1}'|| &= ||\bvec{W}_i(\bvec{x}_{i-1} - \bvec{x}'_{i-1})|| \\
                                 &= ||\bvec{U}_i\bvec{\Sigma}_i \bvec{V}^{T}_i(\bvec{x}_{i-1} - \bvec{x}'_{i-1})||\\
                                 &= ||\bvec{\Sigma}_i \bvec{V}^{T}_i(\bvec{x}_{i-1} - \bvec{x}'_{i-1})||\\
                                 &\geq ||\sigma_{\min}^{i} \bvec{I} \bvec{V}^{T}_i(\bvec{x}_{i-1} - \bvec{x}'_{i-1})||\\
                                 &\geq \sigma_{\min}^{i} || \bvec{V}^{T}_i(\bvec{x}_{i-1} - \bvec{x}'_{i-1})||\\
                                 &= \sigma_{\min}^{i} || (\bvec{x}_{i-1} - \bvec{x}'_{i-1})|| ,
  \end{align*}
  where the second and last equalities use the fact that an orthogonal matrix does not change the norm of operated vectors, and the two inequalities are derived based on our result that for $\bvec{x}_{i-1} -
  \bvec{x'}_{i-1} \in \bvec{\Delta}_{i-1}$, it lies in the subspace spanned by the right singular
  vectors of $\bvec{W}_i$ whose corresponding singular values are great than or equal to the nonzero $\sigma^{i}_{\min}$.

  Similarly, we have
  \begin{displaymath}
    ||\bvec{W}_i\bvec{x}_{i-1} - \bvec{W}_i\bvec{x}_{i-1}'|| \leq \sigma_{\max}^{i} || (\bvec{x}_{i-1} - \bvec{x}'_{i-1})|| .
  \end{displaymath}

  Denote $||\bvec{x}_{i-1} - \bvec{x}'_{i-1}||$ as $d_i$ with $d = ||\bvec{x} - \bvec{x}'||$, we have
  \begin{displaymath}
    \sigma_{\min}^{i}d_i \leq ||\bvec{W}_i\bvec{x}_{i-1} - \bvec{W}_i\bvec{x}_{i-1}'|| \leq \sigma_{\max}^{i}d_i .
  \end{displaymath}

  Cascading on all layers, we have
  \begin{align*}
    & \prod\limits_{i=1}^{L}\sigma_{\min}^{i}d \leq ||\prod\limits_{i=1}^{L}\bvec{W}_i\bvec{x} - \prod\limits_{i=1}^{L}\bvec{W}_i\bvec{x}'|| \leq \prod\limits_{i=1}^{L}\sigma_{\max}^{i}d\\
    \iff & \prod\limits_{i=1}^{L}\sigma_{\min}^{i}d \leq ||\bvec{T}\bvec{x} - \bvec{T}\bvec{x}'|| \leq \prod\limits_{i=1}^{L}\sigma_{\max}^{i}d .
  \end{align*}

  Thus
  \begin{align*}
    & \left| \|\bvec{T}\bvec{x} - \bvec{T}\bvec{x}'|| - ||\bvec{x} - \bvec{x}'\| \right| \\
    \leq & \max(|\prod\limits_{i=1}^{L}\sigma_{\max}^{i}d - d|, |\prod\limits_{i=1}^{L}\sigma_{\min}^{i}d - d|) \\
    \leq & \max(|\prod\limits_{i=1}^{L}\sigma_{\max}^{i} - 1|2b, |\prod\limits_{i=1}^{L}\sigma_{\min}^{i} - 1|2b) .
  \end{align*}

  We conclude the proof by showing $\bvec{T}$ is of $2b\max(|\prod\limits_{i=1}^{L}\sigma_{\max}^{i} - 1|, |\prod\limits_{i=1}^{L}\sigma_{\min}^{i} - 1|)$-isometry.
\end{proof}

\section{Proof of Lemma 3.2.}
\label{ProofNonlinearNNLocalLinearity}

\begin{proof}
We proceed by induction on layer $l$.

For $l = 1$, each row in $\bvec{W}_l$ corresponds to a hyperplane in $\mathcal{X}$.
Thus, $\bvec{W}_l$ imposes a hyperplane arrangement $\mathcal{A} = \{\bvec{W}_{l,i}\}_{i=1,\ldots, n_l}$ on
$\mathcal{X}$, and is associated with an index set $\mathcal{T}(\mathcal{A}, \tau)$ of the region set
$\mathcal{R}(\mathcal{A})$, where $\bvec{W}_{l,i}$ denotes the $i^{th}$
row of $\bvec{W}_l$. Denote $a_{li}$ the neuron corresponding to a hyperplane $\bvec{W}_{l,i} \in \mathcal{A}$, we have
\begin{displaymath}
  a_{li}(\bvec{x}) =
  \begin{cases}
  \bvec{W}_{l,i}\bvec{x} & \text{if} \ \bvec{x} \in r \ \forall \ r \in \{q \in \mathcal{R} | \pi_i\tau(q) = 1 \} \\ % $\{\bvec{x} \in r | r \in \{q \in \mathcal{R} | \pi_i\tau(q) = 1 \}\}$
    0 & \text{otherwise} ,
  \end{cases}
\end{displaymath}
i.e., $a_{li}$ is linear over regions of $\mathcal{R}$ that are active on the $i^{th}$ neuron of the layer. We then have $\mathcal{Q}_{li} = \{q \in \mathcal{R} | \pi_i\tau(q) = 1 \}$ as the support of $a_{li}$.

To present the effect in the form of $\bvec{W}_l$, we have
\begin{displaymath}
  \bvec{W}^{q}_{l}= \diag(\tau(q))\bvec{W}_l ,
\end{displaymath}
which is a linear map over each $q \in \mathcal{Q}_l = \bigcup_{i=1}^{n_l}\mathcal{Q}_{li}$.
The case $l=1$ is proved, with $\tau_1 = \tau$.

Assume now for all the neurons $a_{lj}$, $j = 1, \ldots, n_l$, of layer $l$, $a_{lj}$ is a linear functional over its support
$Q_{lj}$, and is $0$-valued otherwise. We proceed by building a new set of regions
$\mathcal{Q}_{(l+1)i}, i = 1 ,\ldots, n_{l+1}$, for layer $l+1$, whose neurons are linear functionals over regions of $\mathcal{Q}_{(l+1)i}$.

We separately discuss the cases of $g$ with or without max pooling.

First, when $g$ does not include max pooling, for a neuron $a_{(l+1)i}$ of layer $l+1$, it is a functional
of the form
\begin{align*}
  a_{(l+1)i} & = g \sum\limits_{j=1}^{n_l} \bvec{W}_{l+1,ij} a_{lj}\\
             & = g \pre (a_{(l+1)i}) ,
\end{align*}
where $\bvec{W}_{l+1,ij}$ is the $(i,j)$-entry of $\bvec{W}_{l+1}$. Since $\forall q \in \bigcup_{j=1}^{n_l}\mathcal{Q}_{lj}$, $a_{lj}, j=1,\ldots, n_l$, is a linear
functional over $q$, so is the linear combination $\pre (a_{(l+1)i})$ of
them.

When $\pre (a_{(l+1)i})(\bvec{x}) > 0 \ \forall \bvec{x} \in q$, $g \pre (a_{(l+1)i}) = \pre
(a_{(l+1)i})$, and $q$ is not further divided by neuron $i$. When $\pre (a_{(l+1)i})(\bvec{x})
> 0$ for some of $ \bvec{x} \in q$, by the fact that $\pre (a_{(l+1)i})$ is a monotonous function, it
splits $q$ into two regions $q_{+}$ and $q_{-}$, where $\forall \bvec{x} \in q_{+}, \pre
(a_{(l+1)i})(\bvec{x}) > 0$ and $\forall \bvec{x} \in q_{-}, \pre (a_{(l+1)i})(\bvec{x}) \leq 0$. Since $g$ sets $a_{(l+1)i} = 0$ $\forall
\bvec{x} \in q_{-}$, $a_{(l+1)i}$ is a linear functional over $q_{+}$. When $\pre (a_{(l+1)i})(\bvec{x}) \leq 0 \ \forall \bvec{x} \in q$, $q$ does not provide support for neuron $i$, but it may support other neurons.

Consequently, neurons $a_{l+1}$ further divide the region $q$ into
sub-regions, where for each region, $a_{l+1}$ is a linear map. We say the
boundaries of the new set of regions as the hyperplane arrangement induced by
neurons, with a alight abuse of terminology. For the
newly created set of regions $\mathcal{Q}_{l+1}$,
we define the labeling function $\tau_{l+1}$ for layer $l+1$ such that for $\bvec{x} \in q' \in \mathcal{Q}_{l+1}$, we have
\begin{displaymath}
  \pi_i\tau^{\relu}_{l+1}(\bvec{x}) =
  \begin{cases}
    1 & \text{if $a_{(l+1)i}(\bvec{x}) > 0$}\\
    0 & \text{if $a_{(l+1)i}(\bvec{x}) \leq 0$} .
  \end{cases}
\end{displaymath}
Since for each $q' \in \mathcal{Q}_{l+1}$, it is a sub-region of $q \in
\mathcal{Q}_{l}$, $q' \in \mathcal{Q}_{l}$ holds true as well.
In this case, the new set of regions are built.

For the case that $g$ includes max pooling, the neuron is of the form
\begin{align*}
  a_{(l+1)i} & = \max\limits_{k\in K}(\ReLU \sum\limits_{j=1}^{n_l} \bvec{W}_{l+1,(si + k)j} a_{lj})\\
             & =  \max\limits_{k\in K}(\pre (a_{(l+1)i})_k) ,
\end{align*}
where $K$ is index set of neurons being pooled with $|K| = s$, and $\pre (a_{(l+1)i})_k$ denotes $\ReLU \sum\limits_{j=1}^{n_l} \bvec{W}_{l+1,(si + k)j} a_{lj}$.

Similarly, for each $k \in K$, $\pre (a_{(l+1)i})_k$ may split $q$ into two
sub-regions. Denote the boundary as $H_k$ if it indeed splits. $\{H_k\}_{k\in
K}$, together with the boundary of $q$, forms a hyperplane arrangement
$\mathcal{A}^{*}_{(l+1)i}$ induced by neurons within $q$. Denote the set of regions in this new
arrangement as $\mathcal{Q}_{(l+1)i}^{*}$ , for each $q^{*} \in
\mathcal{Q}_{(l+1)i}^{*}$ , consider the set of hyperplanes $\mathcal{A}'_{(l+1)i} = \{\pre (a_{(l+1)i})_k
- \pre (a_{(l+1)i})_{k'}\}_{k < k', k,k' \in K}$. With a similar argument, they will create another
hyperplane arrangement $\mathcal{A}'_{(l+1)i}$ within $q^*$. For $q' \in
\mathcal{Q}'_{(l+1)i}$ , $H \in \mathcal{A}'_{(l+1)i}$ does not have discontinuity
in derivative --- does not suddenly switch from constant function $0$ to non-zero
linear function. Now within each $q'$, we impose an order on
$\mathcal{A}^{*}_{(l+1)i}$, if $\pre (a_{(l+1)i})_k - \pre (a_{(l+1)i})_{k'} \geq
0$, we say $H_k \geq H_k'$. Given that $K$ is a finite totally ordered set, the maximum element w.r.t. the defined order
exists, and we denote its index as $k_{\max}$. Thus, for each $q'$, $a_{(l+1)i} =
\pre (a_{(l+1)i})_{k_{\max}}$, which we have proved to be a linear function
over $q'$ in the $\ReLU$ case, so is $a_{li}$. Thus similar to the ReLU only
case, Max Pooling with ReLU divides $q$ into a new set of regions as well.

For the newly created set of regions $\mathcal{Q}_{l+1}$, we have a composed
labeling function
\begin{displaymath}
  \tau_{l+1}(q) = \tau^{\max}_{l+1}(q)\tau^{\relu}_{l+1}(q) ,
\end{displaymath}
where $\tau^{\relu}$ is defined as before, and
\begin{displaymath}
  \pi_{k}\tau^{\max}_{l+1}(q) =
  \begin{cases}
    1 & \text{if $k=\argmax\limits_{k \in K} a_{(l+1)k}(\bvec{x}), \forall \bvec{x} \in q$}\\
    0 & \text{otherwise} .
  \end{cases}
\end{displaymath}

Since for each $q' \in \mathcal{Q}_{l+1}$, it is a sub-region of $q \in
\mathcal{Q}_{l}$, $q' \in \mathcal{Q}_{l}$ holds true as well.
In this case, the new set of regions are built. The max pooling case is proved.

The same with the case $l=1$, to present the effect in the form of
$\bvec{W}_l$, denoting $\tau_{l+1}(\bvec{x}) = \tau^{\relu}_{l+1}$ for
ReLU only case, and $\tau_{l+1}(\bvec{x}) = \tau^{\max}_{l+1}\tau^{\relu}_{l+1}$ for
ReLU with Max
Pooling case, we have in the ReLU only case
\begin{displaymath}
  \bvec{W}^{q}_{l+1}= \diag(\tau_{l+1}(q))\bvec{W}_{l+1} ,
\end{displaymath}
and in the ReLU and Max Pooling case,
\begin{displaymath}
  \bvec{W}^{q}_{l+1}= \bvec{P}_{l+1}\diag(\tau_{l+1}(q))\bvec{W}_{l+1} ,
\end{displaymath}
which both are linear maps over $q \in \mathcal{Q}_{l+1} = \bigcup_{i=1}^{n_{l+1}}\mathcal{Q}_{(l+1)i}$(
note that $\bvec{P}_{l+1}$ is a also a linear mapping/matrix).

By induction, $\forall i \leq l+1$, $\bvec{W}^{q}_{i}$ is a linear map. Cascading the
result, we have the neural network $\bvec{T}$ as a linear map over $q \in \mathcal{Q}_{(l+1)}$
\begin{displaymath}
  \bvec{T}^{l+1}_{q} = \prod\limits_{i=1}^{l+1}\bvec{W}^{q}_{i} .
\end{displaymath}

We now finish the induction and prove that for $ 0 < l \leq L$, there exists a set $\mathcal{Q}_{l}$ such that $\forall q \in \mathcal{Q}_l$, $\bvec{T}^{l}_q$ is
linear over $q$. Given that we are interested in $\bvec{T}^{L}_{q}$ in this paper, we drop the upper index, and denote it as
\begin{displaymath}
  \bvec{T}_{q} = \prod\limits_{i=1}^{L}\bvec{W}^{q}_{i} ,
\end{displaymath}
and the corresponding region set $Q_{L}$ is the set of regions over which $\bvec{T}$ is
linear. We also drop the index, and denote it as $Q$.
\end{proof}

\section{Proof of Lemma 3.3.}
\label{ProofGammaCover4DNN}

\begin{proof}
  By \cref{lm:2}, a set of regions $\mathcal{Q}$ exists such that for $q \in
  \mathcal{Q}$, $\bvec{T}_q$ is a linear mapping induced by $\bvec{T}$ over $q$.

  For any given $\bvec{x} \in \mathcal{X}$, denote the region it belongs to as $q_{\bvec{x}}$.
  Let $d_{\min} = \min_{\bvec{x} \in S_m^{(x)}} \min_{\bvec{x}' \in \partial q_{\bvec{x}}} \rho(\bvec{x}, \bvec{x}')$, the shortest
distance from $\bvec{x}$ to the boundary of $q_{\bvec{x}}$, denoted as $\partial q_{\bvec{x}}$, among all training samples. Denote by $a_{lk}$
the neuron that defines the hyperplane corresponding to the boundary that produces the shortest distance $d_{\min}$. Note that $a_{lk}$ may exist in the intermediate network layers, i.e., $1\leq l \leq L$, and the specific value of $l$ depends on the training set $S_m$ and the learned $\bvec{T}$. By Corollary 3.1.,
we have $\forall \{\bvec{x}' \in \mathcal{X} |\: \bvec{x} + (\bvec{x}' - \bvec{x}) \in q_{\bvec{x}}\}$, $a_{lk|\bvec{x}}(\bvec{x}) -
a_{lk|\bvec{x}}(\bvec{x}') = a_{lk|\bvec{x}}(\bvec{x} - \bvec{x}')$. Thus
\begin{align*}
  ||a_{lk|\bvec{x}}(\bvec{x}) - a_{lk|\bvec{x}}(\bvec{x}')|| &= ||a_{lk|\bvec{x}}(\bvec{x} - \bvec{x}')||\\
                             &\leq ||a_{lk|\bvec{x}}||||\bvec{x} - \bvec{x}'||\\
                             &= ||a_{lk|\bvec{x}}||d_{\min} .
\end{align*}
Given $\bvec{x}' \in \partial q_{\bvec{x}}$, it implies $a_{lk}(\bvec{x}') = 0$, and we have a lower bound on $d_{\min}$ as
\begin{align*}
  d_{\min} &\geq \frac{||a_{lk|\bvec{x}}(\bvec{x}) - a_{lk|\bvec{x}}(\bvec{x}')||}{||a_{lk|\bvec{x}}||}\\
           &= \frac{|a_{lk|\bvec{x}}(\bvec{x})|}{||a_{lk|\bvec{x}}||}\\
           &\geq \frac{|a_{lk|\bvec{x}}(\bvec{x})|}{\prod\limits_{i=1}^{l}\sigma^{i}_{\max | \bvec{x}}} ,
\end{align*}
where $\sigma^{i}_{\max|\bvec{x}}$ is the maximum singular value of $\bvec{W}^{q}_{i}$ for $i = 1, \ldots ,l$.

Since $\bvec{W}^{q}_{i}$ is a submatrix of $\bvec{W}_i$, by Cauchy interlacing law by
rows deletion \cite{Chafa2011SingularVO}, we have $\sigma^{i}_{\max | \bvec{x}} \leq \sigma^{i}_{\max}$, and
\begin{displaymath}
  d_{\min} \geq \frac{|a_{lk|\bvec{x}}(\bvec{x})|}{\prod\limits_{i=1}^{l}\sigma^{i}_{\max}} .
\end{displaymath}

Denote $o(S_m, \bvec{T})/2 = |a_{lk|\bvec{x}}(\bvec{x})|$ to stress the fact that it is a fixed value once $S_m$ and $\bvec{T}$ are given, we have a lower bound
\begin{displaymath}
  r = \frac{o(S_m, \bvec{T})/2}{\prod\limits_{i=1}^{l(S_m, \bvec{T})}\sigma^{i}_{\max}} ,
\end{displaymath}
where we have explicitly write $l(S_m, \bvec{T})$ to emphasize the dependence of $l$ on $S_m$ and $\bvec{T}$. In addition, $q_{\bvec{x}} \in \mathcal{Q}$, we have $\bvec{T}$ is linear over $q_{\bvec{x}}$.
Consequently, a covering set of $\mathcal{X}$ with radius $r$ is found, such
that within each covering ball, $\bvec{T}$ is linear. Then for any given $\bvec{x} \in
\mathcal{X}$ and $\{\bvec{x}' \in \mathcal{X} |\: ||\bvec{x} - \bvec{x}'|| \leq r\}$, $\bvec{x}' \in q_{\bvec{x}}$, thus
$\bvec{T}\bvec{x} - \bvec{T}\bvec{x}' = \bvec{T}_{|\bvec{x}}(\bvec{x} - \bvec{x}')$. The diameter $\gamma$ of the covering ball is $2r$.
\end{proof}

\section{Proof of Lemma 3.4.}
\label{ProofGammaCoverDeltaIsometry4DNN}

\begin{proof}
By \cref{lm:3}, there exists a covering of $\mathcal{X}$ such that $\bvec{T}$ is
linear over each covering ball $B$ containing $\bvec{x} \in S_m^{(x)}$, denoted as $\bvec{T}_{|B}$. By \cref{lm:1}, within such a
$B$, $\bvec{T}_{|B}$ is $\delta_{| B}$-isometry w.r.t. variation space $\mathcal{X} - \mathcal{N}(\bvec{T}_{|B})$. By
Cauchy interlacing law by rows deletion \cite{Chafa2011SingularVO}, we
have
\begin{displaymath}
  \sigma^{i}_{\max | B} \leq \sigma^{i}_{\max}, \ \sigma^{i}_{\min | B} \geq \sigma^{i}_{\min} ,
\end{displaymath}
where $\sigma^{i}_{\min}$ and $\sigma^{i}_{\max}$, $i = 1, \ldots, L$, are respectively the
minimum and maximum singular values of weight matrices of $\bvec{T}$, and
$\sigma^{i}_{\min | B}$ and $\sigma^{i}_{\max | B}$ are the corresponding ones of
$\bvec{T}_{|B}$.

Some extra attentions need to be taken to deal with the $\bvec{P}_{l}$ matrix
introduced by max pooling. Note that in \cref{lm:2},
$||\bvec{P}_l\diag(\tau_l(q))\bvec{W}_l\bvec{x}||$ is equivalent to
$||\diag(\tau_l(q))\bvec{W}_l\bvec{x}||$ since in computing the norm, a summation
is computed anyway. Thus, the Cauchy interlacing law by row deletion applies to
$\bvec{W}_{l}^{q}$ with $\bvec{P}_l$ as well.

Denote
\begin{align*}
  \delta^{1}_{|B} &= \prod\limits_{i=1}^{L}\sigma^{i}_{\max | B} , \ \delta^{2}_{|B} = \prod\limits_{i=1}^{L}\sigma^{i}_{\min | B} , \\
  \delta^{1} &= \prod\limits_{i=1}^{L}\sigma^{i}_{\max}, \ \delta^{2} = \prod\limits_{i=1}^{L}\sigma^{i}_{\min} .
\end{align*}
Anchoring the four points $\delta^{1}_{B}, \delta^{2}_{|B}, \delta^{1},
\delta^{2}$ on the graph of  $f(x) = |x - 1|$, we observe that $[\delta^{2}_{|B}, \delta^{1}_{|B}]$ lies between the interval $[\delta^{2}, \delta^{1}]$. Thus we have
\begin{displaymath}
  \max(|\delta^{1}_{|B} - 1|, |\delta^{2}_{|B} - 1|) \leq \max(|\delta^{1} - 1|, |\delta^{2} - 1|) .
\end{displaymath}
Since $\bvec{T}_{|B}$ means $\bvec{T}$ over $B$ is $\max(|\delta^{1}_{|B} - 1|, |\delta^{2}_{|B} - 1|)$-isometry, which implies that it is also
$\max(|\delta^{1} - 1|, |\delta^{2} - 1|)$-isometry.
\end{proof}

\section{Proof of \cref{thm:ge_dnn}}
\label{proof-thm-ge_dnn}

\begin{proof}
  Similar to the proof of \cref{thm:1}, we partition the space $\mathcal{Z}$
via the assumed $\gamma$-cover. Since $\mathcal{X}$ is a $k$-dimensional manifold,
its covering number is upper bounded by $C_{\mathcal{X}}^{k}/\gamma^{k}$. Let $K$ be
the overall number of covering set, which is upper bounded by
$|\mathcal{Y}|C_{\mathcal{X}}^{k}/\gamma^{k}$. Denote $C_{i}$ the $i$th covering ball,
and let $N_i$ be the set of index of training samples that fall into
$C_i$. Note that $(|N_i|)_{i=1\ldots K}$ is an IDD multimonial random variable with
parameters $m$ and $(|\mu(C_i)|)_{i=1\ldots K}$. Then
  \begin{align}
    &|R(f \circ \bvec{T}) - R_m(f \circ \bvec{T})| \nonumber \\
    = &  |\sum\limits_{i=1}^{K}\mathbb{E}_{z \sim \mu}[\mathcal{L}(f(\bvec{T}\bvec{x}), y)]\mu(C_i) -
        \frac{1}{m}\sum\limits_{i=1}^{m}\mathcal{L}(f(\bvec{T}\bvec{x}_i), y_i)| \nonumber \\
    \leq &  |\sum\limits_{i=1}^{K}\mathbb{E}_{z \sim \mu}[\mathcal{L}(f(\bvec{T}\bvec{x}), y)]\frac{|N_i|}{m} -
        \frac{1}{m}\sum\limits_{i=1}^{m}\mathcal{L}(f(\bvec{T}\bvec{x}_i), y_i)| \nonumber \\
      & +  |\sum\limits_{i=1}^{K}\mathbb{E}_{z \sim \mu}[\mathcal{L}(f(\bvec{T}\bvec{x}), y)]\mu(C_i) -
        \sum\limits_{i=1}^{K}\mathbb{E}_{z \sim \mu}[\mathcal{L}(f(\bvec{T}\bvec{x}), y)]\frac{|N_i|}{m}|
        \nonumber \\
    \leq & |\frac{1}{m}\sum\limits_{i=1}^{K}\sum\limits_{j\in N_i}\max\limits_{z' \in
        C_i, \bvec{x}' - \bvec{x}_j \in \mathcal{P}_{\bvec{x}_j}}|\mathcal{L}(f(\bvec{T}\bvec{x}'), y') - \mathcal{L}(f(\bvec{T}\bvec{x}_j), y_j)| \label{t:1}\\
    & + |\max\limits_{z \in \mathcal{Z}}|\mathcal{L}(f(\bvec{T}\bvec{x}),
      y)|\sum\limits_{i=1}^{K}|\frac{|N_i|}{m} - \mu(C_i)|| \label{t:2} .
  \end{align}
  Remember that $z = (\bvec{x}, y)$.

  By the assumption that $\bvec{T}$ is $\gamma$-cover $\delta$-isometry w.r.t. $\mathcal{P}_{\bvec{x}}$ of $\bvec{x} \in S_m^{(x)}$ and the
  Lipschitz constant of $\mathcal{L} \circ f$ is $A$, suppose the maximum is achieved at $\bvec{x}_k$ and
  $\bvec{x}_k \in C_p$, we have
  \begin{align}
    & \max\limits_{z' \in C_p, \bvec{x}' - \bvec{x}_k \in \mathcal{P}_{\bvec{x}_k}}|\mathcal{L}(f(\bvec{T}\bvec{x}'), y') - \mathcal{L}(f(\bvec{T}\bvec{x}_k),
      y_k)|\nonumber \\
    \leq & A\max\limits_{z' \in C_p, \bvec{x}' - \bvec{x}_k \in \mathcal{P}_{\bvec{x}_k}}||\bvec{T}_{|\bvec{x}_k}(\bvec{x}' - \bvec{x}_{k})||\label{t:3}\\
    \leq & A\max\limits_{z' \in C_p, \bvec{x}' - \bvec{x}_k \in \mathcal{P}_{\bvec{x}_k}}(||\bvec{x}' - \bvec{x}_{k}|| + \delta)\label{t:4}\\
    \leq & A(\gamma + \delta) \nonumber,
  \end{align}
  where $\delta = 2b|\prod_{i=1}^{L}\sigma^{i}_{\max} - 1|$ in (\ref{t:4}) since we have $||\bvec{T}_{|\bvec{x}_k}(\bvec{x}' - \bvec{x}_{k})|| \leq ||\bvec{x}' - \bvec{x}_{k}|| + 2b|\prod_{i=1}^{L}\sigma^{i}_{\max} - 1|$ by Lemma \ref{lm:1}, and $\gamma = o(S_m, \bvec{T})/\left(\prod\limits_{i=1}^{l(S_m, \bvec{T})}\sigma^{i}_{\max}\right) > 0$ by Lemma \ref{lm:3}, with values of $o(S_m, \bvec{T})$ and $1 \leq l(S_m, \bvec{T}) \leq L$ depending on the training set $S_m$ and learned network $\bvec{T}$. Thus \cref{t:1} is less than or equal to $A(\gamma + \delta)$ with the specified $\gamma$ and $\delta$.
  By Breteganolle-Huber-Carol inequality, \cref{t:2} is less than or equal to
  $M\sqrt{\frac{\log(2)|\mathcal{Y}|2^{k+1}C_{\mathcal{X}}^{k}}{\gamma^{k}n} + \frac{2
      \log(1/\nu)}{m}}$.

  The proof is finished. Note that given the covering number of $\mathcal{X}$ as $\mathcal{N} = (\frac{C_{\mathcal{X}}}{\gamma/2})^{k}$, we have also proved that the algorithm is $(|\mathcal{Y}| \mathcal{N}, A(\gamma + \delta))$-robust.
\end{proof}

\section{}
\label{ApendixSecStiefelOptimAlgm}

In this section, we present a SGD based algorithm for the constrained optimization problem (i.e. problem (\ref{EqnStiefelConstrained})) of training a DNN of $L$ layers with parameters $\Theta = \{ \mathbf{W}_l, \mathbf{b}_l\}_{l=1}^L$ and objective function ${\cal{L}}$. The constraints enforce the weight matrix (kernel) $\mathbf{W}_l \in \mathbb{R}^{n_l\times n_{l-1}}$ of any $l^{th}$ layer of the network staying on the Stiefel manifold defined as $ {\cal{M}}_l = \{ \mathbf{W}_l \in \mathbb{R}^{n_l\times n_{l-1}} | \mathbf{W}_l^{\top}\mathbf{W}_l = \mathbf{I} \}$, assuming $n_l \geq n_{l-1}$. The presented algorithm applies directly to fully-connected network layers. For convolutional layers used in CNNs, one may refer to \cref{SecOrthCNNs} for how to convert their layer kernels as matrices.

For the $t^{th}$ iteration of SGD, the algorithm performs the following sequential steps to update $\mathbf{W}^t_l \in {\cal{M}}_l$ for the $l^{th}$ network layer with $l \in \{1, \dots, L\}$. Updating of other network parameters such as bias vectors $\{\mathbf{b}_l\}_{l=1}^L$ is the same as standard SGD based methods. The algorithm is similar to those of optimization on matrix manifolds in \cite{MatrixManifoldOptimBook,SGDRiemannianManifold,OzayOkataniCNNKernelSubManifold}, where properties of convergence are also analyzed.
\begin{enumerate}
\item Compute the gradient $\frac{\partial{\cal{L}}}{\partial{\mathbf{W}^t_l}}$ in the embedding Euclidean space via back-propagation. One may alternatively use the momentum \cite{Momentum} to replace the gradient term in the following steps.

\item Project $\frac{\partial{\cal{L}}}{\partial{\mathbf{W}^t_l}}$ (or its momentum version) onto the tangent space $T_{\mathbf{W}^t_l}{\cal{M}}_l$ by ${\cal{P}}_{\mathbf{W}^t_l}\frac{\partial{\cal{L}}}{\partial{\mathbf{W}^t_l}}$, to obtain the manifold gradient $\Omega_{\mathbf{W}^t_l}$. For the considered Stiefel manifold, the tangent space at $\mathbf{W}^t_l$ is defined as $T_{\mathbf{W}^t_l}{\cal{M}}_l = \{ \mathbf{Z} \in \mathbb{R}^{n_l\times n_{l-1}} | \mathbf{W}_l^{t \top}\mathbf{Z} + \mathbf{Z}^{\top}\mathbf{W}^t_l = 0 \}$, and the projection operator is defined as ${\cal{P}}_{\mathbf{W}^t_l}\frac{\partial{\cal{L}}}{\partial{\mathbf{W}^t_l}} = \left(\mathbf{I} - \mathbf{W}^t_l\mathbf{W}_l^{t\top}\right)\frac{\partial{\cal{L}}}{\partial{\mathbf{W}^t_l}} + \frac{1}{2}\mathbf{W}^t_l\left(\mathbf{W}_l^{t\top}\frac{\partial{\cal{L}}}{\partial{\mathbf{W}^t_l}} - \frac{\partial{\cal{L}}}{\partial{\mathbf{W}^t_l}}^{\top}\mathbf{W}^t_l\right)$, where $\mathbf{I}$ is the identity matrix of compatible size.

\item Update $\mathbf{W}^t_l$ as $\mathbf{W}^t_l - \eta^t \Omega_{\mathbf{W}^t_l}$ with the step size $\eta^t$ that satisfies conditions of convergence \cite{MatrixManifoldOptimBook,SGDRiemannianManifold}.

\item Perform the retraction ${\cal{R}}_{\mathbf{W}^t_l}(-\eta^t \Omega_{\mathbf{W}^t_l})$ that defines a mapping from the tangent space to the Stiefel manifold, and update $\mathbf{W}^t_l$ as $\mathbf{W}^{t+1}_l = {\cal{R}}_{\mathbf{W}^t_l}(-\eta^t \Omega_{\mathbf{W}^t_l})$. The retraction is achieved by ${\cal{R}}_{\mathbf{W}^t_l}(-\eta^t \Omega_{\mathbf{W}^t_l}) = {\cal{Q}}(\mathbf{W}^t_l - \eta^t \Omega_{\mathbf{W}^t_l})$, where the operator ${\cal{Q}}$ denotes the Q factor of the QR matrix decomposition. QR decomposition can be computed using Gram-Schmidt orthonormalization.
\end{enumerate}

%\section{}
%\label{ApendixSecSVBOptimAlgm}
%
%We present the algorithmic details of our proposed Singular Value Bounding (SVB) in Algorithm \ref{AlgmSVB}.

\section{Proof of Lemma 4.1}
\label{ApendixSecBBNProof}

%\noindent\textbf{Lemma 4.1.} \emph{For a matrix $\mathbf{W} \in \mathbb{R}^{M\times N}$ with singular values of all $1$, and a diagonal matrix $\mathbf{G} \in \mathbb{R}^{M\times M}$ with nonzero entries $\{ g_i \}_{i=1}^M$, let $g_{\max} = \max(|g_1|, \dots, |g_M|)$ and $g_{\min} = \min(|g_1|, \dots, |g_M|)$, the singular values of $\widetilde{\mathbf{W}} = \mathbf{G}\mathbf{W}$ is bounded in $[g_{min}, g_{\max}]$. When $\mathbf{W}$ is fat, i.e., $M \le N$, and $\textrm{rank}(\mathbf{W}) = M$, singular values of $\widetilde{\mathbf{W}}$ are exactly $\{ |g_i| \}_{i=1}^M$.}

\begin{proof} We first consider the general case, and let $P = \min(M, N)$. Denote singular values of $\mathbf{W}$ as $\sigma_1 = \cdots = \sigma_P = 1$, and singular values of $\widetilde{\mathbf{W}}$ as $\tilde{\sigma}_1 \ge \cdots \ge \tilde{\sigma}_P$. Based on the properties of matrix extreme singular values, we have
\begin{eqnarray}
\sigma_1 = \|\mathbf{W}\|_2 = \max_{\mathbf{x}\ne 0}\frac{\|\mathbf{W}\mathbf{x}\|_2}{\|\mathbf{x}\|_2} = \min_{\mathbf{x}\ne 0}\frac{\|\mathbf{W}\mathbf{x}\|_2}{\|\mathbf{x}\|_2} = \sigma_P = 1 . \nonumber
\end{eqnarray}
Let $\mathbf{x}^{*} = \arg\max_{\mathbf{x}\ne 0} \frac{\|\widetilde{\mathbf{W}}\mathbf{x}\|_2}{\|\mathbf{x}\|_2}$, we have
\begin{eqnarray}
\tilde{\sigma}_1 = \frac{\|\widetilde{\mathbf{W}}\mathbf{x}^{*}\|_2}{\|\mathbf{x}^{*}\|_2} = \frac{\|\mathbf{G}\mathbf{W}\mathbf{x}^{*}\|_2}{\|\mathbf{x}^{*}\|_2} \le \frac{\|\mathbf{G}\|_2 \|\mathbf{W}\mathbf{x}^{*}\|_2}{\|\mathbf{x}^{*}\|_2} , \nonumber
\end{eqnarray}
where we have used the fact that $\|\mathbf{A}\mathbf{b}\|_2 \le \|\mathbf{A}\|_2\|\mathbf{b}\|_2$ for any $\mathbf{A} \in \mathbb{R}^{m\times n}$ and $\mathbf{b} \in \mathbb{R}^n$. We thus have
\begin{eqnarray}
\tilde{\sigma}_1 \le \|\mathbf{G}\|_2 \frac{\|\mathbf{W}\mathbf{x}^{*}\|_2}{\|\mathbf{x}^{*}\|_2} \le \|\mathbf{G}\|_2 \max_{\mathbf{x}\ne 0} \frac{\|\mathbf{W}\mathbf{x}\|_2}{\|\mathbf{x}\|_2} = |g_{\max}| . \nonumber
\end{eqnarray}
Since $\mathbf{G}$ has nonzero entries, we have $\mathbf{W} = \mathbf{G}^{-1}\widetilde{\mathbf{G}}$. Let $\mathbf{x}^{*} = \arg\min_{\mathbf{x}\ne 0}\frac{\|\widetilde{\mathbf{W}}\mathbf{x}\|_2}{\|\mathbf{x}\|_2}$, the properties of matrix extreme singular values give $\tilde{\sigma}_P = \frac{\|\widetilde{\mathbf{G}}\mathbf{x}^{*}\|_2}{\|\mathbf{x}^{*}\|_2}$, and $\sigma_P = \min_{\mathbf{x}\ne 0}\frac{\|\mathbf{W}\mathbf{x}\|_2}{\|\mathbf{x}\|_2} = 1$. We thus have
\begin{eqnarray}
1 = \min_{\mathbf{x}\ne 0}\frac{\|\mathbf{G}^{-1}\widetilde{\mathbf{G}}\mathbf{x}\|_2}{\|\mathbf{x}\|_2} \le \frac{\|\mathbf{G}^{-1}\widetilde{\mathbf{G}}\mathbf{x}^{*}\|_2}{\|\mathbf{x}^{*}\|_2} \le \|\mathbf{G}^{-1}\|_2 \frac{\|\widetilde{\mathbf{G}}\mathbf{x}^{*}\|_2}{\|\mathbf{x}^{*}\|_2} , \nonumber
\end{eqnarray}
which gives $\tilde{\sigma}_P \ge |g_{\min}|$. Overall, we have
\begin{eqnarray}
|g_{\max}| \ge \tilde{\sigma}_1 \ge \cdots \ge \tilde{\sigma}_P \ge |g_{\min}| . \nonumber
\end{eqnarray}

We next consider the special case of $M \le N$ and $\textrm{rank}(\mathbf{W}) = M$. Without loss of generality, we assume diagonal entries $\{ g_i \}_{i=1}^M$ of $\mathbf{G}$ are all positive and ordered. By definition we have $\widetilde{\mathbf{W}} = \mathbf{I}\mathbf{G}\mathbf{W}$, where $\mathbf{I}$ is an identity matrix of size $M\times M$. Let $\mathbf{V} = \left[ \mathbf{W}^{\top}, \mathbf{W}^{\bot\top} \right]$, where $\mathbf{W}^{\bot}$ denotes the orthogonal complement of $\mathbf{W}$, we thus have the SVD of $\widetilde{\mathbf{W}}$ by construction as $\widetilde{\mathbf{W}} = \mathbf{I} \left[ \mathbf{G}, \mathbf{0} \right] \mathbf{V}^{\top}$. When some values of $\{ g_i \}_{i=1}^M$ are not positive, the SVD can be constructed by changing the signs of the corresponding columns of either $\mathbf{I}$ or $\mathbf{V}$. Since matrix singular values are uniquely determined (while singular vectors are not), singular values of $\widetilde{\mathbf{W}}$ are thus exactly $\{ |g_i| \}_{i=1}^M$.
\end{proof}

%\section{}
%\label{ApendixSecBBNAlgm}
%
%We present the algorithmic details of our proposed Bounded Batch Normalization (BBN) in Algorithm \ref{AlgmBBN}.
%
%%We note that in Algorithm \ref{AlgmBBN}, we do not take the absolute values. This is because values of $\{ \upsilon_i \}_{i=1}^n$ are usually initialized as $1$, and they are empirically observed to keep positive during the process of network training.

\end{document}